\documentclass{article}
\usepackage{comment}
\usepackage{amsmath}
\usepackage{amsthm}
\usepackage{amssymb}
\usepackage{graphicx, mathtools}
\usepackage{sidecap}
\usepackage{enumitem}
\usepackage{algorithm}
\usepackage{algpseudocode}
\usepackage{wrapfig}
\usepackage{tcolorbox}
\usepackage{natbib}

\theoremstyle{definition} % Definitions, Assumptions, Examples usually roman font
\newtheorem{definition}{Definition}[section]

\theoremstyle{plain}      % Propositions, Theorems usually italics
\newtheorem{theorem}{Theorem}[section]
\newtheorem{lemma}[theorem]{Lemma}
\newtheorem{corollary}[theorem]{Corollary}
\newtheorem{proposition}[theorem]{Proposition} % Shared counter with definition

\theoremstyle{remark}
\newtheorem*{remark}{Remark}  % MS: this was one bug (typo 'rem' vs 'remark')

\providecommand{\texorpdfstring}[2]{#1}

\providecommand{\R}{\mathbb R}

\providecommand{\Lip}{\operatorname{Lip}}
\providecommand{\supp}{\operatorname{supp}}
\providecommand{\co}{\operatorname{co}}

\providecommand{\EE}{\mathbb E}
\providecommand{\cN}{\mathcal N}

\PassOptionsToPackage{numbers}{natbib}

\usepackage[main, final]{neurips_2026}

\usepackage[utf8]{inputenc} % allow utf-8 input
\usepackage[T1]{fontenc}    % use 8-bit T1 fonts
\usepackage[colorlinks=true, breaklinks=true,
            linkcolor=blue!80!black, %red
            citecolor=blue!80!black, %blue,
            urlcolor=blue!80!black]
            {hyperref}
\usepackage{url}            % simple URL typesetting
\usepackage{booktabs}       % professional-quality tables
\usepackage{amsfonts}       % blackboard math symbols
\usepackage{nicefrac}       % compact symbols for 1/2, etc.
\usepackage{microtype}      % microtypography
\usepackage{xcolor}         % colors

\def\showcomments{}

\ifdefined\showcomments
  \newcommand\matt[1]{\textcolor{blue}{[ms: #1]}}
\else
    \newcommand\matt[1]{}
\fi

\def\showcomments{}

\ifdefined\showcomments
  \newcommand\nilava[1]{\textcolor{red}{[nm: #1]}}
\else
    \newcommand\nilava[1]{}
\fi

\def\showcomments{}

\ifdefined\showcomments
  \newcommand\ams[1]{\textcolor{green}{[ams: #1]}}
\else
    \newcommand\ams[1]{}
\fi
\title{
Attention as In-Context Empirical Bayes: \\ A Two-Stage View via Particle Dynamics
}
\author{%
  Matthew Smart\textsuperscript{1,4,$\ast$}
    \quad
  Soumya Ganguly\textsuperscript{2,$\ast$}
    \quad
  Nilava Metya\textsuperscript{2}\\[2pt]
  {\bfseries
  Alexandre V. Morozov\textsuperscript{3}\quad
  Anirvan M. Sengupta\textsuperscript{3,5}}\\[5pt]
  \textsuperscript{1}Lewis-Sigler Institute for Integrative Genomics,
  Princeton University\\
  \textsuperscript{2}Department of Mathematics,
  Rutgers University \\
  \textsuperscript{3}Department of Physics and Astronomy, Rutgers University\\
  \textsuperscript{4}Center for Computational Biology;
  \textsuperscript{5}Center for Computational Quantum Physics \\and Center for Computational Mathematics,
  Flatiron Institute\\[4pt]
}

\begin{document}

% ============================================
\maketitle

\begingroup
\renewcommand{\thefootnote}{*}
\footnotetext[1]{%
  %\raggedright
  Equal contribution.
  %Correspondence:
  %Emails:
  \texttt{mattsmart@princeton.edu},
  \texttt{soumya.ganguly@rutgers.edu},
  \texttt{nilava.metya@rutgers.edu},
  \texttt{morozov@physics.rutgers.edu},
  \texttt{anirvans.physics@gmail.com}.%
}
\endgroup

% ============================================

\begin{abstract}
We study minimal attention-only transformers under all-token corruption and show they admit a two-stage empirical Bayes interpretation. A single attention step computes a kernel-weighted posterior mean with respect to the empirical distribution defined by the context. Depth refines this distribution through particle dynamics (Stage 1), while a long-range skip-connection carries the noisy input as a query for posterior inference (Stage 2), revealing distinct statistical roles for depth and attention residuals. The framework isolates a minimal setting in which the context itself induces a depth-dependent energy landscape governing in-context inference. We show that effective denoising can emerge without an explicit noise schedule: a fixed kernel bandwidth and finite integration horizon suffice, yielding a principled depth–noise relationship. We further establish a posterior-mean recovery guarantee for Gaussian priors, showing the empirical estimator converges to the Bayes-optimal predictor under sequential asymptotic conditions. Connecting these dynamics to reverse-diffusion limits, our results provide a statistical interpretation of attention as in-context inference via sample-based posterior estimation, without explicit density modeling.
\end{abstract}

% ============================================
\section{Introduction}
% ============================================

\begin{figure*}[ht!]
\centering
\includegraphics[width=0.999\textwidth]{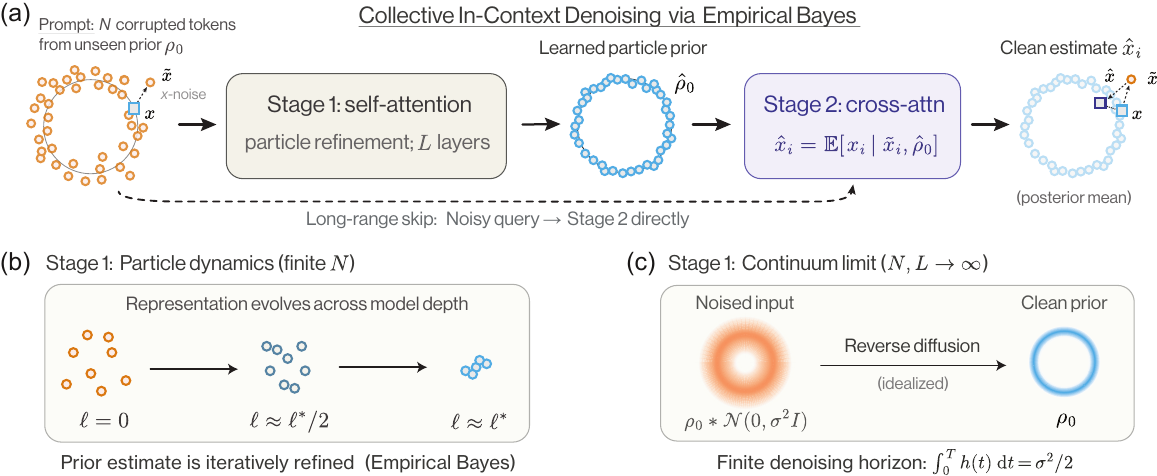}
\caption{
  (a) \textbf{Collective in-context denoising}: Multilayer attention implements a two-stage EB procedure. Depth iteratively refines a particle prior (Stage 1), while a long-range skip--- acting as an Attention Residual (AttnRes)---enables posterior averaging with the initial noisy input as query (Stage 2).
  (b) \textbf{Particle dynamics}: Discrete self-attention updates move corrupted tokens from their initial noised distribution toward high-density regions of the underlying clean prior.
  (c) \textbf{Theoretical limit}: In the large-context, continuous-depth regime, these dynamics recover an anti-diffusive denoising operator consistent with a reverse diffusion process.
  }
\label{fig:setup}
\end{figure*}

Transformers \citep{vaswani2017attention} have achieved remarkable empirical success across language, vision, and generative modeling, yet the mechanisms underlying this flexibility are only partially understood. Several distinct perspectives on their defining ingredient --- \emph{attention} --- have emerged. \citet{geshkovski2023emergence, rigollet2025mean, bruno2025_neurips_meanfield} treat multilayer attention as interacting particle systems; \citet{rosu2025_denoising} connect attention to score-based denoising; and \citet{ramsauer2021iclr} interpret attention as memory retrieval in dense associative memory networks \citep{krotov2016dense}. This latter connection raises the question of what task would naturally motivate identity-scaled weights and single-step energy minimization. \citet{smart2025context} proposed \emph{in-context denoising} as such a task, showing that a single attention layer provably solves Bayes-optimal denoising when context tokens are clean. However, these perspectives have largely been studied in isolation; their relationship, as well as the architectural consequences they imply, remains unclear.

Motivated by this lineage, we study transformer dynamics through a collective in-context denoising task: $N$ tokens drawn i.i.d.\ from an unknown prior $\rho_0$ are corrupted by isotropic Gaussian noise of known variance $\sigma^2$ and processed jointly by multilayer self-attention, with the goal of recovering the clean tokens. This generalizes \citet{smart2025context} to a regime in which the prior itself must be inferred from corrupted samples, wherein depth becomes useful for prior refinement. Our central claim is that attention-only transformers in this setting implement a finite-sample empirical Bayes (EB) denoiser, decomposable into two stages (Fig.~\ref{fig:setup}). \emph{Stage 1} (self-attention across depth) iteratively refines a particle approximation \citep{del2013mean, sander2022sinkformers} to the prior $\rho_0$ from the corrupted samples. \emph{Stage 2} (a final cross-attention step from the noisy input, enabled by a long-range residual connection) computes a posterior mean against the refined prior. This decomposition gives the architectural elements of depth and attention residuals statistically derivable roles.

%%%
This framework has roots in classical empirical Bayes \citep{robbins1956empirical, miyasawa1961empirical, efron2011tweedie, JohnstoneSilverman2005, raphan2011_simoncelli}, 
% NEW CR:
where multiple observations under a known likelihood are used to estimate a common unknown prior before posterior inference. Under Gaussian corruption, the posterior mean connects naturally to score-based denoising through Tweedie’s formula, linking our perspective to modern diffusion modeling \citep{ho2020denoising, song2021scorebased}. 
Our EB approach is complementary to existing perspectives on attention: 
unlike recent score-based denoising treatments \citep{rosu2025_denoising}, 
we provide an explicit empirical Bayes interpretation.
By connecting mean-field particle dynamics \citep{geshkovski2023emergence, rigollet2025mean, bruno2025_neurips_meanfield, castin2026unified_v2-2026} to a denoising task with a finite integration horizon, our work contextualizes the distinct regimes those works emphasize.
Finally, our two-stage decomposition reveals a \emph{dynamic} associative memory landscape \citep{ramsauer2021iclr, smart2025context} that evolves in-context with model depth and governs the network's inference step. 

Our main contributions are as follows:

\begin{enumerate}

  \item We generalize \emph{in-context denoising} \citep{smart2025context} to the all-token corruption setting, showing that attention-only transformers admit a two-stage empirical Bayes interpretation: \textbf{depth} iteratively refines a particle approximation of the prior from noisy samples alone (Stage 1), while \textbf{attention residuals} use the original noisy input to query the resulting energy landscape for posterior inference (Stage 2). This decomposition clarifies the provenance of stored patterns in associative-memory views of attention \citep{ramsauer2021iclr}.

  \item 
  We use the established attention--score connection to formulate Stage~1 as a continuous-time particle flow for nonparametric, in-context prior refinement. This links looped attention head dynamics to collective denoising within our posterior-recovery analysis.
  
  \item 
  We show that effective denoising does not require a noise or parameter schedule: it can be achieved via a fixed attention kernel bandwidth $\beta$ and a finite integration horizon $T^* \approx \sigma^2/2$, providing a principled depth–noise relationship that informs architectural design.

   \item 
   % NEW CR:
   %We prove a conditional posterior-mean recovery theorem for an admissible class of priors \(\mathcal{A}_\tau\) characterized by stability of the mean-field flow under hard truncation and recovery of the clean law.
   % ORIG:
   We prove a sequential posterior-mean recovery theorem for an admissible class of priors \(\mathcal{A}_\tau\) characterized by stability of the mean-field flow under hard truncation. 
   The theorem shows that the truncated particle-flow estimator converges, through particular sequential limits, to the Bayes posterior mean uniformly on compact query sets. We verify that Gaussian priors belong to \(\mathcal{A}_\tau\).
   In contrast to \citet{bruno2025_neurips_meanfield}, whose particle system is compact via spherical layer norm, our analysis is set on $\mathbb{R}^d$ and must control the dynamics through hard truncation, requiring stability estimates for the truncated mean-field flow.
 
\end{enumerate}

% =============================================
\section{Problem formulation and background}
% =============================================
\label{sec:problem_setup}

\paragraph{Collective denoising task.}
We assume we are given a collection of noisy samples $\{\tilde x_i\}_{i=1}^N \subset \mathbb{R}^d$ 
obtained by adding noise to uncorrupted datapoints $\{ x_i\}_{i=1}^N$ independently sampled from an unknown prior distribution $P_0$ with density $\rho_0$ (Fig. \ref{fig:setup}). We have $\tilde x_i=x_i+\epsilon_i, \epsilon_i\overset{iid}\sim \mathcal{N}(0,\sigma^2 I_d)$ where $I_d$ is the $d$-dimensional identity matrix.
The task is to approximately recover the uncorrupted data under the $x$-prediction mean squared error (MSE) objective $\mathbb{E}\big[\| \hat{x}_i - x_i\|^2\big]$. We generate denoised estimates $\hat{x}_i$ using an empirical Bayes approach which is divided into two stages:

\begin{itemize}
\item[\textbf{Stage 1.}] We perform iterative refinement on $\{\tilde x_i\}$ to construct a particle approximation to the underlying prior distribution $\rho_0$ from which $\{x_i\}$ were sampled. 

\item[\textbf{Stage 2.}] Given the particle approximation for the prior from Stage 1, we denoise each $\tilde x_i$ by computing the posterior mean under the reconstructed prior.
\end{itemize}

At first glance, this problem appears challenging, as the prior $\rho_0$ is not known \emph{a priori} and must be inferred from the corrupted samples alone.
Nevertheless, we will show that simple nonparametric schemes can achieve near Bayes-optimal performance in certain settings, and that these schemes align naturally with the structure of attention-based architectures.
The problem generalizes the single-token denoising task studied in \cite{smart2025context}, which was shown to be solvable by a \emph{single} attention layer (corresponding to Stage 2 above). 
In contrast, our approach addresses all-token corruption by first estimating the prior across depth (Stage 1) before posterior readout (Stage 2).

Below, we begin by drawing an explicit connection between Stage 1 and reverse diffusion. This leads to an attention-based construction of the score (log-density gradient) that can be applied directly to noisy data in a non-parametric manner. This connection motivates a discrete particle update (Alg.~\ref{alg:two_stage_multilayer_attn}), which admits a continuous-time limit (Alg.~\ref{alg:full-sample-sequential-posterior}). For analytical tractability, we introduce a truncated variant (Alg.~\ref{alg:truncated-sequential-posterior}), whose behavior can be studied rigorously. We provide empirical demonstrations in a finite-sample setting before formally analyzing the resulting particle dynamics and establishing posterior recovery guarantees. We emphasize that the reverse-diffusion connection is heuristic, and several steps in this chain are not controlled in full generality.

% ========================================================
\section{Reverse diffusion as an attention-like particle dynamics for Stage 1} 
\label{sec:reverse_diffusion_prop}
% ========================================================

A connection between attention and score-based methods has been noted \citep{rosu2025_denoising, ilin2026discoformerplugindensityscore}, with classical roots in kernel-based estimation of density gradients \citep{FukunagaHostetler1975, CamaniciuMeer1999meanshift}.
We revisit this correspondence to set up the subsequent analysis.
Here we provide an informal proposition
showing that reverse diffusion can be approximated by a particle system whose update rule takes the form of a leaky residual averaging step. The resulting dynamics closely resembles attention mechanisms~\citep{vaswani2017attention} and kernel regression \citep{nadaraya1964estimating,watson1964smooth}.

\begin{proposition}
\label{prop:informal_reverse_diffusion_derivation}
    {\bf{(Informal)}} Assume that the typical density of the corrupted data distribution is $\bar\rho$. Under the setup described in section \ref{sec:problem_setup}, we start with noisy outputs $\left\{ \tilde x_i \right\}_{i=1}^N \subset \mathbb{R}^d$ and run the following dynamics 
    \begin{equation} 
    \label{upd:discr}
    z_i^{(\ell+1)} = (1-\eta) z_i^{(\ell)} + \eta \frac{ \sum_j e^{-\frac{\beta}{2}\|z_i^{(\ell)} - z_j^{(\ell)}\|^2} z_j^{(\ell)} }{ \sum_j e^{-\frac{\beta}{2}\|z_i^{(\ell)} - z_j^{(\ell)}\|^2} } \equiv (1-\eta) z_i^{(\ell)} + \eta \sum_j a_{ij} (\{ z_i^{(\ell)} \}) z_j^{(\ell)}, 
    \end{equation}
    for $l=0,1, \ldots ,L-1$, starting with $z_i(0)=\tilde x_i$. Here, $\eta \in (0,1)$ is the step size and $L \approx {\beta\sigma^2}/{2\eta}$ is the number of steps or transformation layers. Heuristically, if we let $N\to\infty$ and $\beta\to\infty$ in such a way that $\tfrac{N\bar \rho}{\beta^{d/2}}\to\infty$ and also let $\eta\to 0$, then $\{ z_i^{(L)} \}_{i=1}^N \subset \mathbb{R}^d$ forms an approximation of $P_0$.
\end{proposition}

\begin{proof} The details of the arguments are provided in Appendix \ref{app:informal_argument}. 
\end{proof}

The update Eq.~\eqref{upd:discr} admits three equivalent interpretations: (i) reverse diffusion via local denoising, (ii) Nadaraya--Watson kernel regression, (iii) attention-like averaging with Gaussian similarity weights. 

In the infinite-particle, continuous-depth limit of these dynamics, the Stage 1 integration time remains $\sigma^2/2$ at leading order in $\beta$, with a first-order correction of size $O(\beta^{-1})$.

\begin{proposition}
{\bf{(Informal)}}
\label{prop:denoising_time}
For sufficiently regular priors $P_0$ in the limit $N\to\infty$ and finite $\beta$, the denoising time is $\frac{\sigma^2}{2}+O(\beta^{-1})$.
\end{proposition}
\begin{proof}
See Lemmas~\ref{lem:purturb} and~\ref{lem:time} for the formal statements and proofs.
\end{proof} 

Proposition~\ref{prop:denoising_time} uses score-flow time. In the time-discretized dynamics, each layer advances this time by a fixed amount $h=\eta/\beta$, as illustrated in Fig.~\ref{fig:gaussian_linear_variance_decay}(a).
The associated depth calibration is a population prediction; optimal depth at finite context remains open.

Leveraging insights from Propositions~\ref{prop:informal_reverse_diffusion_derivation} and ~\ref{prop:denoising_time},
we next demonstrate that multilayer attention schemes can implement a finite-sample empirical Bayes procedure that dynamically refines a particle representation to improve posterior inference with depth, before presenting our asymptotic analysis.

\begin{algorithm}[h]
\small
\caption{Two-stage denoising via multilayer attention}
\label{alg:two_stage_multilayer_attn}
\begin{algorithmic}[1]
\Require Noisy tokens $\{\tilde x_i\}_{i=1}^N \subset \mathbb{R}^d$, noise variance $\sigma^2$, bandwidth $\beta$, step size $\eta$, depth $L$.
\State Initialize $z_i^{(0)} \gets \tilde x_i$.

\For{$\ell=0,\dots,L-1$}  \Comment{Stage 1 (particle refinement)}
    \State $a_{ij}^{(\ell)} =
    \dfrac{\exp\!\big(-\tfrac{\beta}{2}\|z_i^{(\ell)}-z_j^{(\ell)}\|^2\big)}
    {\sum_k \exp\!\big(-\tfrac{\beta}{2}\|z_i^{(\ell)}-z_k^{(\ell)}\|^2\big)}$
    %\Comment{Compute self-attention weights}
    \State $z_i^{(\ell+1)} \gets (1-\eta)z_i^{(\ell)} + \eta\sum_j a_{ij}^{(\ell)} z_j^{(\ell)}$
    
\EndFor

\State $\beta_c \gets 1/\sigma^2$
\For{$i=1,\dots,N$} \Comment{Stage 2 (posterior mean)}
    \State $b_{ij} =
    \dfrac{\exp\!\big(-\tfrac{\beta_c}{2}\|\tilde x_i - z_j^{(L)}\|^2\big)}
    {\sum_k \exp\!\big(-\tfrac{\beta_c}{2}\|\tilde x_i - z_k^{(L)}\|^2\big)}$
    %\Comment{Compute cross-attention weights}
    \State $\hat x_i \gets \sum_j b_{ij} z_j^{(L)}$
    
\EndFor
\end{algorithmic}
\end{algorithm}

% ------------------------
\section{Attention architecture for collective denoising} 

Algorithm \ref{alg:two_stage_multilayer_attn} outlines a multilayer attention architecture implementing the collective denoising approach to Fig. \ref{fig:setup} (details in Appendix \ref{app:attention_arch_details}). 
We use this for finite-sample, finite-depth empirical demonstrations before presenting an asymptotic variant that connects to our theoretical analysis.
In particular, the Stage 1 dynamics Eq.~(\ref{upd:discr}) can be implemented as a looped attention layer with $W_Q=W_K=\sqrt{\beta}I_d$ and $W_V=I_d$ and a Gaussian kernel \citep{chen2021skyformer}. 
Appendix~\ref{app:attention_arch_details} explains architectural aspects including its relation to dot-product attention and the convex hull invariance property of the leaky residual update.

Fig.~\ref{fig:gaussian_linear_variance_decay}
studies Stage 1 variance dynamics for the canonical problem of denoising a Gaussian distribution. 
The predicted reverse-diffusion behavior is linear variance decay $dv/dt = -2$ over the horizon $[0, T^*]$ with $T^* = \sigma^2/2$. 
Empirically, only intermediate values of kernel bandwidth $\beta$ recover the diffusion-like linear variance decay and stopping time predicted by theory. Large $\beta$ induces premature clustering at finite $N$, while small $\beta$ produces overly global averaging that misses the local denoising structure and leads to slower denoising. The slowed $\beta = 1$ behavior is captured by the solution to an ODE $v'(t) = -2v / (v + \beta^{-1})$ arising from a perturbative analysis 
(Appendix~\ref{sec:stoppingtime}, Theorem~\ref{thm:gaussiandenoise}).

Our iterative attention mechanism uses a single time-independent bandwidth $\beta$, whereas the score-based diffusion construction of \citet{rosu2025_denoising} (Section~4.1) uses a depth-dependent noise scale. Correct denoising is recovered provided the depth $L$ and step size $\eta$ are scaled according to noise $\sigma^2$.

\begin{figure}[t]
\centering
\includegraphics[width=0.77\textwidth]
{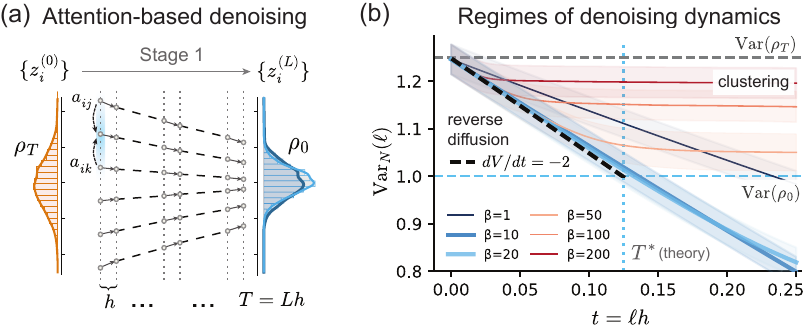}
\caption{
  \textbf{Attention as discretized reverse diffusion.}
  (a) Particles initialized from a noised distribution $\rho_T$ are iteratively updated across layers using Gaussian attention steps ($t=\ell h$) to approximately recover the clean prior $\rho_0 = \mathcal{N}(0, 1)$. 
  (b) Variance dynamics for Gaussian denoising under different kernel bandwidths $\beta$. 
  Parameters: 20 seeds, $N=5000, L_0=200, \sigma^2=0.25$.
  }
\label{fig:gaussian_linear_variance_decay}
\end{figure}

% ----------------------------------------------------
\subsection{In-context learning of particle priors and energy landscape of posterior inference}
\label{sec:4.1}
% ----------------------------------------------------

\begin{figure*}[h!]
\centering
\includegraphics[width=0.99\textwidth]
{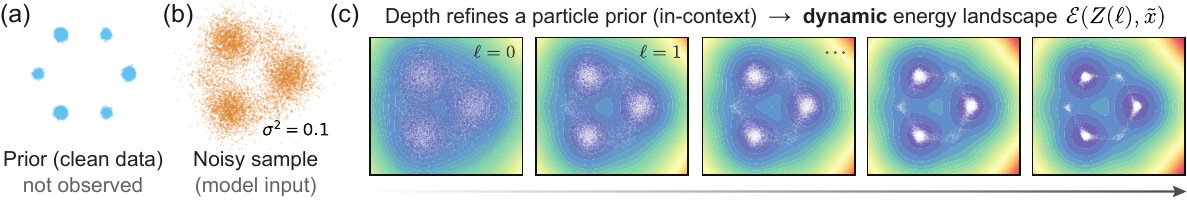}
\caption{
  \textbf{Dynamic energy landscape.}
  (a) Ground truth GMM prior. (b) Corrupted input prompt ($N=5,000$). (c) Dynamic refinement: As model depth $\ell$ increases, Stage 1 self-attention iteratively sharpens the particle prior (white cloud). This process dynamically sculpts an associative memory landscape $\mathcal{E}(Z(\ell), \tilde{x})$ which determines the posterior averaging step in Stage 2; Eq. (\ref{eq:energy_DAM}).
  }
\label{fig:stage1_vs_stage2_energy_mechanism}
\end{figure*}

Fig. \ref{fig:gaussian_linear_variance_decay}(b) provides numerical indication that multilayer attention can denoise distributions (Stage~1) when model depth and parameters are in line with the theoretical predictions. We next study the downstream role of Stage 2 for sample denoising ($x$-prediction). For such posterior inference tasks, note that the prediction step of Alg. \ref{alg:two_stage_multilayer_attn} (details in App. \ref{app:attention_arch_details}) can be interpreted as in \cite{smart2025context}: 
each noisy query $\tilde x_i$ is updated by taking a gradient step on an energy landscape defined by the refined particle approximation
\begin{equation} 
\label{eq:energy_DAM}
\hat x_i = \tilde x_i - \nabla_q \mathcal{E}(Z(\ell); q) \big|_{q = \tilde x_i},
\qquad
\mathcal{E}(Z(\ell); q) = -\frac{1}{\beta_c} \log \sum_j e^{-\frac{\beta_c}{2}\|q - z_j^{(\ell)}\|^2}.
\end{equation}

This defines a dense associative memory \citep{krotov2016dense, ramsauer2021iclr} induced by the context. In contrast to \citet{smart2025context}, where the memories are fixed during inference, here Stage 1 dynamically \emph{sculpts} the landscape across depth: as the particle prior is refined, the energy surface on which the noisy query descends sharpens correspondingly (Fig.~\ref{fig:stage1_vs_stage2_energy_mechanism}). 
We next examine the role of depth-dependent refinement on approaching Bayes optimality on the all-token corruption task.

% ----------------------------------------------------
\subsection{Roles of context length, depth, and cross-attention for denoising all-token corruption}
% ----------------------------------------------------

While Stage 1 can recover the underlying structure of the data, it does not by itself yield the Bayes-optimal estimator. 
Using a symmetric two-component Gaussian mixture prior, we observe that improvements from depth are pronounced in regimes where the posterior distribution is ambiguous across multiple modes (Fig. \ref{fig:stage1_vs_stage2_bayes_optimal}). 
In this minimal setting, the Bayes optimal predictor (exact posterior mean given the clean prior) is analytically tractable (see Appendix \ref{sec:app_bayes_opt_GMM_formula}). This example establishes a key validation: the combination of Stage 1 prior refinement and Stage 2 posterior averaging approaches the Bayes optimal bound as a function of model depth. Additional numerics in Fig. \ref{fig:app_stage1_vs_stage2_bayes_optimal} of App. \ref{app:additional_numerics}.

\begin{figure*}[h!]
\centering
\includegraphics[width=0.95\textwidth]{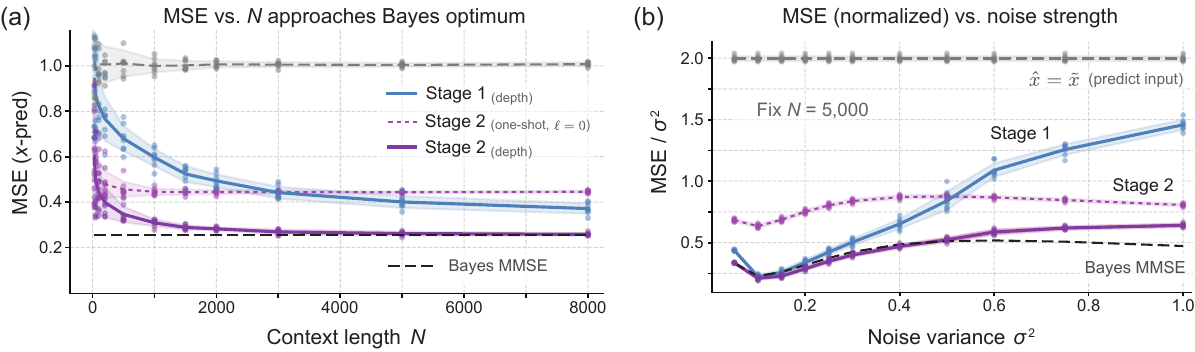}
\caption{
  Multilayer attention-based denoiser approaches Bayes-optimal $x$-pred MSE with increasing context and depth.
  (a) MSE vs. context length (particle count) $N$. (b) MSE (normalized) vs. noise variance $\sigma^2$. 
  Solid lines show the mean (bands $\pm 1$ standard deviation) over 8 seeds (with subsampling preserved across $N$).  
  Prior: Symmetric Gaussian mixture with $\mu=(\pm 1, 0)$ and $a=0.1$.
  Parameters:
  (a) $\beta=20$, $\sigma^2=0.5$, $L_0=200$; 
  (b) $N=5000$ is fixed for varying $\sigma^2$. 
  The one-shot \(\ell=0\) curve is an ablation applying the same Stage 2 readout directly to the noisy context (i.e., no Stage 1 loop), isolating the effect of prior refinement.
  Details in Appendix \ref{app:attention_arch_details}.
  }
\label{fig:stage1_vs_stage2_bayes_optimal}
\end{figure*}

% ====================================================
% ----------------------------------------------------
\subsection{Attention dynamics corresponding to asymptotic analysis (hard truncation)}
% ----------------------------------------------------

Algorithm \ref{alg:two_stage_multilayer_attn} and the associated numeric demonstrations are equivalent to using Euler steps on a neural ODE \citep{chen2018neural} acting on a finite context length. This discrete dynamics admits a continuous-time mean-field limit, which we formalize via an ODE governing particle evolution. Algorithm~\ref{alg:full-sample-sequential-posterior} is the corresponding practical full-sample version. For technical reasons, our analysis is carried out on a truncated version of this flow to ensure compact support. Algorithm~\ref{alg:truncated-sequential-posterior} is the theorem-certified version.  
For Gaussian noisy laws, the tail estimate shows that taking \(R=R_N\) slightly larger than \(\sqrt{\log N}\) makes the full-sample and truncated procedures agree with high probability.
We refer to Appendix~\ref{subsec:pseudocode} for the precise algorithms.

% ====================================================
\section{Sequential posterior-mean recovery}
\label{sec:sequential-posterior-mean-recovery}
% ====================================================

We now state the recovery guarantee underlying the Stage~2 posterior-mean estimator.  Recall that the clean law is denoted by \(P_0\), and the observation model is
\[
Y=X+\sqrt{\tau}\,Z,
\qquad
X\sim P_0,\qquad Z\sim \mathcal N(0,I_d),
\qquad \tau=\sigma^2.
\]
Thus the noisy one-particle law is
\[
f_0=\gamma_\tau*P_0,
\qquad
\gamma_\tau(z)=(2\pi\tau)^{-d/2}e^{-|z|^2/(2\tau)}.
\]
For any candidate prior \(\nu\), define the Gaussian posterior-mean map
\begin{equation}\label{eq:main-posterior-mean-map}
m_\nu(y)
:=
\frac{\int_{\mathbb R^d} x\,\gamma_\tau(y-x)\,\nu(dx)}
{\int_{\mathbb R^d}\gamma_\tau(y-x)\,\nu(dx)}.    
\end{equation}
For \(\nu=P_0\), this is the Bayes-optimal square-loss denoiser,
\[
m_{P_0}(y)=\mathbb E[X\mid Y=y].
\]
The goal of the sequential construction is therefore to recover the function
\(m_{P_0}\) asymptotically from noisy samples.
Let
\[
G_\beta(z)=\Bigl(\frac{\beta}{2\pi}\Bigr)^{d/2}
e^{-\frac{\beta}{2}|z|^2}
=
\mathcal N(0,\beta^{-1}I_d),
\]
and define the exact Gaussian attention drift
\begin{equation}\label{eq:main-exact-gaussian-attention-drift1}
X_\beta[\mu](x)
:=
\frac1\beta\nabla\log(G_\beta*\mu)(x).
\end{equation}
Equivalently,
\begin{equation}\label{eq:main-exact-gaussian-attention-drift2}
X_\beta[\mu](x)=F_{\beta,\mu}(x)-x,
\qquad
F_{\beta,\mu}(x)
:=
\frac{\int z\,G_\beta(x-z)\,\mu(dz)}
{\int G_\beta(x-z)\,\mu(dz)}.
\end{equation}
Thus \(X_\beta[\mu]\) is precisely the continuous-depth Gaussian-attention
barycentric update.

Appendix~\ref{subsec:recoverable-admissible-priors}
defines a class
\(\mathcal A_\tau\subset\mathcal P_1(\mathbb R^d)\) of
\emph{recoverable admissible priors}.  In main-text terms, a prior
\(P_0\in\mathcal A_\tau\) is one whose noisy law \(f_0=\gamma_\tau*P_0\)
has a well-defined noncompact mean-field flow
\[
\partial_t f_t^\beta
+
\nabla\cdot\bigl(f_t^\beta X_\beta[f_t^\beta]\bigr)
=0,
\qquad
f_{t=0}^\beta=f_0,
\]
which is obtained as the \(R\to\infty\) limit of the corresponding hard-truncated
compactly supported flows, and which recovers the clean prior at a $\beta$-scaled denoising time, $T_\beta:=\frac{\beta\tau}{2}$,
\[
W_1(f_{T_\beta}^{\beta},P_0)\longrightarrow0
\qquad\text{as }\beta\to\infty .
\]
Here $R>0$ is the deterministic hard-truncation radius of the initial noisy law, with $B_R=\{x:|x|\le R\}$. It is used in the proof construction and is not supplied to Algorithm~\ref{alg:two_stage_multilayer_attn}.

The class $\mathcal A_\tau$ packages the two assumptions needed for the theorem: noncompact mean-field admissibility and deterministic recovery of the clean law in the large-$\beta$ denoising limit; propagation of chaos at fixed $(\beta, R)$ is proved separately in Appendix~\ref{sec:selfcontained-compact-support-posterior-mean}.
%%%
The following theorem is the main-text version of the hard-truncated recovery
result proved as Theorem~\ref{thm:recoverable-prior-hard-truncated-posterior-recovery}
in Appendix~\ref{sec:exact-gaussian-drift-hard-truncation-recovery}.

\begin{theorem}[Sequential posterior-mean recovery]
\label{thm:main-text-sequential-posterior-mean-recovery}
Let \(P_0\in\mathcal A_\tau\), and define
\[
f_0^{[R]}:=\frac{\mathbf 1_{B_R}f_0}{f_0(B_R)}.
\]
Initialize \(N\) particles independently from \(f_0^{[R]}\), and evolve them by
\[
\dot X_i(t)=X_\beta[\mu_t^{N,\beta,[R]}](X_i(t)),
\qquad
\mu_t^{N,\beta,[R]}:=\frac1N\sum_{j=1}^N\delta_{X_j(t)}.
\]
Then, for every bounded observation region \(\{|y|\le M\}\),
\begin{equation}\label{eq:main-sequential-posterior-mean-recovery}
\lim_{\beta\to\infty}
\limsup_{R\to\infty}
\limsup_{N\to\infty}
\mathbb E
\sup_{|y|\le M}
\left|
m_{\mu_{T_\beta}^{N,\beta,[R]}}(y)-m_{P_0}(y)
\right|
=0.
\end{equation}
Consequently, the same convergence holds in probability.
\end{theorem}

The order of limits is part of the statement: first \(N\to\infty\) at fixed
\((\beta,R)\), then \(R\to\infty\) at fixed \(\beta\), and finally
\(\beta\to\infty\).  Thus this is a sequential consistency theorem rather than a
joint scaling law in \((N,R,\beta)\), and it does
not quantify finite-depth Euler discretization error.  
It says that after evolving the noisy samples
by the exact Gaussian-attention dynamics up to depth \(T_\beta=\beta\tau/2\), the
posterior mean computed from the evolved empirical measure converges uniformly on
compact observation regions to the Bayes-optimal denoiser \(m_{P_0}\).

The class \(\mathcal A_\tau\) is nonempty.  In particular,
Appendix~\ref{sec:gaussian-prior-posterior-mean-recovery} proves that every Gaussian
prior
\[
P_0=\mathcal N(a,\Sigma_0),
\]
with \(a\in\mathbb R^d\) and \(\Sigma_0\) symmetric nonnegative semidefinite, belongs
to \(\mathcal A_\tau\).  Indeed, if
\[
f_0=\gamma_\tau*P_0=\mathcal N(a,\Gamma_0),
\qquad
\Gamma_0=\Sigma_0+\tau I_d,
\]
then the deterministic mean-field flow remains Gaussian:
\[
f_t^\beta=\mathcal N(a,\Gamma_t),
\qquad
\dot\Gamma_t=-2\Gamma_t(I+\beta\Gamma_t)^{-1}.
\]
At \(T_\beta=\beta\tau/2\), \(\Gamma_{T_\beta}\to\Sigma_0\), hence
\(f_{T_\beta}^\beta\to P_0\) in \(W_1\).  
Together with convergence under removal of hard truncation
proved in Appendix~\ref{sec:exact-gaussian-drift-hard-truncation-recovery}, this shows that Gaussian priors satisfy
the full sequential posterior-mean recovery theorem.

For Gaussian noisy data, the hard-truncation loss is explicit.  If
\(f_0=\mathcal N(a,\Gamma_0)\) and
\(\lambda_+=\lambda_{\max}(\Gamma_0)\), then for \(R\ge 1+2|a|\),
\[
f_0(B_R^c)
\le
C(1+R)^d
\exp\!\left(-\frac{R^2}{8\lambda_+}\right).
\]
Therefore, for untruncated noisy samples \(Y_1,\ldots,Y_N\sim f_0\),
\[
\mathbb P(Y_1,\ldots,Y_N\in B_R)
\ge
1-
CN(1+R)^d
\exp\!\left(-\frac{R^2}{8\lambda_+}\right).
\]
Choosing \(R=R_N\) slightly larger than \(\sqrt{\log N}\) makes this loss probability vanish.  This relates the practical
full-sample construction Alg. \ref{alg:full-sample-sequential-posterior} to the theorem-certified truncated construction Alg. \ref{alg:truncated-sequential-posterior}.

% ====================================================
\section{Related work}
% ====================================================

\textbf{Regimes of in-context denoising.}
\citet{smart2025context} introduced in-context denoising for attention networks, showing that a single layer solves it when context tokens are clean. The present work extends this to the \emph{all-token corruption} regime, where every token is noisy and representations evolve jointly across depth, requiring the network to infer the prior from corrupted samples. Concurrent work by \citet{rosu2025_denoising} studies architectural equivalences between attention and denoising algorithms, including manifold denoising \citep{hein2006manifold, belkin2003laplacian}. %of this regime. 
Our analysis is complementary: a finite-sample empirical Bayes interpretation, a two-stage decomposition in which a long-range residual preserves the original noisy query for posterior averaging against the refined particles, and an explicit depth--noise relation $T^* = \sigma^2/2$ under fixed timesteps.

\textbf{Dynamics of attention layers.}
A growing body of work studies multilayer attention through interacting-particle and mean-field descriptions ~\citep{geshkovski2023emergence, geshkovski2025mathematical, rigollet2025mean, bruno2025_neurips_meanfield, burger2025roysoc}, with analyses of Wasserstein gradient flows and diffusion limits for Sinkhorn attention \citep{sander2022sinkformers}, multiscale heat-flow and metastable clustering regimes \citep{bruno2025_neurips_meanfield}, and unrolled subspace denoising \citep{wang2025attention_unrolled_denoising}. \citet{castin2026unified_v2-2026} provide a unified PDE treatment of several attention variants, establishing mean-field limits for compactly supported data and Gaussian-preserving dynamics. 
Our in-context denoising formulation connects these dynamics to a noise-calibrated finite horizon and a separate posterior readout. Empirically, particle concentration and clustering can support inference (suggesting a task-specific role for the metastable clustering emphasized in prior work), whereas continued iteration ultimately leads to collapse and can degrade inference.

\textbf{Empirical Bayes and score-based denoising.}
Classical empirical Bayes methods study posterior estimation under unknown priors from repeated observations \citep{robbins1956empirical, miyasawa1961empirical, stein1981estimation, efron2011tweedie, JohnstoneSilverman2005, raphan2011_simoncelli}. Under Gaussian corruption, these estimators are closely connected to score-based denoising and diffusion formulations through Tweedie's formula \citep{robbins1956empirical, miyasawa1961empirical, efron2011tweedie, vincent2011ieee, ho2020denoising, song2021scorebased}. Kernel approximation of this connection \citep{FukunagaHostetler1975} can be represented by Gaussian attention, which is leveraged here as well as in recent work \citep{rosu2025_denoising, ilin2026discoformerplugindensityscore}. Related transport-based perspectives include constrained EB denoising via optimal transport \citep{jaffe2025constrained} and stochastic interpolants for generative modeling from corrupted observations \citep{modi2025generative}. Our contribution within this lineage is to identify a specifically in-context empirical Bayes role for multilayer attention, where depth performs finite-sample refinement of a particle prior and attention residuals implement posterior averaging against the refined representation.

\textbf{Transformers and empirical Bayes.}
Recent work shows that transformers can be trained to solve empirical Bayes problems in an in-context manner \citep{teh2025solving}.
Our approach is complementary: we provide a mechanistic interpretation of multilayer attention as implementing internal prior refinement followed by posterior averaging via a long-range attention residual connection. 
\citet{wakayama2026incontext} develop a finite-sample theory of Bayesian ICL under mixtures of regression task families, while \citet{ma2026optimalincontextadaptivitydistributional} establish task adaptivity and robustness to bounded test-time distribution shift under a pretraining prior that is a mixture over ICL tasks. Our empirical Bayes formulation instead treats the clean-token distribution $P_0$ as unknown and constructs a particle approximation from corrupted context tokens, without specifying a meta-prior over possible $P_0$'s.
%%%
More broadly, our collective in-context denoising formulation instantiates the perspective that transformer depth executes an implicit algorithm in the forward pass \citep{von2023transformers}. 

% ====================================================
\section{Discussion}
% ====================================================
Our collective in-context denoising framework decomposes attention-based denoising into two stages: iterative refinement of a particle prior (Stage 1) and posterior averaging via a long-range attention residual (Stage 2). The cross-attention from the noisy input to a dynamically refined particle prior distinguishes empirical Bayes denoising from iterative score-based denoising \citep{rosu2025_denoising, bruno2025_neurips_meanfield} and from the clean-context setting of \citet{smart2025context}, where Stage 1 is unnecessary and Stage 2 recovers the exact Bayes posterior mean. 
In the population limit, 
Tweedie's formula expresses the Bayes posterior mean through the score of the noisy marginal, without explicitly estimating the prior \citep{miyasawa1961empirical, raphan2006_simoncelli, raphan2011_simoncelli}; depth in our framework instead constructs a particle representation of the prior from a finite noisy context for subsequent posterior inference.
(cf.\ the $f$- vs.\ $g$-modeling strategies of \citet{efron2014two}).
The prior-refinement step (Stage 1) is also useful independently of posterior averaging, e.g., for representation learning and distribution-level denoising \citep{hein2006manifold}.

The general setting connects to but differs from diffusion models such as DDPM \citep{ho2020denoising} in two ways. First, our denoiser is nonparametric and context-dependent: the score is estimated directly in-context via attention (no time-dependent score is learned). Second, where diffusion models traverse a full trajectory of intermediate distributions $\rho_t$ via a noise schedule, we target recovery from a single corruption level $\sigma^2$, for which a fixed-bandwidth kernel and finite horizon $T^* = \sigma^2/2$ suffice. The convergence of these two paradigms, transformer-based sequence models and diffusion-style denoising, is increasingly visible in practical architectures (see DiT \citep{Peebles_2023_ICCV_DiT}, SiT \citep{Ma2024_SiT_Albergo_Boffi_EVE}, and JiT \citep{li2026basicsletdenoisinggenerative}). Our framework offers one statistical account of why this convergence is principled, showing that attention natively implements a finite-sample empirical Bayes operator for iterative denoising.

A growing body of work interprets attention through the lens of associative memory and energy-based models \citep{ramsauer2021iclr, smart2025context}, drawing on ideas from statistical physics \citep{krotov2016dense}. From the perspective of a single attention head, however, it is not clear where the ``memories'' encoded in keys and values should come from. We show that a deep, single-head attention network can construct such memories \emph{in-context}: the particle prior refined by Stage 1 acts as a working memory that is itself shaped by collective dynamics, then read out by Stage 2 for posterior inference. This picture in which the energy landscape is not fixed but emerges from coupled interactions among many agents has parallels in multicellular self-organization \citep{smart2023emergent}, and connects to recent looped- and energy-transformer formulations \citep{yang2024looped, saunshi2025reasoning, hoover2023energytransformer, dehmamy2026nrgpt, gladstone2026energybased}.
An open question is whether the finite denoising horizon analyzed here can inform depth selection for recurrent attention networks more broadly.

\paragraph{Outlook.} 
Several practical directions are natural follow-ups. Efficient approximations to the kernel sum, such as Nystr\"om-style schemes \citep{xiong2021nystromformer, rosu2025_denoising} or fast Gaussian-kernel approximations \citep{greengard1991fast}, could broaden applicability, while lower-dimensional embeddings analogous to latent diffusion \citep{Rombach_2022_CVPR} could improve convergence rates. 
Unlike standard settings, however, the kernel geometry here is not fixed: the particle prior and associated energy landscape evolve in-context with depth. Extending acceleration and compression schemes to this dynamically evolving collective representation remains an interesting problem.

More conceptually, the two-stage decomposition suggests an in-context amortized inference perspective: once the particle prior is refined through Stage 1, new queries can be processed cheaply via Stage 2 without recomputing the full dynamics. We leave empirical evaluation of this reuse strategy to future work.
The value of cross-depth attention in our framework also resonates with recent empirical work showing that depth-wise attention residuals improve training stability and downstream performance \citep{kimiteam2026attentionresiduals}; our framework provides one statistical reason for this benefit, since the original corrupted input carries token-identity information that posterior inference preserves. More broadly, the empirical Bayes lens on attention --- in which depth refines a particle approximation to an unknown prior --- may offer insight into the dynamic latent representations that emerge in deep generative models more generally.

\subsection*{Limitations}

\textbf{Theoretical considerations.} Our analysis relies on structural assumptions on the prior, formalized through the admissible class $\mathcal{A}_\tau$ that may be restrictive in practice. The theory does not establish propagation of chaos in full generality, and instead proceeds via a truncated mean-field construction with sequential limits. In particular, we do not obtain finite-sample convergence rates, joint scaling laws in $(N, R, \beta)$, or guarantees for the practical untruncated algorithm.
The posterior-recovery theorem is proved only for priors in $\mathcal{A}_\tau$. 
Gaussian priors are verified as a concrete case, while heavy-tailed, singular, or more general priors remain open. The proof relies on a hard-truncated compact-support flow with sequential limits $N\to\infty$ at fixed $(\beta,R)$, then $R\to\infty$, and finally $\beta\to\infty$; although Gaussian tail bounds make truncation negligible with high probability for $R\gtrsim\sqrt{\log N}$. Extending these results to broader classes of priors, establishing finite-sample rates, and developing joint scaling laws are natural directions for future work.

\textbf{Architectural scope.} The analysis applies to a minimal attention architecture: single-head, scaled-identity weights, tied parameters across depth, and no MLP blocks or positional embedding. We further neglect layer norm which motivates our use of the Gaussian attention kernel \citep{chen2021skyformer}.
Beyond these and other design choices, practical transformers learn their enormous weight vectors by training on fixed datasets. 
We do not establish that trained multi-head Transformers with learned $Q/K/V$ maps implement the same mechanism.
We conjecture the empirical Bayes lens extends to such settings, with learned weights potentially implementing approximations to the kernel structure analyzed here. A rigorous treatment would first require extending the analysis to anisotropic weights (cf. \cite{bruno2025_neurips_meanfield, rosu2025_denoising}) which should then also vary with depth. It will be particularly interesting to characterize how MLP layers augment the finite-$N$ particle dynamics studied here and elsewhere.

\textbf{Extensions left to future work.} 
Allowing parameters to vary across depth during Stage 1 opens the door to noise-schedule designs from diffusion models. Further, multi-head architectures with distinct scales $\{\beta_h\}$ could enable multi-scale particle refinement, potentially helpful for more complex priors.

Our framework treats the noise level $\sigma^2$ as input, as it informs the integration time and sets the Stage 2 scale $\beta_c = 1/\sigma^2$. Inferring $\sigma^2$ from corrupted samples in-context is an interesting challenge. We note that without restrictions on $P_0$, the clean distribution and noise variance are not jointly identifiable from the noisy marginal: e.g., an observed law $\mathcal N(0,vI_d)$ is compatible with $P_0=\mathcal N(0,(v-\sigma^2)I_d)$ for any $0<\sigma^2<v$. 
Inferring the noise level in-context therefore requires
additional information or structural assumptions on the clean prior; see e.g. \citet{kadkhodaie2026_blinddenoising}.

Finally, the theoretical horizon $T^*=\sigma^2/2$ is a leading-order prediction. Required depth in finite-$N$ settings depends on $N$, $\beta$, and prior structure; we do not claim that a single fixed, noise-unconditioned depth is optimal across noise levels. This motivates adaptive depth-selection schemes that could be based on input characteristics or clean reference samples from the prior, if available.

\paragraph*{Broader impacts:}
This work is primarily theoretical and studies attention dynamics in controlled synthetic settings. Potential positive impacts include improved understanding of transformer architectures and principled inference mechanisms. Potential negative impacts are indirect, as broader advances in generative modeling could be misused for tasks such as synthetic content generation.

\textbf{Code availability.}
Code and instructions for reproducing the numerical experiments are available at \url{https://github.com/mattsmart/collective-in-context-denoising}.

\begin{ack}
We thank Alberto Bietti, Joan Bruna, Bob Carpenter, Suryanarayana Maddu, Ilya Nemenman, and Gautam Reddy for valuable discussions and feedback. M.S. acknowledges support from the Simons Foundation. 
A.V.M. and A.M.S. acknowledge financial and logistical support from the Center for Quantitative Biology, Rutgers University. 
The authors declare no competing interests.
\end{ack}

\newpage

% -----
% The natbib package will be loaded for you by default. Citations may be author/year or numeric, as long as you maintain internal consistency. 
% -----
\bibliographystyle{plainnat}
%\bibliographystyle{unsrtnat}
%\bibliographystyle{unsrt}
%\bibliography{transformer_dynamics}

%%%%%%%%%%%%%%%%%%%%%%%%%%%%%%%%%%%%%%%%%%%%%%%%%%%%%%%%%%%%

\appendix

% ====================================
\section{The informal argument for Empirical Bayes, Stage 1, via reverse diffusion}
\label{app:informal_argument}
% ====================================

Before we give the informal arguments for the Proposition \ref{prop:informal_reverse_diffusion_derivation}, let us indicate its relationship to the formal Theorem \ref{thm:main-text-sequential-posterior-mean-recovery}:
\begin{itemize}
\item \textbf{Proposition~\ref{prop:informal_reverse_diffusion_derivation} (Informal particle dynamics).}
    Provides a heuristic description of attention updates as a finite-step particle system, where each update corresponds to kernel-weighted averaging. This perspective connects attention to reverse diffusion, kernel regression, and associative memory dynamics.

    \item \textbf{Algorithm~\ref{alg:two_stage_multilayer_attn} (Discrete attention dynamics).}
    Instantiates the informal dynamics of Proposition~3.1 as a finite-depth, finite-sample procedure. It can be viewed as an Euler discretization of an underlying continuous-time particle flow, with step size $\eta$ and depth $L$ determining the effective integration horizon.

    \item \textbf{Algorithm~\ref{alg:full-sample-sequential-posterior} (Continuous-time flow).}
    Provides a continuum formulation of the dynamics underlying Algorithm~1, in which the empirical particle distribution evolves according to a Gaussian-attention drift
    $X_\beta[\mu] = \beta^{-1} \nabla \log (G_\beta \ast \mu)$, where $\mu$ is the empirical distribution of the particles. 

    \item \textbf{Algorithm~\ref{alg:truncated-sequential-posterior} (Truncated, theorem-certified dynamics).}
    Introduces the hard-truncated version of the empirical flow, obtained by restricting the initial noisy sample to a radius-\(R\) ball and then evolving by the same Gaussian-attention drift. The rigorous theorem is stated for \(N\) i.i.d. samples from the normalized truncated law \(f_0^{[R]}\); the algorithmic filtering version is the practical finite-sample analogue. Compact support is preserved by the flow, which makes the propagation-of-chaos analysis available.

    \item \textbf{Theorem~\ref{thm:main-text-sequential-posterior-mean-recovery} (Posterior-mean recovery).}
   Establishes that, for recoverable admissible priors and under appropriate sequential limits the posterior mean computed from the hard-truncated evolved empirical measure converges uniformly on compact observation regions to the Bayes-optimal predictor \(m_{P_0}(y)=\mathbb E[X\mid Y=y]\).
\end{itemize}

Now we discuss the steps towards Proposition \ref{prop:informal_reverse_diffusion_derivation}.

\subsection{Noise addition as forward diffusion} 
Let $X_0 \sim P_0$ be a distribution on $\mathbb{R}^d$ with density $\rho_0$.
Consider the forward diffusion over the time interval $s \in [0,T]$:
\begin{equation} \label{rho:s}
\partial_s \rho_s = \Delta \rho_s,
\end{equation}
where $\rho_s$ is the density at time $s$, with the initial condition
$\rho_{s=0} = \rho_0$. Eq.~\eqref{rho:s} is solved by
\begin{equation} 
\rho_s = \rho_0 * \mathcal{N}(0, 2sI). 
\end{equation} 

\subsection{Backward-time formulation} 
We introduce backward time:
\begin{equation} 
t = T - s, \qquad f_t := \rho_{T-t}. 
\end{equation} 
Then $f_t$ satisfies 
\begin{equation} \label{eq:dynamics}
\partial_t f_t = -\Delta f_t = -\nabla \cdot \bigl(f_t \nabla \log f_t\bigr). 
\end{equation} 
Comparing Eq.~\eqref{eq:dynamics} with the continuity equation, $\partial_tf_t+\nabla (f_t v_t)=0$, gives the velocity field as $v_t(x) = \nabla \log f_t(x)$. Thus we may view the diffusive flow as the transport of a cloud
of particles with instantaneous velocities determined by $v_t = \nabla \log f_t$.
If a particle has position $Z_t$ at time $t$, with the probability density $f_t$, its trajectory follows the deterministic probability flow~\citep{song2021scorebased}:
\begin{equation} \label{eq:particleflow}
\dot Z_t = \nabla \log f_t(Z_t), 
\end{equation} 
which transports $f_0 = \rho_T$ backwards in time toward $f_T = \rho_0$. Under this formulation, the drift is governed by the score function $\nabla \log f_t$.
The set of trajectories $\{Z_t\}$ constitutes a particle approximation to the backward density evolution~\citep{del2013mean}, effectively acting as a normalizing flow~\citep{rezende2015variational} which maps the diffused measure back to the data distribution.

We introduce the Gaussian-kernel dynamics: 
$$\partial_tf_t = -\nabla\cdot \left(f_t\frac{\nabla(G_\beta*f_t)}{G_\beta*f_t}\right) = -\nabla\cdot\left(f_t\nabla\log(G_\beta*f_t)\right), \qquad t\ge 0$$
where $G_\beta$
is the Gaussian distribution with mean $0$ and covariance $\beta^{-1}I$. This is an intermediate model because the original backward-time dynamics is expressed in terms of the exact score field $\nabla \log f_t$, which depends on the unknown population density itself and is therefore not directly accessible to a finite attention mechanism. In this case, the velocity field is $v_t = \nabla \log(G_\beta * f_t)$; replacing $\nabla \log f_t$ with the smoothed score $\nabla \log (G_\beta * f_t)$ rewrites the denoising flow in terms of Gaussian local averages.

By Proposition~\ref{prop:sc-barycentric-drift}, the corresponding flow equation (Eq.~\eqref{eq:particleflow}) becomes
\begin{equation} \label{flow:baby}
\dot{Z_t} = \beta\left(F_{\beta,f_t}(Z_t)-Z_t\right),
\end{equation}
where
\begin{equation} \label{F:beta_rho}
F_{\beta,\rho}(y) = \frac{\int_{\mathbb R^d} x \exp\left(-\frac{\beta\|y-x\|^2}{2}\right)\rho(x)dx}{\int_{\mathbb R^d} \exp\left(-\frac{\beta\|y-x\|^2}{2}\right)\rho(x)dx}.
\end{equation}
When the observation-noise variance is \(\tau=\beta^{-1}\), $F_{\tau^{-1},\rho}=m_\rho$, the Gaussian posterior mean for prior \(\rho\).

Next, we treat the $f_t$'s as the empirical measure on the time-evolved particles in discrete time. Consider $N$ particles at time $t$ with positions $\{Z_t^{(i)}\}_{i=1}^N$ and the corresponding empirical measure $\displaystyle \nu_N^{(t)} = \frac1N \sum_{i=1}^N \delta_{Z_i^{(t)}}$. Then the particle positions at time $t+\eta/\beta$ are given by:
\begin{equation} \label{Z:onestep}
Z_{t+\eta/\beta}^{(i)} -Z_{t}^{(i)} = \eta\left(F_{\beta,\nu_N^{(t)}}(Z_{t}^{(i)})-Z_{t}^{(i)}\right).
\end{equation}
If the total time budget is $T$ and passing through one layer corresponds to moving $h=\eta/\beta$ in time, $L\eta/\beta=T$, so that the total number of layers is $L = T\beta/\eta$ and the layer-wise dynamics becomes
$$Z_{\ell+1}^{(i)} = (1-\eta)Z_{\ell}^{(i)} + \eta F_{\beta, \nu_N^{(\ell)}} (Z_{\ell}^{(i)}).$$
In this way, the Gaussian-kernel formulation preserves the denoising interpretation while making the dynamics compatible with an empirical particle approximation. It also regularizes the ideal backward-heat flow, which is formally anti-diffusive and unstable, by replacing the raw score with a modified one. Finally, in the large-$\beta$ regime, the Gaussian attention barycenter satisfies $F_{\beta,\rho}(x)-x \approx \beta^{-1} \nabla \log \rho(x)$, and the corresponding nonlocal PDE reduces, after rescaling time, to the original backward-heat equation.
Thus the Gaussian-kernel dynamics is meant to provide the correct bridge from the exact continuum denoising law to the finite-sample attention dynamics implemented via particles.

%%%%%%%%%%%%%%

\subsection{Tweedie's formula} 
Let $\tilde X = X + \varepsilon$, where $\varepsilon \sim \mathcal{N}(0,\sigma^2 I_d)$ and $X \sim \rho$. Then 
$$
\mathbb{E}[X \mid \tilde X = \tilde x] = \tilde x + \sigma^2 \nabla \log \tilde\rho(\tilde x), 
$$
where $\mathbb{E}[X \mid \tilde X = \tilde x] = F_{\sigma^{-2}, \rho} (\tilde x)$, and 
$$
\nabla \log \tilde\rho(\tilde x) = \frac{1}{\sigma^2}
\bigl(\mathbb{E}[X \mid \tilde X = \tilde x] - \tilde x \bigr). 
$$
Thus the score is a denoising residual. 

\subsection{Local denoising along the diffusion} 
Fix backward time $t$ and let $s = T-t$. Over a short forward time increment $\delta>0$, 
$$
X_{s} = X_{s-\delta} + \varepsilon, \qquad \varepsilon \sim \mathcal{N}(0,2\delta I).
$$
Equivalently, 
$$
f_t = f_{t+\delta} * \mathcal{N}(0,2\delta I). 
$$
Applying Tweedie's formula yields the local identity 
$$
\nabla \log f_t(x) = \frac{1}{2\delta} \Bigl( \mathbb{E}[Y \mid X = x] - x \Bigr), 
$$
where $Y \sim f_{t+\delta}$ and $X = Y + \varepsilon \sim f_t$.

\subsection{Particle approximation} 
We discretize backward time: 
$$
t_\ell = \ell h, \qquad h = \frac{T}{L}. 
$$
Let $\{z_i^{(\ell)}\}_{i=1}^N$ approximate $f_{t_\ell}$. The reverse flow gives 
$$
z_i^{(\ell+1)} = z_i^{(\ell)} + h \nabla \log f_{t_\ell} (z_i^{(\ell)}). 
$$ 
Using the denoising identity, we instead write the update as 
$$ z_i^{(\ell+1)} = (1-\eta) z_i^{(\ell)} + \eta\, \mathbb{E}[Y \mid X = z_i^{(\ell)}], \qquad \eta = \frac{h}{2\delta}. 
$$

\subsection{Kernel approximation} 
Since $X = Y + \varepsilon$ with Gaussian noise, we have:
$$
\mathbb{E}[Y \mid X = z] = \frac{ \int y\, e^{-\frac{\|z-y\|^2}{4\delta}} f_{t+\delta}(y) dy }{ \int e^{-\frac{\|z-y\|^2}{4\delta}} f_{t+\delta}(y) dy }. 
$$
Approximating $f_{t+\delta}$ by $f_t$, with further approximation of $f_t$ with particles, yields 
$$
\mathbb{E}[Y \mid X = z] \approx \frac{ \sum_j e^{-\frac{\|z-z_j^{(\ell)}\|^2}{4\delta}} z_j^{(\ell)} }{ \sum_j e^{-\frac{\|z-z_j^{(\ell)}\|^2}{4\delta}} }. 
$$
Let $\beta = {1}/{2\delta}$. The update becomes
\begin{equation} 
z_i^{(\ell+1)} = (1-\eta) z_i^{(\ell)} + \eta \frac{ \sum_j e^{-\frac{\beta}{2}\|z_i^{(\ell)} - z_j^{(\ell)}\|^2} z_j^{(\ell)} }{ \sum_j e^{-\frac{\beta}{2}\|z_i^{(\ell)} - z_j^{(\ell)}\|^2} }. 
\end{equation} 
In this replacement of the integral by the sum,  the condition $\tfrac{N\bar\rho}{\beta^{d/2}} \gg 1$ becomes important (here, $\bar{\rho}$ denotes a typical density, e.g. $\bar{\rho} \simeq \int \rho(x)^2 dx$). Since the Gaussian centered around a particle position is effectively supported by a spherical volume $\sim\beta^{-d/2}$, this occupancy heuristic corresponds to the number of other particles in this volume being large, making the replacement by the sum reasonable. 
The occupancy heuristic does not by itself control the empirical score field; convergence of the resulting particle dynamics in a joint $N$--$\beta$ limit is not established here.
Theorem \ref{thm:main-text-sequential-posterior-mean-recovery}
proceeds by sequential limits and does not rely on this scaling.

%%%%%%%%%%%%%%%%%%%%%%%%%%%

\section{Stability and well-definedness of the Gaussian kernelized score field}

\subsection{Stability under smooth score field}
The advantage of replacing the raw score $\nabla \log f_t$ by $\nabla \log (K_\beta * f_t)$ is stability and well-definedness. It  holds for any smooth positive kernel $K_\beta$.

Let $K_\beta:\mathbb R^d\to(0,\infty)$ be a $\mathcal C^1$ probability density, and define
$$u_f(x):=(K_\beta*f)(x),\qquad
b_f(x):=\nabla \log (K_\beta*f)(x)
=\frac{\nabla(K_\beta*f)(x)}{(K_\beta*f)(x)}.$$

The raw score $\nabla \log f$ is unstable because it may be undefined when $f$ vanishes, and differentiation amplifies microscopic oscillations. It vanishes at most points even if $f$ is the empirical measure. If $f$ is empirical, $\nabla \log f$ is not a classical vector field.

If $K_\beta$ is $\mathcal C^1$ and positive, then for probability measure with density $f$,
$$K_\beta*f\in \mathcal C^1(\mathbb R^d),\qquad \nabla(K_\beta*f)=(\nabla K_\beta)*f,$$
and $K_\beta*f>0$. Hence $b_f$ is globally well-defined.

We also have the following heuristic properties:
\begin{enumerate}
\item On any compact set $\mathcal B_R$,
$$b_f-b_g =
\frac{\nabla u_f-\nabla u_g}{u_f}
+
\nabla u_g\left(\frac{1}{u_f}-\frac{1}{u_g}\right),$$
so if $u_f$ and $u_g$ are bounded below on $\mathcal B_R$, then
$$ \|b_f-b_g\|_{L^\infty(B_R)} \lesssim \|\nabla K_\beta*f-\nabla K_\beta*g\|_{L^\infty(B_R)} + \|K_\beta*f-K_\beta*g\|_{L^\infty(B_R)}. $$
Thus stability of the score field reduces to stability of two linear convolutions.
\item If $K_\beta$ and $\nabla K_\beta$ are Lipschitz, then for probability measures $f,g$,
$$\|K_\beta*f-K_\beta*g\|_{L^\infty(B_R)} + \|\nabla K_\beta*f-\nabla K_\beta*g\|_{L^\infty(B_R)} \lesssim W_1(f,g),$$
hence $ \|b_f-b_g\|_{L^\infty(B_R)}\lesssim W_1(f,g)$ whenever the denominators stay uniformly positive on $\mathcal B_R$.
\item For empirical measures
$$\nu^N=\frac1N\sum_{j=1}^N \delta_{z_j},$$
one has
$$(K_\beta*\nu^N)(x)=\frac1N\sum_{j=1}^N K_\beta(x-z_j),
\qquad
\nabla(K_\beta*\nu^N)(x)=\frac1N\sum_{j=1}^N \nabla K_\beta(x-z_j),$$
so
$$\nabla\log(K_\beta*\nu^N)(x)
=
\frac{\sum_{j=1}^N \nabla K_\beta(x-z_j)}
{\sum_{j=1}^N K_\beta(x-z_j)}.$$
Thus the field is well-defined for empirical laws and varies continuously under perturbations of the particle cloud.
\item Finally, if $K_\beta\in \mathcal C^2$, then on any compact region where $K_\beta*f$ is bounded below, the field $b_f$ is locally Lipschitz in space, since
$$\nabla b_f
=
\frac{\nabla^2(K_\beta*f)}{K_\beta*f}
-
\frac{\nabla(K_\beta*f)\otimes \nabla(K_\beta*f)}{(K_\beta*f)^2}.$$
Hence the ODE
$$\dot Z_t=\nabla \log (K_\beta*f_t)(Z_t)$$
is locally well-posed.
\end{enumerate}

In summary, for any smooth positive kernel $K_\beta$, the map $$f\mapsto \nabla \log (K_\beta*f)$$ is a stable regularization of the raw score: it is well-defined for empirical measures, continuous under perturbations of the law, and suitable for particle and mean-field limits. 

The Gaussian is special only because it adds further exact structures, not because basic stability is unique to it. This is shown in the next subsection.

\subsection{Importance of the Gaussian kernel for smoothing}

Continuing from the previous discussion, let us say that we use a smooth density $K_\beta$. 

We focus on the following two properties:
\begin{enumerate}[label=(A.\Roman*)]
\item \label{ass:p1} There is a positive definite matrix $C_\beta $ (for every positive $\beta$) such that $\mathbb E\left[ X \mid X + e_\beta = y\right] = y + C_\beta \nabla \log (K_\beta * f) (y)$ for every prior $f$ on $X$, where $e_\beta \sim K_\beta$.
\item \label{ass:p2} There is a positive definite matrix $C_\beta$ such that for every empirical measure $$\nu^N = \frac{1}{N}\sum_{i=1}^N\delta_{z_j},$$  the field looks like $$\nabla \log(K_\beta * \nu_N) (x) = C_\beta^{-1}\left[\sum_{j=1}^N w_j(x) z_j - x\right],\qquad w_j(x) = \frac{K_\beta(x-z_j)}{\sum_{i=1}^NK_\beta(x-z_i)}.$$ 
\end{enumerate}

Both of these expressions appear in the paper and the clean formulae help us derive the attention dynamics. Now, assuming either one of \ref{ass:p1} or \ref{ass:p2} above holds, then $K_\beta$ is forced to be a Gaussian with covariance matrix $C_\beta$.

Now that we have $K_\beta$ is Gaussian, we still need $K_\beta * f \to f$ (either in $L^1$ convergence or almost everywhere convergence) if $\beta\to \infty$. This forces all eigenvalues of $C_\beta$ (which are all taken as $\beta^{-1}$ in the paper) to approach $0$ as $\beta\to\infty$. %Again, this can be proven formally. 
This still leaves us with a Gaussian with a covariance matrix $C_\beta$. For simplicity, we use an isotropic Gaussian with $C_\beta = r_\beta I$, where $r_\beta=\beta^{-1}$. Note that by the last paragraph, $\lim_{\beta\to\infty}r_\beta=0$.
%The neat choice $r_\beta=\beta^{-1}$ is the simplest.

% ================================================
\section{Attention architecture for two-stage collective denoising}
\label{app:attention_arch_details}
% ================================================

Here we detail the multilayer attention architecture (Alg. \ref{alg:two_stage_multilayer_attn}) that implements the two stages in Fig. \ref{fig:setup}. When the Stage 1 parameters are fixed across depth, it can be viewed as a minimal loop transformer. 

\textbf{Stage 1.} Given a collection of noisy input tokens $\{\tilde x_i\}_{i=1}^N$, 
the role of Stage 1 is to construct a particle approximation to the prior using Eq.~\eqref{upd:discr}.
This approximation serves as the model's
%We emphasize that this is not solely a preprocessing step for Stage 2; it
internal representation of the unobserved data distribution that supports downstream inference. 
Eq.~\eqref{upd:discr} can be interpreted as applying a Gaussian self-attention kernel \citep{chen2021skyformer} with scaled-identity weights $W_Q = W_K = \sqrt \beta I_d$ and $W_V = I_d$. 
Repeated application of the attention operation produces a dynamic internal representation $Z(\ell) = \{ z_i^{(\ell)} \}_{i=1}^N$, which we use interchangeably with its $N \times d$ matrix form when convenient.
$Z(\ell)$ serves as the model's particle estimate of the prior at depth $\ell$; it provides an estimate of $\hat \rho_0$ at depth $L$ in Fig. \ref{fig:setup}.

The first term of Eq. (\ref{upd:discr}) differs slightly from the canonical attention update $z_i \leftarrow z_i + \textrm{Attn}(W_V Z , W_K Z, W_Q z_i)$. Since $\eta \in (0,1)$, this \emph{leaky} residual update ensures forward invariance of the convex hull of $\{z_i^{(\ell)}\}_{i=1}^N$. Assuming that $\eta$ is fixed across layers (tied), the continuous-depth limit, $\eta \to 0$, $L \to \infty$ with $\eta L=T$ fixed, can be viewed as a neural ODE \citep{chen2018neural}:
\begin{equation}
\label{eq:z_dynamics_discrete_N_flow}
  \dot z_i = -z_i + \sum_{j=1}^N a_{ij}(Z) z_j,
\end{equation}
which is integrated from $t=0$ to $t=T$ with the initial condition $z_i(0) = \tilde x_i ~\forall\, i$. In practice, we use Eq.~\eqref{upd:discr} with a fixed step size, corresponding to the forward Euler discretization of Eq.~\eqref{eq:z_dynamics_discrete_N_flow}.
% up to a rescaling of time.

\textbf{Stage 2.} Once the approximate prior is obtained, the denoised estimates $\hat x_i$ are computed for each noised token $\tilde x_i$ using cross-attention with scale $\beta_c = 1/\sigma^2$ between the corrupted input $z_i(0)=\tilde x_i$ and the refined representation after $\ell$ steps, $Z(\ell)$, where $\ell$ is chosen to approximate the theoretical stopping depth $L$:
\begin{equation}
\label{eq:stage2_cross_attn}
    \hat x_i(\ell) = m_{\beta_c, Z(\ell)}(\tilde x_i) = \frac
    {\sum_j \exp
    \bigl(-\tfrac{\beta_c \|z_j^{(\ell)} - \tilde x_i \|^2}{2}\bigr) 
    z_j(\ell)
    }
    {\sum_k \exp 
      \bigl(-\tfrac{\beta_c\|z_k^{(\ell)} - \tilde x_i \|^2}{2}\bigr)
    }.
\end{equation}
As an architectural element, this is enabled by an additional skip connection from the input to layer $\ell$ (Fig. \ref{fig:setup}(a), dashed path).
This connects to recent work identifying the empirical benefits of residual pathways in attention models, with particular emphasis on AttnRes connections to the embedding layer  \citep{kimiteam2026attentionresiduals}.
Here, the preservation of the original noisy token plays a specific statistical role: it serves as the query in the final posterior estimation step. 
%Thus, rather than being merely a mechanism for stabilizing representations, 
The long-range skip connection implements the separation between data (query) and learned prior (memory) required by the empirical Bayes formulation.

\subsection{Remarks on the two-stage architecture and numerics}

\textbf{Remark (Choice of prior)}. Gaussian mixture models admit closed-form Bayes-optimal estimators (Appendix \ref{sec:app_bayes_opt_GMM_formula}), which provides an unambiguous reference point against which to validate Stage 2 inference. Low-dimensional instantiations permit direct visualization of the internal energy landscape (Fig. \ref{fig:stage1_vs_stage2_energy_mechanism}), allowing us to illustrate the two-stage mechanism. We emphasize however that our framework is not specific to low-dimensional Gaussian mixtures.

\textbf{Remark (Finite discretization and parameter scaling).}
In the finite-depth setting, the parameters $(\sigma, L, \beta)$ cannot be chosen independently if the dynamics are to approximate the reverse diffusion flow. Fixing a noise level $\sigma^2$ determines the total (finite) time horizon $T^* = \sigma^2/2$. Discretizing with $L$ steps gives a step size $h = T^*/L = \sigma^2/(2L)$. Considering \emph{infinitesimal} denoising between adjacent time slices in the reverse diffusion introduces a parameter $\delta = (\beta \sigma^2)^{-1} \ll 1$, requiring $\beta$ to be sufficiently large, and induces an effective step size $\eta={h}/{\sigma^2 \delta}$. Substituting $\delta$ gives
\begin{equation}
\eta = \beta h = \frac{\beta \sigma^2}{2L} < 1,
\end{equation}
which should not be treated as an independent parameter. In particular, consistent approximation of the continuous reverse flow corresponds to regimes with large $L$, large $\beta$, and $\beta/L \to 0$, so that $\eta \to 0$ while $\delta \to 0$. In practice, the numerics fix large $L_0$ and integrate to $3 L_0$ using $h = T^*/L_0$.

\textbf{Remark (Annealing schedule).}
Allowing depth-dependent weights $\eta(t), \beta(t)$ induces a time-inhomogeneous flow, analogous to the noise (or annealing) schedules commonly used in diffusion models. 
In this interpretation, $\beta(t)$ controls the locality of the particle-based score approximation via $\delta(t) = (\beta(t)\sigma^2)^{-1}$, while the discrete flow step $h(t)=\eta(t) / \beta(t)$ satisfies  $\int_0^T h(t)\,dt \approx \sigma^2/2$. In the finite-depth setting, this implies $\sum_{t=1}^L \frac{\eta(t)}{\beta(t)} \approx T$.
We emphasize that the denoising behavior is governed not by $\beta(t)$ or $\eta(t)$ individually, but by their ratio through this finite-time horizon constraint.

\textbf{Remark (Convex hull invariance).} The convex hull of the particle set $\{z_i(t)\}_{i=1}^N$ is positively invariant under Eq.~\eqref{upd:discr} for any $\eta \in (0,1)$, $\beta > 0$. To see this, observe that each particle updates its position as a convex combination of its current position and the attention-weighted barycenter $\sum_j a_{ij}(t) z_j(t)$, which itself lies in the convex hull.

\textbf{Remark (Layer-norm).} We study simplified dynamics without layer-norm/spherical projection. 
Off sphere, distances are relevant and the Gaussian kernel is appropriate. 
%%%
To connect to dot-product attention, 
expand $-\beta\|z_i - z_j\|^2/2$. This retains the key-dependent term $-\beta\|z_j\|^2/2$ in addition to $\beta z_i^\top z_j$, while the query-norm term cancels under row-wise softmax.
%%%
If one projects to the sphere between attention layers, then the Gaussian kernel reduces to softmax dot-product attention.

\textbf{Remark (One-shot Stage 2 readout).}
Write the normalized Stage 2 weights as
\[
b_{ij}^{(\ell)}
=
\frac{\exp[-\frac{\beta_c}{2}
\|\widetilde x_i-z_j^{(\ell)}\|^2]}
{\sum_k \exp[-\frac{\beta_c}{2}
\|\widetilde x_i-z_k^{(\ell)}\|^2]},
\qquad \beta_c=\sigma^{-2}.
\]
Without Stage 1, $z_j^{(0)}=\widetilde x_j$, giving
\begin{equation}
\label{eq:stage2_ablate}
\widehat x_i(0)
=\sum_j b_{ij}^{(0)}\widetilde x_j
=\widetilde x_i+\sum_j b_{ij}^{(0)}
(\widetilde x_j-\widetilde x_i).
\end{equation}
This applies the same Stage 2 readout directly to the unrefined noisy context, and therefore provides an architectural ablation for the role of depth-dependent prior refinement in posterior inference.

\textbf{Remark (Relation to Tweedie estimation).}
Consider the Gaussian KDE 
$\widehat p_h(y)=N^{-1}\sum_j \varphi_{h^2}(y-\widetilde x_j)$, where
$\varphi_{h^2}$ is the $\mathcal N(0,h^2I_d)$ density.
The Gaussian mean-shift identity on Eq.~\eqref{eq:stage2_ablate} gives
\[
\widehat x_i(0)
=\widetilde x_i+\sigma^2
\nabla\log\widehat p_{\sigma}(\widetilde x_i).
\]
This fixed-scale predictor should be distinguished from an asymptotically consistent one-shot Tweedie estimator. Its kernel score estimates the score of an additionally smoothed noisy distribution, so its smoothing bias does not vanish as $N\to\infty$.
A consistent construction should incorporate a distinct kernel bandwidth parameter $h_N^2$ that is separate from the observation variance $\sigma^2$:
\[
\widehat x_i^{\mathrm{Tweedie}}
=\widetilde x_i+\sigma^2
\nabla\log\widehat p_{h_N}(\widetilde x_i).
\]
The bandwidth conditions $h_N\to0$ and $Nh_N^{d+2}\to\infty$ give pointwise consistent estimation of both the noisy density and its gradient \citep[see Sec.~II]{FukunagaHostetler1975}. Their ratio therefore consistently estimates the score $\nabla\log p_{\widetilde X} =\nabla p_{\widetilde X}/p_{\widetilde X}$, since Gaussian corruption makes $p_{\widetilde X}(y)>0$ everywhere. 
Thus, for each fixed query $y$, the estimator converges pointwise as $N\to\infty$ to
$y+\sigma^2\nabla\log p_{\widetilde X}(y)$,
which equals the Bayes-optimal posterior mean $\mathbb E[X\mid\widetilde X=y]$ by Tweedie's formula.

\textbf{Remark (Depth/stepsize).} While we intentionally use small steps (many layers) to match the theory, fewer layers could be used in practice, provided the step size per layer is increased according to $T^*$.

\textbf{Remark (Posterior sufficiency and clustering).} 
While Stage 1 is motivated by recovery of the clean prior $\rho_0$, we empirically observe that accurate posterior inference can be achieved even when the particle distribution concentrates into discrete clusters. This suggests that, for Stage 2 cross-attention, it may not be necessary to reconstruct the full prior density. Instead, it appears sufficient to concentrate particle mass near high-density regions of the data, forming a discrete surrogate representation that preserves the relevant local energy landscape. This provides a possible functional interpretation of the clustering behavior emphasized in prior analyses of attention dynamics, suggesting that it could enable sparse in-context representations rather than a geometric artifact or pathology.

\subsection{Additional numerics}
\label{app:additional_numerics}

Figure \ref{fig:app_stage1_vs_stage2_bayes_optimal} indicates that careful tuning of $\beta$ and sufficient depth can allow self-attention kernel denoising (Stage 1) alone to approach Bayes optimality. However, consistent with the empirical Bayes approach, it further emphasizes that the combination of depth and cross-attention to the input (Stage 2) provides a more robust mechanism for posterior averaging across regimes, requiring significantly less depth or context length to achieve the same error. 
This decomposition naturally points towards an amortized inference design: once the model refines the particle prior (Stage 1), it could, in principle, be cached for reuse across novel queries (Stage 2) via a relatively inexpensive gradient descent step on Eq. (\ref{eq:energy_DAM}). 

\begin{figure*}[h!]
\centering
\includegraphics[width=0.99\textwidth]{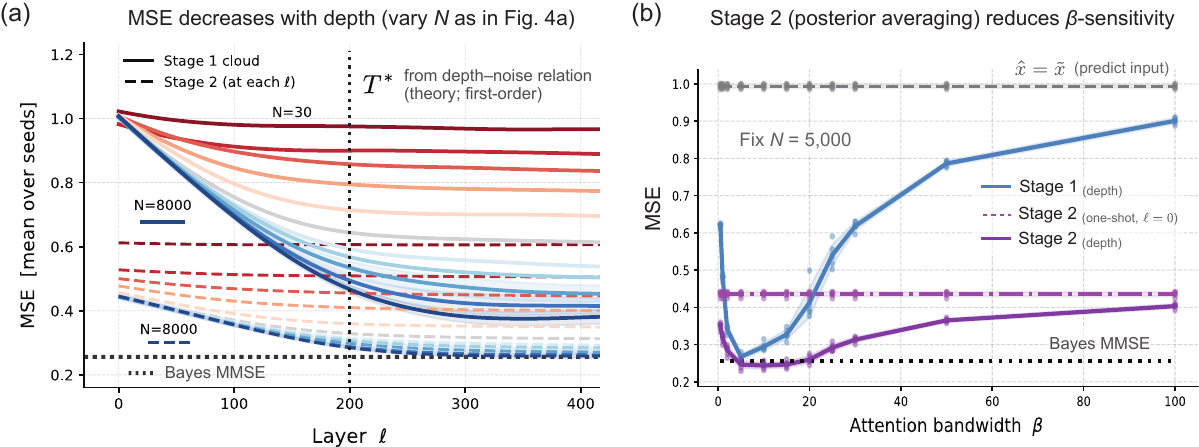}
\caption{
  %   Supplement to Fig. \ref{fig:stage1_vs_stage2_bayes_optimal} in the main text.
  (a) Mean squared error (MSE) as a function of depth $\ell$ for varying context size $N$. Increasing depth improves performance up to a finite horizon $T^*$ predicted by the depth–noise relation, beyond which gains saturate. Larger context sizes yield better particle priors and lower error, approaching the Bayes MMSE. (b) MSE as a function of attention bandwidth $\beta$ for fixed $N$. While Stage 1 (particle refinement) is sensitive to kernel bandwidth, Stage 2 (posterior averaging) substantially reduces this sensitivity, yielding robust performance across a wide range of $\beta$. These results highlight the complementary roles of depth (prior refinement) and cross-attention (posterior inference). 
  The one-shot $\ell=0$ curve applies Stage 2 directly to the noisy context without prior refinement (Appendix C.1). Here, $\beta$ controls the Stage 1 bandwidth, while the Stage 2 readout uses $\beta_c=1/\sigma^2$ throughout.
  Same parameters used as Fig. \ref{fig:stage1_vs_stage2_bayes_optimal}.
  }
\label{fig:app_stage1_vs_stage2_bayes_optimal}
\end{figure*}

Figure \ref{fig:app_stage1_vs_stage2_bayes_optimal}(a) shows the evolution of MSE with depth (i.e., in-context during a single forward pass) for varying context sizes $N$. Consistent with the empirical Bayes interpretation, refinement of the particle approximation to the prior facilitates improved posterior estimates. The improvement saturates near a finite depth $T^*$, in agreement with the predicted depth–noise relationship. Larger context sizes reduce finite-sample error and allow the estimator to approach the Bayes optimal bound more closely.

Figure \ref{fig:app_stage1_vs_stage2_bayes_optimal}(b) examines sensitivity to the attention bandwidth $\beta$. We emphasize that, for moderate noise strength, there should be no expectation that the particle prior alone be used for $x$-pred denoising. Posterior averaging (Stage 2) is the correct step for this task in the framework, and the empirics indicate its general robustness. Indeed, while Stage 1 performance depends relatively strongly on $\beta$ (consistent with Fig. \ref{fig:gaussian_linear_variance_decay}(b) in the main text), posterior averaging substantially mitigates this sensitivity.

% ------
\section{Posterior mean for Gaussian mixture prior}
\label{sec:app_bayes_opt_GMM_formula}
% ------

We recall the posterior mean for Gaussian mixture priors in order to make explicit its connection to the attention-based estimator used in the main text (see e.g. \citet{bishop2006pattern} for background). The point-mass case arises as a limiting regime.

Recall that under squared error loss (MSE), the Bayes-optimal estimator is the posterior mean
\[
\hat{x}(y) 
  = \mathbb{E}[X \mid \tilde X = y] 
  = \int x\, p_{X|\tilde X}(x \mid y)\, dx.
\]

Under Gaussian mixture priors, $\rho_0 = \sum_{i=1}^K \pi_i \mathcal{N}(\mu_i, \Sigma_i)$, this expression has a closed form. 
For observations corrupted by additive Gaussian noise, $\tilde X = X + \epsilon\,\,$ with $\epsilon \sim \mathcal{N}(0, \sigma^2 I_d)$,
the posterior distribution $p(X \mid \tilde X = y)$ is also a mixture, but with observation-dependent weights:
\[
w_i(y)
:= P(\textrm{component}=i \mid y)
=
\frac{
\pi_i \, \mathcal{N}(y \mid \mu_i, \Sigma_i + \sigma^2 I_d)
}{
\sum_{j=1}^K \pi_j \, \mathcal{N}(y \mid \mu_j,\,\Sigma_j + \sigma^2 I_d)}.
\]

Evaluating the posterior mean 
$\int x\, p_{X|\tilde X}(x \mid y)\, dx$ yields:
\begin{equation}
\label{eq:posterior_mean_GMM}
\mathbb{E}[X \mid \tilde X = y]
=
\sum_{i=1}^K w_i(y)\,
\Bigl[
\mu_i + \Sigma_i (\Sigma_i + \sigma^2 I_d)^{-1}(y - \mu_i)
\Bigr],
\end{equation}

which is a weighted combination of component-wise shrinkage estimators. The weights are determined by the likelihood of each mixture component under the noisy observation $y$. 
Intuitively, the estimate moves the observation toward nearby component means, combining them according to their likelihood and shrinking the displacement according to the signal-to-noise ratio.
E.g., for a single isotropic component with zero mean and covariance $\Sigma = a^2 I_d$, this reduces to the familiar shrinkage estimator $\mathbb{E}[X \mid \tilde X = y] = \frac{a^2}{a^2 + \sigma^2}\, y$.

\textbf{Point-mass mixture.}
In the limiting case where the prior is a point mass mixture, $\rho_0 = \sum_{i=1}^K \pi_i \delta_{\mu_i}$, the posterior weights reduce to 
$w_i(y) \propto \pi_i \exp\!\left(-\frac{\|y - \mu_i\|^2}{2\sigma^2}\right)$ and the posterior mean becomes 
\[
\mathbb{E}[X \mid \tilde X = y]
=
\sum_{i=1}^K w_i(y)\, \mu_i
=
\frac{
\sum_{i=1}^K \pi_i \exp\!\left(-\frac{\|y - \mu_i\|^2}{2\sigma^2}\right)\, \mu_i
}{
\sum_{j=1}^K \pi_j \exp\!\left(-\frac{\|y - \mu_j\|^2}{2\sigma^2}\right)
}.
\]
Thus, in the point-mass limit, the posterior mean reduces exactly to a Gaussian kernel-weighted average over support points, coinciding with the attention operation in Algorithm~\ref{alg:two_stage_multilayer_attn} and yielding a probabilistic interpretation of finite-sample attention as posterior averaging under a discrete prior.

\textbf{Bayes MMSE.}
Eq. (\ref{eq:posterior_mean_GMM}) defines the Bayes optimal estimator for the $x$-prediction MSE objective when the prior is a known Gaussian mixture. As this is the best estimate one could make for $\hat x$ when the prior is known, this serves as the key baseline in Fig.~\ref{fig:stage1_vs_stage2_bayes_optimal}. In contrast, our framework operates in the regime where the prior itself is unknown and must be jointly estimated from corrupted observations.

%\input{math appendix nilava}
%\input{soumya_appendix_neuripsv6_short}

%\input{soumya_appendix_neuripsv7_short}

%\input{soumya_appendix_neuripsv8_short}

% --- Local notation conventions for Soumya appendix ---
% Expectation: use \E{...} when there is no subscript.
% For conditional/subscripted expectations use \Ex[...]{...}.
\newcommand{\Ex}[2][]{\mathbb{E}_{#1}\!\left[#2\right]}

% Posterior mean map, consistent with main paper:
\newcommand{\postmean}[1]{m_{#1}}

% Gaussian-attention barycenter / finite-beta denoising map:
\newcommand{\attmean}[1]{F_{\beta,#1}}

\section{Exact Gaussian Drift, Hard Truncation, and Sequential Posterior-Mean Recovery}
\label{sec:exact-gaussian-drift-hard-truncation-recovery}

\subsection{Pseudocode: ODE-based continuous-time sequential posterior-mean estimator}
\label{subsec:pseudocode}

We record the continuous-depth analogues of the two-stage estimator in Algorithm~\ref{alg:two_stage_multilayer_attn}.  The notation is chosen to match the main text: the observed noisy tokens are \(\tilde x_i\), the evolving particles are \(z_i\), and \(y\) denotes a query point.  To denoise the original prompt one takes \(y=\tilde x_i\) for each input token.  The observation-noise variance is denoted by \(\tau\) in the analysis below; it is the same quantity as \(\sigma^2\) in the main text, so \(T_\beta=\beta\tau/2=\beta\sigma^2/2\).  Algorithm~\ref{alg:full-sample-sequential-posterior} is the full empirical flow corresponding to the finite-sample attention dynamics, while Algorithm~\ref{alg:truncated-sequential-posterior} is the compact-support version used in the proof for fixed \(R\).

\begin{algorithm}[h]
\caption{Full-sample continuous-depth posterior-mean estimator}
\label{alg:full-sample-sequential-posterior}
\begin{algorithmic}[1]
\Require Noisy tokens \(\tilde x_1,\ldots,\tilde x_N\sim f_0=\gamma_\tau*P_0\), inverse temperature \(\beta\), terminal time \(T_\beta=\beta\tau/2\), query \(y\)
\State Initialize \(z_i(0)\gets \tilde x_i\) and \(\mu_t^N:=\frac1N\sum_{j=1}^N\delta_{z_j(t)}\)
\State Evolve, for \(0\le t\le T_\beta\) and \(i=1,\ldots,N\),
\[
\dot z_i(t)=X_\beta[\mu_t^N](z_i(t))
=
\frac{\sum_{j=1}^N z_j(t)G_\beta(z_i(t)-z_j(t))}
{\sum_{j=1}^N G_\beta(z_i(t)-z_j(t))}
-z_i(t).
\]
\State Output the posterior mean against the refined particle prior:
\[
\widehat m_N(y):=
m_{\mu_{T_\beta}^N}(y)
=
\frac{\sum_{i=1}^N z_i(T_\beta)\gamma_\tau(y-z_i(T_\beta))}
{\sum_{i=1}^N \gamma_\tau(y-z_i(T_\beta))}.
\]
\end{algorithmic}
\end{algorithm}

\begin{algorithm}[H]
\caption{Empirical hard-truncation continuous-depth posterior-mean estimator}
\label{alg:truncated-sequential-posterior}
\begin{algorithmic}[1]
\Require Noisy tokens \(\tilde x_1,\ldots,\tilde x_N\sim f_0=\gamma_\tau*P_0\), radius \(R\), inverse temperature \(\beta\), terminal time \(T_\beta=\beta\tau/2\), query \(y\)
\State Retain the tokens in the radius-\(R\) ball and relabel them as
\[
\{z_1(0),\ldots,z_{N_R}(0)\}:=\{\tilde x_i:\ |\tilde x_i|\le R\}.
\]
\If{\(N_R=0\)}
\State Increase \(R\) or resample.
\EndIf
\State Set \(\mu_t^{N_R,\beta,[R]}:=\frac1{N_R}\sum_{j=1}^{N_R}\delta_{z_j(t)}\)
\State Evolve, for \(0\le t\le T_\beta\) and \(i=1,\ldots,N_R\),
\[
\dot z_i(t)=X_\beta[\mu_t^{N_R,\beta,[R]}](z_i(t))
=
\frac{\sum_{j=1}^{N_R} z_j(t)G_\beta(z_i(t)-z_j(t))}
{\sum_{j=1}^{N_R} G_\beta(z_i(t)-z_j(t))}
-z_i(t).
\]
\State Output the posterior mean against the truncated refined particle prior:
\[
\widehat m_{N_R,R}(y):=
m_{\mu_{T_\beta}^{N_R,\beta,[R]}}(y)
=
\frac{\sum_{i=1}^{N_R} z_i(T_\beta)\gamma_\tau(y-z_i(T_\beta))}
{\sum_{i=1}^{N_R} \gamma_\tau(y-z_i(T_\beta))}.
\]
\end{algorithmic}
\end{algorithm}
%\newpage

\subsection{Mathematical preliminaries}
This section proves a sequential particle approximation theorem for Gaussian-prior
posterior means under the exact Gaussian logarithmic drift.  The argument has three
steps.  First, at fixed \(\beta\), we establish the compact-support mean-field theory,
including the barycentric form of the drift, invariant-ball estimates, deterministic
well-posedness, stability, and pointwise-in-time propagation of chaos.  Second, we
formulate a noncompact admissibility criterion based on hard truncation and show that
it transfers compact-support particle approximation to noncompact data.  Third, we
introduce an admissible-recoverable class of priors and verify that Gaussian priors
belong to it.

We will use the following background results: the Kantorovich--Rubinstein duality formula for \(W_1\)
\cite[Particular Case~5.16]{Villani2009OptimalTransport}, the characterization of
Wasserstein convergence by weak convergence plus moment convergence
\cite[Definition~6.8 and Theorem~6.9]{Villani2009OptimalTransport}, the continuity-equation
formulation of absolutely continuous Wasserstein curves
\cite[Sections~8.1--8.3]{AmbrosioGigliSavare2008GradientFlows}, the Dobrushin
fixed-point and stability method for Vlasov-type mean-field equations
\cite[main Theorem, Propositions~2--4, and Section~6]{Dobrushin1979Vlasov}, the standard coupling method
for propagation of chaos and the nonlinear-process viewpoint
\cite[Chapter~I, Sections~1--2]{Sznitman1991PropagationChaos}
\cite[Section~4.1]{ChaintronDiez2022PropagationChaosI}, the empirical \(W_1\) estimate obtained from
\cite[Theorem~1, with \(p=1\)]{FournierGuillin2015EmpiricalWasserstein}, the Bihari--Osgood form of the generalized Bellman lemma \cite[Sections~3--4]{Bihari1956GronwallGeneralization}.

We shall also use the following elementary form of Gronwall's inequality.  If
\(u:[0,T]\to[0,\infty)\) is continuous and
\[u(t)\le a(t)+L\int_0^t u(s)\,ds,\qquad 0\le t\le T,
\]
where \(L\ge0\) and \(a:[0,T]\to[0,\infty)\) is nondecreasing, then
\[u(t)\le a(t)e^{Lt},\qquad 0\le t\le T.
\]
In particular, if \(u(t)\le L\int_0^t u(s)\,ds\), then \(u\equiv0\).  The only
nonlinear variant used below is the Osgood--Bihari argument in
Theorem~\ref{thm:weak-strong-uniqueness-gaussian}.

All noncompact particle statements below are sequential: first \(N\to\infty\) at fixed
\((\beta,R)\), then \(R\to\infty\) at fixed \(\beta\), and finally \(\beta\to\infty\).

All existence assertions are made on a prescribed finite time interval.  For compactly
supported data this causes no loss, because the invariant-ball and Lipschitz estimates
below give global-in-time characteristic solutions on every finite horizon.  For general
noncompact data we do not assert a general global theory; instead, admissibility is a
finite-horizon hypothesis.  In the Gaussian case this hypothesis is verified directly by an
explicit positive-definite covariance flow.

\subsection{Fixed-\texorpdfstring{$\beta$}{beta} mean-field theory for the exact Gaussian drift}
\label{sec:selfcontained-compact-support-posterior-mean}

Throughout we write \(\mathcal P(\R^d)\) for the set of Borel probability
measures on \(\R^d\), and
\[
\mathcal P_1(\R^d):=\left\{\mu\in\mathcal P(\R^d):\;
\int_{\R^d}|x|\,\mu(dx)<\infty\right\}.
\]
For \(\mu,\nu\in\mathcal P_1(\R^d)\), we denote by \(W_1(\mu,\nu)\) the
\(1\)-Wasserstein distance.  If \(E\subset\R^d\) and \(F:E\to\R^k\), we write
\[
\Lip_E(F)
:=
\sup_{\substack{x,x'\in E\\x\ne x'}}
\frac{|F(x)-F(x')|}{|x-x'|}
\]
for the Lipschitz constant of \(F\) on \(E\).  When \(E=\R^d\), we write simply
\[
\Lip(F):=\Lip_{\R^d}(F).
\]
If \(F=F(x,\theta)\) depends on a spatial variable \(x\in\R^d\) and on an auxiliary
parameter \(\theta\), then
\[
\Lip_x F(\cdot,\theta)
\]
denotes the Lipschitz constant of the map \(x\mapsto F(x,\theta)\), with \(\theta\)
held fixed.

We shall use the Kantorovich--Rubinstein formula \cite[Particular Case~5.16]{Villani2009OptimalTransport}
\begin{equation}
\label{eq:sc-KR}
W_1(\mu,\nu)
=
\sup_{\substack{\varphi:\R^d\to\R\\ \Lip(\varphi)\le 1}}
\int_{\R^d}\varphi(x)\,(\mu-\nu)(dx).
\end{equation}
For \(R>0\), we write
\[
B_R:=\{x\in\R^d:\ |x|\le R\}
\]
for the closed Euclidean ball of radius \(R\) centered at the origin, and
\[
\mathcal P(B_R):=\{\mu\in\mathcal P(\R^d):\ \supp\mu\subset B_R\}.
\]

On \(\mathcal P(B_R)\), weak convergence and \(W_1\)-convergence agree by
\cite[Theorem~6.9]{Villani2009OptimalTransport}; since \(B_R\) is compact,
\(\mathcal P(B_R)\) is weakly compact, hence \((\mathcal P(B_R),W_1)\) is compact
and therefore complete.

\subsubsection{Notation and basic objects}
\label{subsec:exact-drift-notation}

The compact-support results in this section are stated for initial laws already supported in a fixed ball. No truncation is used in their formulation. Truncations enter only in Section~\ref{sec:noncompact-admissibility}, where noncompact initial laws are approximated by compactly supported laws.

Throughout this section,  $|\cdot|$ refers to Euclidean vector norm. For \(\beta>0\), set
\[
G_\beta(z):=\Bigl(\frac{\beta}{2\pi}\Bigr)^{d/2}e^{-\frac{\beta}{2}|z|^2},
\qquad z\in\R^d.
\]
For \(\mu\in\mathcal P_1(\R^d)\), define
\[
D_\beta[\mu](x):=(G_\beta*\mu)(x)=\int_{\R^d}G_\beta(x-y)\,\mu(dy).
\]
Since \(G_\beta>0\), the logarithmic drift
\begin{equation}
\label{eq:sc-exact-drift-def}
X_\beta[\mu](x):=\frac1\beta\frac{\nabla(G_\beta*\mu)(x)}{(G_\beta*\mu)(x)}
\end{equation}
is well defined for every \(x\in\R^d\). For a fixed observation-noise variance \(\tau>0\), define the Gaussian posterior-mean map
\begin{equation}
\label{eq:sc-posterior-mean-def}
m_\mu(y)
:=
\frac{\int_{\R^d} x\,\gamma_\tau(y-x)\,\mu(dx)}{\int_{\R^d} \gamma_\tau(y-x)\,\mu(dx)},
\qquad
\gamma_\tau(z):=(2\pi\tau)^{-d/2}e^{-|z|^2/(2\tau)}.
\end{equation}

For a matrix \(A\in\mathbb R^{d\times d}\), we write
\[
\|A\|_{\mathrm{HS}}
:=
\bigl(\operatorname{Tr}(A^\top A)\bigr)^{1/2}
=
\left(\sum_{i,j=1}^d A_{ij}^2\right)^{1/2}
\]
for its Hilbert--Schmidt, equivalently Frobenius, norm.

\subsubsection{Posterior means under \(W_1\)-convergence}
\label{subsec:posterior-mean-noncompact-W1}

\begin{proposition}[Local \(W_1\)-continuity of the Gaussian posterior mean]
\label{prop:nc-posterior-mean-W1}
Fix \(\tau>0\).  Let \(\nu_n,\nu\in\mathcal P_1(\R^d)\), and assume
\[
W_1(\nu_n,\nu)\to0.
\]
Then, for every \(M<\infty\),
\[
\sup_{|y|\le M}|m_{\nu_n}(y)-m_\nu(y)|\to0.
\]
\end{proposition}

\begin{proof}
Write
\[
Z_\lambda(y):=\int\gamma_\tau(y-x)\,\lambda(dx),
\qquad
N_\lambda(y):=\int x\gamma_\tau(y-x)\,\lambda(dx).
\]
For \(|y|\le M\), the functions \(x\mapsto\gamma_\tau(y-x)\) are globally Lipschitz
with a uniform Lipschitz constant.  Therefore
\[
\sup_{|y|\le M}|Z_{\nu_n}(y)-Z_\nu(y)|
\le C_{\tau,M}W_1(\nu_n,\nu)\to0.
\]
Similarly, each coordinate of \(x\mapsto x\gamma_\tau(y-x)\) is bounded and globally
Lipschitz uniformly for \(|y|\le M\), because polynomial factors are dominated by the
Gaussian. Hence
\[
\sup_{|y|\le M}|N_{\nu_n}(y)-N_\nu(y)|\to0.
\]
Since \(Z_\nu\) is continuous and strictly positive on the compact set \(B_M\), it has a
positive lower bound there.  Uniform convergence of numerator and denominator then gives
uniform convergence of the quotient \(N_{\nu_n}/Z_{\nu_n}\to N_\nu/Z_\nu\) on \(B_M\).
\end{proof}

\subsubsection{Exact barycentric form of the drift}

\begin{proposition}[Exact drift as a Gaussian barycenter minus the identity]
\label{prop:sc-barycentric-drift}
For $\mu\in\mathcal P_1(\R^d)$, define
\[
A_\beta[\mu](x):=\int_{\R^d}y\,G_\beta(x-y)\,\mu(dy),
\qquad
\attmean{\mu}(x):=\frac{A_\beta[\mu](x)}{D_\beta[\mu](x)}.
\]
Then, for every $x\in\R^d$,
\begin{equation}
\label{eq:sc-drift-barycentric}
X_\beta[\mu](x)=\attmean{\mu}(x)-x.
\end{equation}
In particular, if $\mu\in\mathcal P(B_R)$, then
\[
\attmean{\mu}(x)\in \overline{\co}(\supp\mu)\subset B_R
\qquad\text{for every }x\in\R^d.
\]
Moreover, if $\beta=\tau^{-1}$, then
\begin{equation}
\label{eq:sc-posterior-mean-and-drift}
X_{\tau^{-1}}[\mu](y)=m_\mu(y)-y
\qquad (y\in\R^d).
\end{equation}
\end{proposition}
\begin{proof}
Since \(\nabla G_\beta(z)=-\beta zG_\beta(z)\),
\[
\frac1\beta\nabla(G_\beta*\mu)(x)
=
-\int_{\mathbb R^d}(x-y)G_\beta(x-y)\,\mu(dy).
\]
Dividing by \(D_\beta[\mu](x)>0\) gives
\[
X_\beta[\mu](x)
=
\frac{\int yG_\beta(x-y)\,\mu(dy)}{\int G_\beta(x-y)\,\mu(dy)}-x
=
\attmean{\mu}(x)-x.
\]
The normalized measure
\[
\frac{G_\beta(x-y)}{D_\beta[\mu](x)}\,\mu(dy)
\]
is a probability measure supported on \(\supp\mu\), so its barycenter lies in
\(\overline{\co}(\supp\mu)\).  If \(\beta=\tau^{-1}\), then \(G_\beta=\gamma_\tau\), and
the identity becomes \(X_{\tau^{-1}}[\mu]=m_\mu-\mathrm{Id}\).
\end{proof}

\subsubsection{Compact-support geometry and invariant balls}

\begin{proposition}[Invariant ball for the exact drift]
\label{prop:sc-invariant-ball}
Fix $\beta>0$ and $R>0$.
\begin{enumerate}
\renewcommand{\labelenumi}{(\roman{enumi})}
\item If $\mu\in\mathcal P(B_R)$, then for every $x\in\R^d$,
\begin{equation}
\label{eq:sc-radial-inward-estimate}
\langle x,X_\beta[\mu](x)\rangle\le R|x|-|x|^2.
\end{equation}
In particular, $\langle x,X_\beta[\mu](x)\rangle\le 0$ whenever $|x|\ge R$.
\item Let $(\mu_t)_{t\in[0,T]}\subset\mathcal P(B_R)$ be measurable, and let $x_t$ solve
\[
\dot x_t=X_\beta[\mu_t](x_t),\qquad x_0\in B_R.
\]
Then $x_t\in B_R$ for every $t\in[0,T]$.
\item If the initial data of the exact $N$-particle system
\[
\dot X_t^{i,N}=X_\beta[\mu_t^N](X_t^{i,N}),
\qquad
\mu_t^N:=\frac1N\sum_{j=1}^N\delta_{X_t^{j,N}},
\qquad i=1,\dots,N,
\]
all belong to $B_R$, then every particle stays in $B_R$ for all later times.
\end{enumerate}
\end{proposition}
\begin{proof}
By Proposition~\ref{prop:sc-barycentric-drift},
\[
X_\beta[\mu](x)=\attmean{\mu}(x)-x,
\qquad
\attmean{\mu}(x)\in B_R
\]
whenever \(\mu\in\mathcal P(B_R)\).  Hence
\[
\langle x,X_\beta[\mu](x)\rangle
=
\langle x,\attmean{\mu}(x)\rangle-|x|^2
\le R|x|-|x|^2,
\]
which proves (i).  Parts (ii) and (iii) follow by the standard first-exit argument applied
to \(t\mapsto |x_t|^2/2\): on the exterior region \(\{|x|>R\}\), the radial derivative is
strictly negative.  Therefore a trajectory starting in \(B_R\) cannot exit \(B_R\).  The
particle assertion is the same argument applied to each particle, since the empirical
measure is supported in \(B_R\) as long as the particles are.
\end{proof}

\subsubsection{Lipschitz control on a bounded support class}

\begin{proposition}[Exact drift is Lipschitz on an invariant ball]
\label{prop:sc-drift-lipschitz-bounded-support}
Fix $\beta>0$ and $R>0$. Then there exists $L_{\beta,R}>0$ such that for all
$x,x'\in B_R$ and all $\mu,\nu\in\mathcal P(B_R)$,
\[
|X_\beta[\mu](x)-X_\beta[\nu](x')|
\le
L_{\beta,R}\bigl(|x-x'|+W_1(\mu,\nu)\bigr).
\]
\end{proposition}

\begin{proof}
Write \(D_\mu=D_\beta[\mu]\), \(A_\mu=A_\beta[\mu]\), and
\(F_{\beta,\mu}=A_\mu/D_\mu\).  By Proposition~\ref{prop:sc-barycentric-drift},
\(X_\beta[\mu]=F_{\beta,\mu}-\mathrm{Id}\), so it suffices to estimate \(F_{\beta,\mu}\).

For \(x,y\in B_R\),
\[
D_\mu(x)\ge c_{\beta,R}:=\inf_{u,v\in B_R}G_\beta(u-v)>0.
\]
Moreover,
\[
|A_\mu(x)|\le R\|G_\beta\|_\infty.
\]
The maps \(y\mapsto G_\beta(x-y)\) and
\(y\mapsto yG_\beta(x-y)\), restricted to \(B_R\), have Lipschitz constants bounded
uniformly for \(x\in B_R\).  Hence, by the Kantorovich--Rubinstein formula
\cite[Particular Case~5.16]{Villani2009OptimalTransport},
\[
|D_\mu(x)-D_\nu(x')|+|A_\mu(x)-A_\nu(x')|
\le C_{\beta,R}\bigl(|x-x'|+W_1(\mu,\nu)\bigr).
\]
Using the quotient identity
\[
\frac{A_\mu(x)}{D_\mu(x)}-\frac{A_\nu(x')}{D_\nu(x')}
=
\frac{A_\mu(x)-A_\nu(x')}{D_\mu(x)}
+
A_\nu(x')\frac{D_\nu(x')-D_\mu(x)}{D_\mu(x)D_\nu(x')}
\]
and the lower bound \(D_\mu,D_\nu\ge c_{\beta,R}\), we obtain
\[
|F_{\beta,\mu}(x)-F_{\beta,\nu}(x')|
\le C_{\beta,R}\bigl(|x-x'|+W_1(\mu,\nu)\bigr).
\]
Since \(X_\beta[\mu](x)=F_{\beta,\mu}(x)-x\), the claim follows after increasing the constant.
\end{proof}

\begin{proposition}[Characteristic well-posedness and stability on the invariant ball]
\label{prop:sc-deterministic-wellposedness-stability}
Fix $\beta>0$, $R>0$, and $T>0$, and define
\[
b(x,\mu):=X_\beta[\mu](x),
\qquad x\in\R^d,\ \mu\in\mathcal P(B_R).
\]
Let $L_{\beta,R}>0$ be the constant from
Proposition~\ref{prop:sc-drift-lipschitz-bounded-support}.

Then the following hold.

\begin{enumerate}
\renewcommand{\labelenumi}{(\roman{enumi})}
\item For every continuous path
\[
q\in C([0,T];\mathcal P(B_R))
\]
and every $x\in B_R$, there exists a unique solution
\[
z^{q,x}\in C^1([0,T];\R^d)
\]
of
\[
\dot z_t=b(z_t,q_t),\qquad z_0=x.
\]
Moreover,
\[
z_t^{q,x}\in B_R\qquad\text{for every }t\in[0,T].
\]
We define
\[
\Phi_t^q(x):=z_t^{q,x},
\qquad x\in B_R,\ t\in[0,T].
\]

\item For every $f_0\in\mathcal P(B_R)$, the map
\[
\Gamma:C([0,T];\mathcal P(B_R))\to C([0,T];\mathcal P(B_R)),
\qquad
(\Gamma q)_t:=(\Phi_t^q)_\# f_0,
\]
is well defined and has a unique fixed point
\[
f\in C([0,T];\mathcal P(B_R)).
\]
Equivalently, there exist a unique curve
\[
f\in C([0,T];\mathcal P(B_R))
\]
and a unique family of maps
\[
\Phi_t:B_R\to B_R,\qquad t\in[0,T],
\]
such that, for every $x\in B_R$, the trajectory
\[
t\longmapsto \Phi_t(x)
\]
is the unique $C^1$ solution of
\[
\dot z_t=b(z_t,f_t),\qquad z_0=x,
\]
and
\[
f_t=(\Phi_t)_\# f_0
\qquad\text{for every }t\in[0,T].
\]

\item If $f,g\in C([0,T];\mathcal P(B_R))$ are two solutions corresponding to initial data
$f_0,g_0\in\mathcal P(B_R)$, then for every $t\in[0,T]$,
\begin{equation}
\label{eq:sc-deterministic-stability-W1}
W_1(f_t,g_t)\le e^{2L_{\beta,R}t}W_1(f_0,g_0).
\end{equation}
\end{enumerate}
\end{proposition}

\begin{proof}
For a prescribed path \(q\in C([0,T];\mathcal P(B_R))\), the vector field
\(x\mapsto b(x,q_t)\) is continuous in \(t\), Lipschitz in \(x\) on \(B_R\) uniformly in
\(t\), and points inward at \(\partial B_R\) by
Proposition~\ref{prop:sc-invariant-ball}. Hence the characteristic equation is globally
well posed on \([0,T]\), and the flow \(\Phi_t^q:B_R\to B_R\) is well defined.

The nonlinear law is obtained as the fixed point of
\[
\Gamma q:=\bigl((\Phi_t^q)_\# f_0\bigr)_{0\le t\le T}.
\]
By Proposition~\ref{prop:sc-drift-lipschitz-bounded-support}, for two paths \(q,r\),
\[
|\Phi_t^q(x)-\Phi_t^r(x)|
\le
L_{\beta,R}\int_0^t e^{L_{\beta,R}(t-s)}W_1(q_s,r_s)\,ds.
\]
Therefore, for the exponentially weighted metric
\[
d_\alpha(q,r):=\sup_{0\le t\le T}e^{-\alpha t}W_1(q_t,r_t),
\]
one has
\[
d_\alpha(\Gamma q,\Gamma r)
\le
\frac{L_{\beta,R}}{\alpha-L_{\beta,R}}\,d_\alpha(q,r).
\]
Choosing \(\alpha>2L_{\beta,R}\), \(\Gamma\) is a contraction on the complete space
\(C([0,T];\mathcal P(B_R))\). This gives existence and uniqueness of \(f\) and the
characteristic flow. This is the usual Dobrushin fixed-point argument in the
Kantorovich--Rubinstein metric; compare
\cite[Propositions~2--3 and Section~6]{Dobrushin1979Vlasov}.

For stability, let \(f,g\) be two solutions and couple their initial data by an arbitrary
coupling \(\pi_0\). If \(X_t,Y_t\) are the corresponding nonlinear characteristics, then
\[
|X_t-Y_t|
\le |X_0-Y_0|
+
L_{\beta,R}\int_0^t\bigl(|X_s-Y_s|+W_1(f_s,g_s)\bigr)\,ds.
\]
Taking expectations, using \(W_1(f_s,g_s)\le \mathbb E|X_s-Y_s|\), and applying
Gronwall gives
\[
W_1(f_t,g_t)\le e^{2L_{\beta,R}t}W_1(f_0,g_0).
\]
This is the standard Dobrushin stability estimate; compare
\cite[Proposition~4]{Dobrushin1979Vlasov}.
\end{proof}

\subsubsection{Weak formulations for the later noncompact passage}
\label{subsec:finite-action-weak-formulation}

The compact-support solutions constructed above are characteristic solutions.  In the
noncompact truncation limit, however, compactness only gives convergence of measure-valued
curves.  We therefore record the weak continuity-equation formulation used later. This is the usual weak formulation of the continuity equation in \(\mathbb R^d\);
compare \cite[Sections~8.1--8.3]{AmbrosioGigliSavare2008GradientFlows}.  Compactly
supported characteristic solutions are special cases of this definition, by the chain rule
along characteristics.

\begin{definition}[Finite-action weak solutions]
\label{def:nc-finite-action-weak-solution}
Fix \(\beta>0\) and \(T>0\).  A curve
\[
f\in C([0,T];\mathcal P_1(\R^d))
\]
is called a finite-action weak solution of
\begin{equation}
\label{eq:nc-noncompact-continuity-equation}
\partial_t f_t+\nabla\cdot(f_tX_\beta[f_t])=0
\end{equation}
on \([0,T]\), with initial condition \(f_0\), if \(f_{t=0}=f_0\),
\[
\int_0^T\int_{\R^d}|X_\beta[f_t](x)|\,f_t(dx)\,dt<\infty,
\]
and, for every \(\phi\in C_c^1(\R^d)\),
\begin{equation}
\label{eq:nc-weak-formulation}
\int\phi\,df_t-
\int\phi\,df_0
=
\int_0^t\int \nabla\phi(x)\cdot X_\beta[f_s](x)\,f_s(dx)\,ds
\end{equation}
for every \(t\in[0,T]\).
\end{definition}

This definition is not an additional existence theorem.  It is only the formulation in
which the noncompact limits below will be identified.  We now return to the compactly
supported particle system and prove the pointwise propagation-of-chaos estimate needed
for the truncation argument.

\subsubsection{Fixed-$\beta$ compact-support pointwise propagation of chaos}

The next theorem uses propagation of chaos in the pointwise-in-time sense: for each fixed
\(t\) and each fixed \(k\), the \(k\)-particle marginal converges to \(f_t^{\otimes k}\).

\begin{theorem}[\texorpdfstring{Compactly supported fixed-\(\beta\) propagation of chaos}{Compactly supported fixed-beta propagation of chaos}]
\label{thm:sc-exact-fixed-beta-poc-compact-support}
Fix $\beta>0$, $T>0$, and $R>0$, and assume
\[
f_0\in\mathcal P(B_R).
\]
Let $(X_0^i)_{i\ge1}$ be an i.i.d. sequence with common law $f_0$. For each $N\ge1$,
use $(X_0^1,\dots,X_0^N)$ as particle initial data and let
\[
\dot X_t^{i,N}=X_\beta[\mu_t^N](X_t^{i,N}),
\qquad
\mu_t^N:=\frac1N\sum_{j=1}^N\delta_{X_t^{j,N}},
\qquad i=1,\dots,N,
\]
be the exact particle system. Let
\[
f\in C([0,T];\mathcal P(B_R))
\]
and
\[
\Phi_t:B_R\to B_R
\]
be the unique nonlinear law path and nonlinear characteristic flow from
Proposition~\ref{prop:sc-deterministic-wellposedness-stability}, and define
\[
\bar X_t^i:=\Phi_t(X_0^i),
\qquad i\ge1,
\qquad
\bar\mu_t^N:=\frac1N\sum_{j=1}^N\delta_{\bar X_t^j}.
\]
Then:

\begin{enumerate}
\renewcommand{\labelenumi}{(\roman{enumi})}
\item the particle system is well posed on $[0,T]$, and the nonlinear law path $f_t$
is well posed on $[0,T]$, and
\[
\supp \mu_t^N\subset B_R,\qquad \supp f_t\subset B_R
\qquad\text{for every }t\in[0,T];
\]

\item the random variables $(\bar X_t^i)_{i\ge1}$ are i.i.d. with common law $f_t$, and
for every $t\in[0,T]$,
\[
\EE|X_t^{1,N}-\bar X_t^1|
\le
L_{\beta,R}e^{2L_{\beta,R}t}
\int_0^t \EE\bigl[W_1(\bar\mu_s^N,f_s)\bigr]\,ds;
\]

\item for every fixed $t\in[0,T]$,
\[
\EE\bigl[W_1(\mu_t^N,f_t)\bigr]\longrightarrow 0
\qquad\text{as }N\to\infty;
\]

\item if $F_t^{N,k}$ denotes the law of $(X_t^{1,N},\dots,X_t^{k,N})$, then for each fixed
$k\ge1$,
\[
W_1\bigl(F_t^{N,k},f_t^{\otimes k}\bigr)\longrightarrow 0
\qquad\text{as }N\to\infty,
\]
where $W_1$ on $(\R^d)^k$ is taken with respect to the cost
\[
(x_1,\dots,x_k),(y_1,\dots,y_k)\longmapsto \sum_{i=1}^k |x_i-y_i|.
\]
\end{enumerate}
\end{theorem}
\begin{proof}
The particle vector field is smooth on \((\mathbb R^d)^N\), since
\[
F_{N,i}(x_1,\ldots,x_N)
=
\frac{\sum_{j=1}^N x_jG_\beta(x_i-x_j)}
{\sum_{j=1}^N G_\beta(x_i-x_j)}-x_i
\]
has a strictly positive denominator.  Proposition~\ref{prop:sc-invariant-ball} gives
global existence on \([0,T]\) and invariance of \(B_R^N\).  The nonlinear law path and
flow are given by Proposition~\ref{prop:sc-deterministic-wellposedness-stability}.

Let
\[
\bar X_t^i:=\Phi_t(X_0^i),
\qquad
\bar\mu_t^N:=\frac1N\sum_{j=1}^N\delta_{\bar X_t^j}.
\]
Then \((\bar X_t^i)_{i\ge1}\) are i.i.d. with common law \(f_t\).  Set
\[
\nu_N(t):=\mathbb E|X_t^{1,N}-\bar X_t^1|.
\]
We spell out the coupling step.  For each time \(s\), the measure
\[
\Pi_s^N:=\frac1N\sum_{j=1}^N
\delta_{(X_s^{j,N},\bar X_s^j)}
\]
is a coupling of \(\mu_s^N\) and \(\bar\mu_s^N\).  Hence
\[
W_1(\mu_s^N,\bar\mu_s^N)
\le
\frac1N\sum_{j=1}^N|X_s^{j,N}-\bar X_s^j|.
\]
Taking expectations and using exchangeability of the particle system and of the coupled
nonlinear characteristics gives
\begin{equation}
\label{eq:poc-particlewise-coupling-bound}
\mathbb E W_1(\mu_s^N,\bar\mu_s^N)
\le \frac1N\sum_{j=1}^N\mathbb E|X_s^{j,N}-\bar X_s^j|
=\nu_N(s).
\end{equation}
On the other hand, the Lipschitz estimate of
Proposition~\ref{prop:sc-drift-lipschitz-bounded-support} gives, for the first particle,
\[
|X_t^{1,N}-\bar X_t^1|
\le
L_{\beta,R}\int_0^t
\left(
|X_s^{1,N}-\bar X_s^1|+W_1(\mu_s^N,f_s)
\right)ds.
\]
Using
\[
W_1(\mu_s^N,f_s)
\le
W_1(\mu_s^N,\bar\mu_s^N)+W_1(\bar\mu_s^N,f_s)
\]
and then \eqref{eq:poc-particlewise-coupling-bound}, we obtain
\[
\nu_N(t)
\le
2L_{\beta,R}\int_0^t\nu_N(s)\,ds
+
L_{\beta,R}\int_0^t\mathbb E W_1(\bar\mu_s^N,f_s)\,ds.
\]
Applying the Gronwall estimate stated at the beginning of the appendix yields
\[
\nu_N(t)
\le
L_{\beta,R}e^{2L_{\beta,R}t}
\int_0^t\mathbb E W_1(\bar\mu_s^N,f_s)\,ds.
\]
This is the usual coupling estimate behind propagation of chaos; compare the
classical nonlinear-process construction in
\cite[Chapter~I, Section~1]{Sznitman1991PropagationChaos} and the modern
discussion of coupling methods in \cite[Section~4.1]{ChaintronDiez2022PropagationChaosI}.

For fixed \(s\), the random variables \((\bar X_s^j)_{j\ge1}\) are i.i.d. with law
\(f_s\in\mathcal P(B_R)\).  The empirical-measure estimate
\cite[Theorem~1, with \(p=1\)]{FournierGuillin2015EmpiricalWasserstein} gives
\[
\mathbb E W_1(\bar\mu_s^N,f_s)\to0
\qquad\text{for every fixed }s,
\]
because the required moment condition \(M_q(f_s)<\infty\) for some \(q>1\) is
automatic from the support bound \(f_s\in\mathcal P(B_R)\).
Since \(0\le W_1(\bar\mu_s^N,f_s)\le2R\), dominated convergence gives
\[
\int_0^t\mathbb E W_1(\bar\mu_s^N,f_s)\,ds\to0.
\]
Therefore \(\nu_N(t)\to0\).  Hence
\[
\mathbb E W_1(\mu_t^N,f_t)
\le
\mathbb E W_1(\mu_t^N,\bar\mu_t^N)+\mathbb E W_1(\bar\mu_t^N,f_t)
\le
\nu_N(t)+\mathbb E W_1(\bar\mu_t^N,f_t)
\to0.
\]
Finally, coupling
\((X_t^{1,N},\ldots,X_t^{k,N})\) with
\((\bar X_t^1,\ldots,\bar X_t^k)\) gives
\[
W_1(F_t^{N,k},f_t^{\otimes k})
\le
\sum_{i=1}^k\mathbb E|X_t^{i,N}-\bar X_t^i|
=
k\nu_N(t)\to0.
\]
\end{proof}

\subsection{Noncompact admissibility and hard-truncated particle approximation}
\label{sec:noncompact-admissibility}

We next formulate the deterministic hypothesis that allows compactly supported particle systems to approximate noncompact initial data. The condition is a stability statement for the noncompact mean-field equation under hard truncation of the initial law. With this in hand, the compact-support propagation-of-chaos theorem applies at fixed radius, and the radius can then be removed at the deterministic level.

\subsubsection{Hard truncation as a noncompact approximation}
\label{subsec:hard-truncation-noncompact-only}

For \(R>0\) and \(\mu\in\mathcal P_1(\R^d)\), set
\[
p_R(\mu):=\mu(B_R),
\qquad
q_R(\mu):=1-p_R(\mu).
\]
Whenever \(p_R(\mu)>0\), define the hard truncation of \(\mu\) by
\begin{equation}
\label{eq:nc-hard-truncation-definition}
\mu^{[R]}:=\frac{\mathbf 1_{B_R}\mu}{\mu(B_R)}.
\end{equation}
Thus \(\mu^{[R]}\in\mathcal P(B_R)\).

\begin{lemma}[Hard truncation in \(W_1\)]
\label{lem:nc-hard-truncation-W1}
Let \(\mu\in\mathcal P_1(\R^d)\), and assume that \(p_R(\mu)>0\).  Let \(\mu^{[R]}\) be defined by
\eqref{eq:nc-hard-truncation-definition}.  Then
\begin{equation}
\label{eq:nc-hard-truncation-W1-bound}
W_1(\mu^{[R]},\mu)
\le
\int_{\{|x|>R\}} |x|\,\mu(dx)
+
\frac{q_R(\mu)}{p_R(\mu)}\int_{B_R}|x|\,\mu(dx).
\end{equation}
In particular,
\[
W_1(\mu^{[R]},\mu)\longrightarrow0
\qquad\text{as }R\to\infty.
\]
\end{lemma}

\begin{proof}
By Kantorovich--Rubinstein duality, it suffices to test against \(1\)-Lipschitz
\(\varphi\) with \(\varphi(0)=0\), so that \(|\varphi(x)|\le |x|\).  Since
\[
\mu^{[R]}-\mu
=
\frac{q_R(\mu)}{p_R(\mu)}\mathbf 1_{B_R}\mu
-
\mathbf 1_{B_R^c}\mu,
\]
we get
\[
\left|\int\varphi\,d(\mu^{[R]}-\mu)\right|
\le
\frac{q_R(\mu)}{p_R(\mu)}\int_{B_R}|x|\,\mu(dx)
+
\int_{B_R^c}|x|\,\mu(dx).
\]
Taking the supremum over \(\operatorname{Lip}(\varphi)\le1\) proves the bound.  The
right-hand side tends to zero because \(p_R(\mu)\to1\), \(q_R(\mu)\to0\), and
\(\mu\in\mathcal P_1(\mathbb R^d)\).
\end{proof}

\begin{remark}[Untruncated samples versus empirical hard truncation]
\label{rem:empirical-hard-truncation}
Let \(\mu\in\mathcal P_1(\R^d)\), let \(Y_1,\ldots,Y_N\) be i.i.d. with law \(\mu\), and set
\[
\nu_N:=\frac1N\sum_{i=1}^N\delta_{Y_i},
\qquad
N_R:=\sum_{i=1}^N\mathbf 1_{\{|Y_i|\le R\}},
\qquad
q_R(\mu):=\mu(B_R^c).
\]
By Hoeffding's inequality, for every \(\eta>0\),
\[
\mathbb P\left(
\left|\frac{N_R}{N}-(1-q_R(\mu))\right|>\eta
\right)
\le
2e^{-2N\eta^2}.
\]
Thus the random retained fraction \(N_R/N\) is close to \(1-q_R(\mu)\) with high
probability.
On the event \(\{N_R>0\}\), define the renormalized empirical truncation
\[
\nu_N^{\mathrm{emp},[R]}
:=
\frac1{N_R}\sum_{i:\,|Y_i|\le R}\delta_{Y_i}.
\]
Then
\[
\mathbb P(Y_1,\ldots,Y_N\in B_R)
=
(1-q_R(\mu))^N
\ge
1-Nq_R(\mu).
\]
Thus this event has high probability whenever \(Nq_R(\mu)\ll1\).

Moreover, on \(\{N_R>0\}\),
\[
W_1(\nu_N,\nu_N^{\mathrm{emp},[R]})
\le
\frac1N\sum_{i=1}^N |Y_i|\mathbf 1_{\{|Y_i|>R\}}
+
R\frac{N-N_R}{N}.
\]
Consequently, if
\[
A_R(\mu):=\int_{\{|x|>R\}}|x|\,\mu(dx),
\]
then for every \(\varepsilon>0\),
\[
\mathbb P\left(
N_R>0,\,
W_1(\nu_N,\nu_N^{\mathrm{emp},[R]})>\varepsilon
\right)
\le
\frac{A_R(\mu)+Rq_R(\mu)}{\varepsilon}.
\]
Since \(Rq_R(\mu)\le A_R(\mu)\), the right-hand side tends to \(0\) as
\(R\to\infty\).  Therefore, for any fixed \(\varepsilon>0\),
\[
\lim_{R\to\infty}\limsup_{N\to\infty}
\mathbb P\left(
N_R>0,\,
W_1(\nu_N,\nu_N^{\mathrm{emp},[R]})>\varepsilon
\right)
=0.
\]
This is a sampling-level comparison only; the main particle theorem is still formulated
with i.i.d. initialization from the deterministic truncated law \(\mu^{[R]}\).
\end{remark}

\begin{definition}[Noncompact admissibility]
\label{def:nc-admissible-law}
Fix \(\beta>0\) and \(T>0\).  A law \(f_0\in\mathcal P_1(\R^d)\) is called
\((\beta,T)\)-admissible for the exact Gaussian drift if the following two conditions hold.

\begin{enumerate}
\renewcommand{\labelenumi}{(\roman{enumi})}
\item There exists a finite-action weak solution
\[
f^\beta\in C([0,T];\mathcal P_1(\R^d))
\]
of \eqref{eq:nc-noncompact-continuity-equation} with initial datum \(f_0\).  This solution is part of the admissibility datum and is called the associated admissible solution.

\item For all sufficiently large \(R\), let
\[
f_0^{[R]}:=\frac{\mathbf 1_{B_R}f_0}{f_0(B_R)},
\]
and let \(f^{\beta,[R]}\in C([0,T];\mathcal P(B_R))\) be the compact-support
solution constructed in Proposition~\ref{prop:sc-deterministic-wellposedness-stability},
initialized from \(f_0^{[R]}\).  Then
\begin{equation}
\label{eq:nc-admissible-truncated-to-noncompact}
\sup_{0\le t\le T}W_1(f_t^{\beta,[R]},f_t^\beta)
\longrightarrow0
\qquad\text{as }R\to\infty.
\end{equation}
\end{enumerate}
\end{definition}

\begin{remark}[Uniqueness of the associated admissible flow]
\label{rem:unique-associated-admissible-flow}
For fixed \((\beta,T)\) and \(f_0\), the associated admissible flow is unique whenever it
exists.  Indeed, suppose \(f^\beta\) and \(g^\beta\) both satisfy
Definition~\ref{def:nc-admissible-law} for the same initial law \(f_0\).  Let
\(f^{\beta,[R]}\) denote the compact-support solution initialized from
\(f_0^{[R]}\).  Then, for every \(t\in[0,T]\),
\[
W_1(f_t^\beta,g_t^\beta)
\le
W_1(f_t^\beta,f_t^{\beta,[R]})
+
W_1(f_t^{\beta,[R]},g_t^\beta).
\]
Taking the supremum over \(t\in[0,T]\) and then letting \(R\to\infty\), the two
terms on the right-hand side vanish by the admissibility condition.  Hence
\(f_t^\beta=g_t^\beta\) for all \(t\in[0,T]\).
\end{remark}

\subsubsection{From noncompact admissibility to hard-truncated particles}
\label{subsec:admissibility-to-particles}

\begin{theorem}[Particle approximation for noncompact admissible data]
\label{thm:nc-admissible-particle-approximation}
Fix \(\beta>0\), \(T>0\), and \(t\in[0,T]\).  Let \(f_0\in\mathcal P_1(\R^d)\) be
\((\beta,T)\)-admissible, with admissible noncompact flow \((f_s^\beta)_{0\le s\le T}\).
For each sufficiently large \(R\), let
\[
X_0^{1,N,\beta,[R]},\ldots,X_0^{N,N,\beta,[R]}
\]
be i.i.d. with law
\[
f_0^{[R]}=\frac{\mathbf 1_{B_R}f_0}{f_0(B_R)}.
\]
For \(j=1,\ldots,N\), let \(X_s^{j,N,\beta,[R]}\) denote the solution of the exact
\(N\)-particle system with inverse temperature \(\beta\) and truncation radius \(R\):
\[
\dot X_s^{j,N,\beta,[R]}
=
X_\beta[\mu_s^{N,\beta,[R]}]\bigl(X_s^{j,N,\beta,[R]}\bigr),
\qquad
\mu_s^{N,\beta,[R]}:=\frac1N\sum_{\ell=1}^N
\delta_{X_s^{\ell,N,\beta,[R]}}.
\]
Thus the superscript \([R]\) records the hard-truncated initial law, while \(\beta\)
records the drift parameter.  The only randomness in this theorem is the i.i.d. initial
sample; conditional on that sample, the particle ODE is deterministic.  Then, for every
\(M<\infty\),
\begin{equation}
\label{eq:nc-admissible-particle-approximation}
\lim_{R\to\infty}\limsup_{N\to\infty}
\EE\sup_{|y|\le M}
\left|m_{\mu_t^{N,\beta,[R]}}(y)-m_{f_t^\beta}(y)\right|=0.
\end{equation}
Consequently, the same convergence holds in probability.
\end{theorem}

\begin{proof}
Fix \(R\).  Since \(f_0^{[R]}\in\mathcal P(B_R)\), the compact-support propagation
of chaos theorem gives
\[
\EE W_1(\mu_t^{N,\beta,[R]},f_t^{\beta,[R]})\to0.
\]
Hence \(W_1(\mu_t^{N,\beta,[R]},f_t^{\beta,[R]})\to0\) in probability.  By
Proposition~\ref{prop:nc-posterior-mean-W1},
\[
\sup_{|y|\le M}
\left|m_{\mu_t^{N,\beta,[R]}}(y)-m_{f_t^{\beta,[R]}}(y)\right|
\to0
\]
in probability.  Moreover both measures are supported in \(B_R\), so both posterior
means take values in \(B_R\), and the last display is bounded by \(2R\).  Since the
convergence is in probability and the sequence is uniformly bounded, the expectations
also converge to zero:
\[
\lim_{N\to\infty}
\EE\sup_{|y|\le M}
\left|m_{\mu_t^{N,\beta,[R]}}(y)-m_{f_t^{\beta,[R]}}(y)\right|=0.
\]
Therefore
\[
\begin{aligned}
&\limsup_{N\to\infty}
\EE\sup_{|y|\le M}
\left|m_{\mu_t^{N,\beta,[R]}}(y)-m_{f_t^\beta}(y)\right| \\
&\qquad\le
\sup_{|y|\le M}
\left|m_{f_t^{\beta,[R]}}(y)-m_{f_t^\beta}(y)\right|.
\end{aligned}
\]
By admissibility,
\[
W_1(f_t^{\beta,[R]},f_t^\beta)\to0.
\]
Proposition~\ref{prop:nc-posterior-mean-W1} then implies that the right-hand side tends
to zero as \(R\to\infty\).

To make the last assertion explicit, set
\[
Z_{N,R}:=
\sup_{|y|\le M}
\left|m_{\mu_t^{N,\beta,[R]}}(y)-m_{f_t^\beta}(y)\right|.
\]
The preceding argument proves
\[
\lim_{R\to\infty}\limsup_{N\to\infty}\mathbb E Z_{N,R}=0.
\]
Therefore, for every \(\varepsilon>0\), Markov's inequality gives
\[
\lim_{R\to\infty}\limsup_{N\to\infty}
\mathbb P(Z_{N,R}>\varepsilon)
\le
\lim_{R\to\infty}\limsup_{N\to\infty}
\frac{\mathbb E Z_{N,R}}{\varepsilon}
=0.
\]
This is the claimed sequential convergence in probability.
\end{proof}

\subsubsection{Recoverable admissible priors}
\label{subsec:recoverable-admissible-priors}

The previous theorem separates particle approximation from the large-\(\beta\)
recovery step.  We now record the corresponding class of priors for which the full
sequential posterior-mean recovery theorem follows.

\begin{definition}[Recoverable admissible priors]
\label{def:recoverable-admissible-prior-class}
Fix \(\tau>0\).  We measure time in the raw attention-depth variable associated with the
normalized drift \(X_\beta=\beta^{-1}\nabla\log(G_\beta*\cdot)\).  Equivalently, if
\(s=t/\beta\) denotes the unnormalized score-flow time, then the denoising horizon
\(s=\tau/2\) corresponds to
\[
T_\beta:=\frac{\beta\tau}{2}.
\]
A prior \(P_0\in\mathcal P_1(\R^d)\) belongs to the class
\(\mathcal A_\tau\) if, with
\[
f_0:=\gamma_\tau*P_0,
\]
the following two conditions hold.

\begin{enumerate}
\renewcommand{\labelenumi}{(\roman{enumi})}

\item For every \(\beta>0\), there exists an admissible flow
\[
(f_t^\beta)_{0\le t\le T_\beta}
\]
such that the noisy law \(f_0\) is \((\beta,T_\beta)\)-admissible in the sense of
Definition~\ref{def:nc-admissible-law}, with associated admissible solution
\((f_t^\beta)_{0\le t\le T_\beta}\).

\item This choice of admissible flow recovers the clean prior at the denoising time:
\begin{equation}
\label{eq:recoverable-prior-large-beta-condition}
W_1(f_{T_\beta}^\beta,P_0)\longrightarrow0
\qquad\text{as }\beta\to\infty.
\end{equation}

\end{enumerate}
\end{definition}

\begin{remark}[Role of the class \(\mathcal A_\tau\)]
\label{rem:role-of-A-tau}
The class \(\mathcal A_\tau\) is an admissible-recoverable class, not an
a priori standard regularity class of priors.  Its first condition is the
noncompact mean-field admissibility needed to pass from hard-truncated
compactly supported particles to the noncompact noisy law \(f_0\).  Its second
condition is the large-\(\beta\) denoising condition: after the raw attention-depth
time \(T_\beta=\beta\tau/2\), the associated deterministic noncompact flow
recovers the clean prior \(P_0\) in \(W_1\).

Thus the theorem below is a transfer result for priors satisfying these two conditions.
The concrete case needed here is the Gaussian one, verified later in this appendix.
\end{remark}

\begin{theorem}[Hard-truncated posterior-mean recovery for recoverable priors]
\label{thm:recoverable-prior-hard-truncated-posterior-recovery}
Fix \(\tau>0\), and let \(P_0\in\mathcal A_\tau\).  Set
\[
f_0:=\gamma_\tau*P_0,
\qquad
T_\beta:=\frac{\beta\tau}{2}.
\]
For \(R>0\), define
\[
f_0^{[R]}:=\frac{\mathbf 1_{B_R}f_0}{f_0(B_R)}.
\]
Let \(\mu_t^{N,\beta,[R]}\) be the empirical measure of the exact particle system
initialized i.i.d. from \(f_0^{[R]}\):
\[
\dot X_i^{N,\beta,[R]}(t)
=
X_\beta[\mu_t^{N,\beta,[R]}]\bigl(X_i^{N,\beta,[R]}(t)\bigr),
\qquad
\mu_t^{N,\beta,[R]}:=\frac1N\sum_{j=1}^N\delta_{X_j^{N,\beta,[R]}(t)}.
\]
Then, for every \(M<\infty\),
\begin{equation}
\label{eq:recoverable-prior-hard-truncated-theorem}
\lim_{\beta\to\infty}
\limsup_{R\to\infty}
\limsup_{N\to\infty}
\EE
\sup_{|y|\le M}
\left|
m_{\mu_{T_\beta}^{N,\beta,[R]}}(y)-m_{P_0}(y)
\right|
=0.
\end{equation}
The order of limits in \eqref{eq:recoverable-prior-hard-truncated-theorem} is part of
the statement: first \(N\to\infty\) at fixed \((\beta,R)\), then \(R\to\infty\) at fixed
\(\beta\), and finally \(\beta\to\infty\).  In particular, this is not a joint scaling
statement in \((N,R,\beta)\).  Consequently, the same convergence holds in probability.
\end{theorem}

\begin{proof}
Fix \(\beta>0\).  Since \(P_0\in\mathcal A_\tau\), the noisy law \(f_0=\gamma_\tau*P_0\)
is \((\beta,T_\beta)\)-admissible.  Applying
Theorem~\ref{thm:nc-admissible-particle-approximation} with \(T=T_\beta\) and
\(t=T_\beta\), we obtain
\[
\limsup_{R\to\infty}\limsup_{N\to\infty}
\EE\sup_{|y|\le M}
\left|
m_{\mu_{T_\beta}^{N,\beta,[R]}}(y)-m_{f_{T_\beta}^\beta}(y)
\right|=0.
\]
Therefore
\[
\begin{aligned}
&\limsup_{R\to\infty}\limsup_{N\to\infty}
\EE\sup_{|y|\le M}
\left|
m_{\mu_{T_\beta}^{N,\beta,[R]}}(y)-m_{P_0}(y)
\right| \\
&\qquad\le
\sup_{|y|\le M}
\left|
m_{f_{T_\beta}^\beta}(y)-m_{P_0}(y)
\right|.
\end{aligned}
\]
By the defining recovery condition \eqref{eq:recoverable-prior-large-beta-condition},
\[
W_1(f_{T_\beta}^\beta,P_0)\to0
\qquad\text{as }\beta\to\infty.
\]
Proposition~\ref{prop:nc-posterior-mean-W1} gives
\[
\sup_{|y|\le M}
\left|
m_{f_{T_\beta}^\beta}(y)-m_{P_0}(y)
\right|
\to0.
\]
This proves \eqref{eq:recoverable-prior-hard-truncated-theorem}.  If
\[
Z_{N,R,\beta}:=\sup_{|y|\le M}
\left|
 m_{\mu_{T_\beta}^{N,\beta,[R]}}(y)-m_{P_0}(y)
\right|,
\]
then for every \(\varepsilon>0\), Markov's inequality gives
\[
\mathbb P(Z_{N,R,\beta}>\varepsilon)
\le \varepsilon^{-1}\mathbb E Z_{N,R,\beta},
\]
and the same ordered limits therefore give convergence in probability.
\end{proof}

\subsection{Gaussian-prior noisy laws are noncompact admissible}
\label{sec:gaussian-noncompact-admissible}

We now verify admissibility for noisy laws obtained by convolving a Gaussian prior with a Gaussian observation kernel. The proof uses the explicit Gaussian solution, compactness of the hard-truncated compact-support flows, and weak--strong uniqueness around the Gaussian solution.

Fix \(\tau>0\).  Let \(m\in\R^d\), let \(\Sigma_0\) be a symmetric nonnegative semidefinite matrix, and set
\[
P_0=\cN(m,\Sigma_0),
\qquad
f_0:=\gamma_\tau*P_0=\cN(m,\Gamma_0),
\qquad
\Gamma_0:=\Sigma_0+\tau I.
\]
Here \(\cN(m,\Sigma_0)\) is understood in the usual possibly degenerate sense: if \(\Sigma_0\) is singular, it is the law of \(m+\Sigma_0^{1/2}Z\) with \(Z\sim \cN(0,I_d)\), supported on \(m+\operatorname{Ran}\Sigma_0^{1/2}\).  Thus \(\Gamma_0\) is strictly positive definite.

\begin{remark}[Sampling interpretation of hard truncation]
The sequential theorem below initializes the particle system directly from the normalized
truncation \(f_0^{[R]}\).  This is different from drawing \(Y_1,\ldots,Y_N\) from the
untruncated law \(f_0\) and then retaining only the particles in \(B_R\).  For the latter
interpretation,
\[
\mathbb P(Y_1,\ldots,Y_N\in B_R)=f_0(B_R)^N=(1-q_R(f_0))^N
\ge 1-Nq_R(f_0),
\qquad q_R(f_0):=f_0(B_R^c).
\]
If \(f_0=\mathcal N(m,\Gamma_0)\), with
\(\lambda_+=\lambda_{\max}(\Gamma_0)\), then standard Gaussian tail bounds give, for
\(R\ge1+2|m|\),
\[
q_R(f_0)\le
C(1+R)^d\exp\left(-\frac{R^2}{8\lambda_+}\right).
\]
Thus \(\mathbb P(Y_1,\ldots,Y_N\in B_R)\to1\) whenever
\[
N(1+R)^d\exp\left(-\frac{R^2}{8\lambda_+}\right)\to0.
\]
This sampling interpretation is not used in the sequential particle theorem.
\end{remark}

\subsubsection{The explicit noncompact Gaussian flow}
\label{subsec:gaussian-explicit-flow}

\begin{lemma}[Gaussian invariance and covariance equation]
\label{lem:gaussian-invariance-covariance}
Fix \(\beta>0\).  There is a unique global positive definite solution
\(\Gamma_t\) of the covariance equation below.  The Gaussian curve
\[
f_t^\beta=\cN(m,\Gamma_t)
\]
solves
\begin{equation}
\label{eq:gaussian-noncompact-pde}
\partial_t f_t+\nabla\cdot(f_tX_\beta[f_t])=0,
\qquad
f_{t=0}=f_0,
\end{equation}
where \(\Gamma_t\) is the positive definite solution of
\begin{equation}
\label{eq:covariance-ode}
\dot\Gamma_t=-2\Gamma_t(I+\beta\Gamma_t)^{-1},
\qquad
\Gamma_0=\Sigma_0+\tau I.
\end{equation}
Equivalently, after diagonalizing \(\Gamma_0\), each eigenvalue \(\lambda_i(t)\) of
\(\Gamma_t\) solves
\begin{equation}
\label{eq:eigenvalue-ode}
\dot\lambda_i(t)=-\frac{2\lambda_i(t)}{1+\beta\lambda_i(t)}.
\end{equation}
\end{lemma}

\begin{proof}
We first record why the covariance ODE is globally well posed and remains positive
definite.  Diagonalize \(\Gamma_0\), and for each initial eigenvalue \(\lambda_i(0)>0\)
consider
\[
\dot\lambda_i(t)=-\frac{2\lambda_i(t)}{1+\beta\lambda_i(t)}.
\]
The function
\[
F_\beta(r):=r+\beta^{-1}\log r,
\qquad r>0,
\]
is strictly increasing and maps \((0,\infty)\) onto \(\mathbb R\).  Along a solution,
\[
\frac{d}{dt}F_\beta(\lambda_i(t))=-\frac2\beta,
\]
so
\[
F_\beta(\lambda_i(t))=F_\beta(\lambda_i(0))-\frac{2t}{\beta}.
\]
For every finite \(t\), this equation has a unique solution \(\lambda_i(t)>0\).  Hence
\(\Gamma_t\) is uniquely defined, remains positive definite for all finite \(t\), and has
the same eigenspaces as \(\Gamma_0\).

If \(f\) is the Gaussian law \(\cN(m,\Gamma)\), with \(\Gamma>0\), then
\(G_\beta*f\) is the density of the Gaussian law
\[
\cN(m,\Gamma+\beta^{-1}I).
\]
Hence
\[
\nabla\log(G_\beta*f)(x)=-(\Gamma+\beta^{-1}I)^{-1}(x-m),
\]
and therefore
\[
X_\beta[f](x)
=
-\frac1\beta(\Gamma+\beta^{-1}I)^{-1}(x-m)
=
-(I+\beta\Gamma)^{-1}(x-m).
\]
This velocity field is affine.  Its flow fixes the mean \(m\).  The covariance satisfies
\[
\dot\Gamma_t
=
-(I+\beta\Gamma_t)^{-1}\Gamma_t
-
\Gamma_t(I+\beta\Gamma_t)^{-1}.
\]
Since \((I+\beta\Gamma_t)^{-1}\) is a matrix function of \(\Gamma_t\), the two matrices
commute.  This gives \eqref{eq:covariance-ode}.  Conversely, if \(\Gamma_t\)
solves \eqref{eq:covariance-ode}, then the affine characteristic flow
\(
\dot x_t=-(I+\beta\Gamma_t)^{-1}(x_t-m)
\)
pushes \(\cN(m,\Gamma_0)\) forward to \(\cN(m,\Gamma_t)\).  Therefore the Gaussian
curve satisfies \eqref{eq:gaussian-noncompact-pde} in weak form.  The eigenvalue
equation follows by diagonalizing the spectral ODE.
\end{proof}

\begin{lemma}[Gaussian verification of the recovery condition]
\label{lem:gaussian-recovery-condition}
Let
\[
T_\beta:=\frac{\beta\tau}{2}.
\]
Then
\[
W_1(f_{T_\beta}^\beta,P_0)\longrightarrow0
\qquad\text{as }\beta\to\infty.
\]
\end{lemma}

\begin{proof}
Let \(s_i\ge0\) be the eigenvalues of \(\Sigma_0\).  The eigenvalues
\(\lambda_i(t)\) of \(\Gamma_t\) solve
\[
\dot\lambda_i(t)=-\frac{2\lambda_i(t)}{1+\beta\lambda_i(t)},
\qquad
\lambda_i(0)=s_i+\tau.
\]
Hence
\[
\frac{d}{dt}
\left(\lambda_i(t)+\frac1\beta\log\lambda_i(t)\right)
=-\frac2\beta.
\]
At \(T_\beta=\beta\tau/2\),
\[
\lambda_i(T_\beta)+\frac1\beta\log\lambda_i(T_\beta)
=
s_i+\frac1\beta\log(s_i+\tau).
\]
If \(s_i>0\), the functions
\[
F_\beta(r):=r+\beta^{-1}\log r
\]
are strictly increasing and converge locally uniformly to \(r\) near \(s_i\).
Thus \(\lambda_i(T_\beta)\to s_i\).

If \(s_i=0\), then \(0<\lambda_i(T_\beta)\le\tau\).  If a subsequence converged
to \(\ell>0\), then passing to the limit in the preceding identity would give
\(\ell=0\), a contradiction.  Hence \(\lambda_i(T_\beta)\to0\).

The covariance ODE is spectral, so the eigenspaces are fixed in time.  Therefore
\[
\Gamma_{T_\beta}\to\Sigma_0
\]
as symmetric nonnegative matrices.  Couple
\[
Y_\beta=m+\Gamma_{T_\beta}^{1/2}Z,
\qquad
Y_0=m+\Sigma_0^{1/2}Z,
\qquad
Z\sim\mathcal N(0,I_d).
\]
By continuity of the matrix square-root on the positive semidefinite cone,
\[
W_1(f_{T_\beta}^\beta,P_0)
\le
\mathbb E|Y_\beta-Y_0|
\le
(\mathbb E|Z|^2)^{1/2}
\|\Gamma_{T_\beta}^{1/2}-\Sigma_0^{1/2}\|_{\mathrm{HS}}
\to0.
\]
\end{proof}

\begin{remark}[Time normalization]
The drift is normalized as
\[
X_\beta[\rho]=\attmean{\rho}-\operatorname{Id}
=\beta^{-1}\nabla\log(G_\beta*\rho).
\]
Thus the raw attention-depth time \(t\) is related to the unnormalized score time \(s\)
by \(s=t/\beta\).  The terminal time
\[
T_\beta=\frac{\beta\tau}{2}
\]
therefore corresponds to the usual backward-heat denoising time \(s=\tau/2\).  If the
observation-noise variance is denoted by \(\sigma^2\), then \(\tau=\sigma^2\), and this
becomes \(T_\beta=\beta\sigma^2/2\) and \(s=\sigma^2/2\).
\end{remark}

\subsubsection{Uniform estimates for the hard-truncated Gaussian flows}
\label{subsec:gaussian-uniform-hard-trunc-estimates}

For \(R>0\), define
\[
p_R:=f_0(B_R),
\qquad
f_0^{[R]}:=\frac{\mathbf 1_{B_R}f_0}{p_R},
\]
and let \(f_t^{\beta,[R]}\) be the compact-support solution starting from \(f_0^{[R]}\):
\begin{equation}
\label{eq:hard-trunc-mf-flow}
\partial_t f_t^{\beta,[R]}
+
\nabla\cdot\bigl(f_t^{\beta,[R]}X_\beta[f_t^{\beta,[R]}]\bigr)=0,
\qquad
f_{t=0}^{\beta,[R]}=f_0^{[R]}.
\end{equation}
For fixed \(\beta\) and \(R\), this is precisely the compact-support flow constructed earlier.

\begin{lemma}[Barycentric action estimate]
\label{lem:barycentric-action-estimate}
For every \(\mu\in\mathcal P(\R^d)\) with finite second moment,
\[
\int_{\R^d}|X_\beta[\mu](x)|^2\,\mu(dx)
\le
2\int_{\R^d}|x|^2\,\mu(dx).
\]
\end{lemma}

\begin{proof}
Fix \(x\in\R^d\), and define the probability measure
\[
\pi_x^\mu(dy):=
\frac{G_\beta(x-y)\,\mu(dy)}{\int_{\R^d}G_\beta(x-z)\,\mu(dz)}.
\]
By the barycentric formula,
\[
X_\beta[\mu](x)=\int_{\R^d}(y-x)\,\pi_x^\mu(dy).
\]
Jensen's inequality gives
\[
|X_\beta[\mu](x)|^2
\le
\int_{\R^d}|y-x|^2\,\pi_x^\mu(dy).
\]
Put \(r(y)=|y-x|^2\) and \(w(y)=e^{-\beta r(y)/2}\).  Since \(w\) is a decreasing
function of \(r\), for independent \(Y,Y'\sim\mu\),
\[
\EE\bigl[(r(Y)-r(Y'))(w(Y)-w(Y'))\bigr]\le0.
\]
Expanding this inequality gives
\[
\frac{\int r(y)w(y)\,\mu(dy)}{\int w(y)\,\mu(dy)}
\le
\int r(y)\,\mu(dy).
\]
Therefore
\[
|X_\beta[\mu](x)|^2
\le
\int_{\R^d}|y-x|^2\,\mu(dy).
\]
Integrating in \(x\) against \(\mu(dx)\), we obtain
\[
\int |X_\beta[\mu](x)|^2\,\mu(dx)
\le
\iint |y-x|^2\,\mu(dy)\mu(dx).
\]
If \(\bar x_\mu:=\int x\,\mu(dx)\), then the double integral is computed exactly as
\[
\iint |y-x|^2\,\mu(dy)\mu(dx)
=
2\int |x|^2\,\mu(dx)-2|\bar x_\mu|^2
\le
2\int |x|^2\,\mu(dx).
\]
This proves the estimate with the stated constant.
\end{proof}

\begin{proposition}[Uniform second moment and action]
\label{prop:uniform-moment-action-hard-trunc}
Fix \(\beta>0\) and \(T>0\).  There exists a constant
\(C_{\beta,T,f_0}<\infty\), independent of \(R\), such that, for all sufficiently large
\(R\),
\[
\sup_{0\le t\le T}
\int_{\R^d}|x|^2\,f_t^{\beta,[R]}(dx)
\le C_{\beta,T,f_0},
\]
and
\[
\int_0^T\int_{\R^d}
|X_\beta[f_t^{\beta,[R]}](x)|^2\,f_t^{\beta,[R]}(dx)\,dt
\le C_{\beta,T,f_0}.
\]
\end{proposition}

\begin{proof}
For each fixed \(R\), the compact-support construction in
Proposition~\ref{prop:sc-deterministic-wellposedness-stability} gives a characteristic
flow \(\Phi_t^R:B_R\to B_R\) such that
\[
f_t^{\beta,[R]}=(\Phi_t^R)_\# f_0^{[R]},
\qquad
\frac{d}{dt}\Phi_t^R(x)=X_\beta[f_t^{\beta,[R]}](\Phi_t^R(x)).
\]
The map \((t,x)\mapsto \Phi_t^R(x)\) is continuously differentiable in \(t\) and bounded
on \([0,T]\times B_R\).  Therefore
\[
M_R(t):=\int |x|^2\,f_t^{\beta,[R]}(dx)
=\int_{B_R}|\Phi_t^R(x)|^2\,f_0^{[R]}(dx)
\]
is absolutely continuous and, for a.e. \(t\),
\[
\frac{d}{dt}M_R(t)
=
2\int \Phi_t^R(x)\cdot
X_\beta[f_t^{\beta,[R]}](\Phi_t^R(x))\,f_0^{[R]}(dx).
\]
Equivalently,
\[
\frac{d}{dt}M_R(t)
=
2\int y\cdot X_\beta[f_t^{\beta,[R]}](y)\,f_t^{\beta,[R]}(dy).
\]
Cauchy--Schwarz and Lemma~\ref{lem:barycentric-action-estimate} give
\[
\frac{d}{dt}M_R(t)
\le
2M_R(t)^{1/2}
\left(
\int |X_\beta[f_t^{\beta,[R]}]|^2\,df_t^{\beta,[R]}
\right)^{1/2}
\le
2\sqrt2\,M_R(t).
\]
Hence
\[
M_R(t)\le e^{2\sqrt2t}M_R(0),
\qquad 0\le t\le T.
\]
Since \(f_0\) is Gaussian,
\[
M_R(0)
=
\frac{\int_{B_R}|x|^2\,f_0(dx)}{f_0(B_R)}
\le
2\int_{\R^d}|x|^2\,f_0(dx)
\]
for all sufficiently large \(R\).  This proves the uniform second-moment bound.
Integrating the barycentric action estimate in time gives
\[
\int_0^T\int |X_\beta[f_t^{\beta,[R]}]|^2\,df_t^{\beta,[R]}\,dt
\le
2\int_0^T M_R(t)\,dt
\le 2T\sup_{0\le t\le T}M_R(t),
\]
which is bounded independently of \(R\).
\end{proof}

\subsubsection{Compactness and identification of noncompact limits}
\label{subsec:gaussian-compactness-identification}

\begin{theorem}[Compactness and PDE identification]
\label{thm:hard-trunc-compactness-identification}
Fix \(\beta>0\) and \(T>0\).  Let \(R_k\to\infty\).  Then there exists a subsequence,
not relabeled, and a curve
\[
\rho\in C([0,T];\mathcal P_1(\R^d))
\]
such that
\[
\sup_{0\le t\le T}W_1(f_t^{\beta,[R_k]},\rho_t)\to0.
\]
Moreover, \(\rho_{t=0}=f_0\), and \(\rho\) is a finite-action weak solution of
\[
\partial_t\rho_t+\nabla\cdot(\rho_tX_\beta[\rho_t])=0
\]
on \([0,T]\).  In addition,
\[
\sup_{0\le t\le T}\int |x|^2\,\rho_t(dx)<\infty,
\qquad
\int_0^T\int |X_\beta[\rho_t](x)|^2\,\rho_t(dx)\,dt<\infty.
\]
\end{theorem}

\begin{proof}
By Proposition~\ref{prop:uniform-moment-action-hard-trunc},
\[
\sup_k\sup_{0\le t\le T}\int |x|^2\,f_t^{\beta,[R_k]}(dx)<\infty.
\]
Hence
\[
\sup_k\sup_{0\le t\le T}
\int_{|x|>A}|x|\,f_t^{\beta,[R_k]}(dx)
\le \frac{C}{A}\to0.
\]
The same proposition gives a uniform \(L^2\)-action bound.  Therefore, for
\(0\le s<t\le T\), Kantorovich--Rubinstein duality and Cauchy--Schwarz give
\[
W_1(f_t^{\beta,[R_k]},f_s^{\beta,[R_k]})
\le
\int_s^t\int |X_\beta[f_r^{\beta,[R_k]}]|\,df_r^{\beta,[R_k]}\,dr
\le C|t-s|^{1/2},
\]
uniformly in \(k\).  The preceding tail estimate gives tightness and uniform integrability of first
moments, hence relative compactness of the time slices in \(W_1\) by Prokhorov's
theorem and \cite[Definition~6.8 and Theorem~6.9]{Villani2009OptimalTransport}.
Together with the uniform \(W_1\)-equicontinuity, the metric Arzelà--Ascoli theorem
yields a subsequence, not relabeled, and a curve
\(\rho\in C([0,T];\mathcal P_1(\R^d))\) such that:
\[
\sup_{0\le t\le T}W_1(f_t^{\beta,[R_k]},\rho_t)\to0.
\]
The initial condition follows from Lemma~\ref{lem:nc-hard-truncation-W1}.  Lower
semicontinuity of the second moment gives
\[
\sup_{0\le t\le T}\int |x|^2\,\rho_t(dx)<\infty,
\]
and Lemma~\ref{lem:barycentric-action-estimate} gives the finite \(L^2\)-action bound.

It remains to pass to the limit in the weak formulation.  Let \(\phi\in C_c^1(\mathbb R^d)\)
and \(K=\supp\nabla\phi\).  If \(\mu_n\to\mu\) in \(W_1\), then
\(G_\beta*\mu_n\to G_\beta*\mu\) and \(\nabla G_\beta*\mu_n\to\nabla G_\beta*\mu\)
uniformly on \(K\), because the translated kernels are bounded Lipschitz uniformly on
\(K\).  Since \(G_\beta*\mu\) is strictly positive on \(K\), it follows that
\[
X_\beta[\mu_n]\to X_\beta[\mu]
\quad\text{uniformly on }K.
\]
Thus the tested fluxes converge pointwise in time.  The uniform \(L^2\)-action bound gives
uniform integrability in time, so Vitali's theorem permits passage to the limit in the
time integral.  For each \(R_k\), the compact-support characteristic solution satisfies the weak
formulation by the chain rule along characteristics.  Passing to the limit in that weak
formulation gives the finite-action weak formulation for \(\rho\), in the sense of
Definition~\ref{def:nc-finite-action-weak-solution}.
\end{proof}

\subsubsection{Weak--strong uniqueness around the Gaussian flow}
\label{subsec:gaussian-weak-strong-uniqueness}

Fix \(\beta>0\) and \(T>0\), and write
\[
g_t:=f_t^\beta=\cN(m,\Gamma_t).
\]
On \([0,T]\), the matrices \(\Gamma_t\) are uniformly positive definite and uniformly
bounded.  For \(t\in[0,T]\) and \(x\in\mathbb R^d\), define
\[
r_t(x):=|\Gamma_t^{-1/2}(x-m)|.
\]
The function \(r_t\) is uniformly Lipschitz in \(x\), and \(1+r_t(x)\) is uniformly
equivalent to \(1+|x|\) on \([0,T]\times\R^d\).

\begin{lemma}[Reference Gaussian convolution bounds]
\label{lem:reference-gaussian-convolution-bounds}
Fix \(\beta>0\) and \(T>0\).  There exist constants
\[
c>0,\qquad C<\infty,\qquad \alpha\in(0,1),
\]
depending only on \((\beta,T,\Gamma_0,d)\), such that, for every \(t\in[0,T]\),
with
\[
D_{g_t}:=G_\beta*g_t,
\qquad
H_{g_t}:=\nabla(G_\beta*g_t),
\]
one has
\[
D_{g_t}(x)\ge c e^{-\alpha r_t(x)^2/2},
\qquad
\frac{|H_{g_t}(x)|}{D_{g_t}(x)}\le C(1+r_t(x))
\]
for every \(x\in\mathbb R^d\).  Moreover, \(r_t\) is uniformly Lipschitz in \(x\) for
\(t\in[0,T]\), and if \(X\sim g_t\), then \(r_t(X)\) has the same law as
\(|Z|\), where \(Z\sim\cN(0,I_d)\).
\end{lemma}

\begin{proof}
Since \(g_t=\cN(m,\Gamma_t)\), the convolution \(D_{g_t}=G_\beta*g_t\) is the
Gaussian density with mean \(m\) and covariance \(\Gamma_t+\beta^{-1}I\).  Thus
\[
D_{g_t}(x)
=
C_t
\exp\left(
-\frac12\left\langle
(\Gamma_t+\beta^{-1}I)^{-1}(x-m),x-m
\right\rangle
\right),
\]
where \(C_t\) is bounded above and below by positive constants on \([0,T]\), because
\(\Gamma_t\) stays uniformly positive definite and uniformly bounded.

Writing \(u=\Gamma_t^{-1/2}(x-m)\), the exponent becomes
\[
\left\langle
\Gamma_t^{1/2}(\Gamma_t+\beta^{-1}I)^{-1}\Gamma_t^{1/2}u,u
\right\rangle
=
\left\langle
\Gamma_t(\Gamma_t+\beta^{-1}I)^{-1}u,u
\right\rangle .
\]
The eigenvalues of
\[
\Gamma_t(\Gamma_t+\beta^{-1}I)^{-1}
\]
are
\[
\frac{\lambda_i(t)}{\lambda_i(t)+\beta^{-1}}.
\]
Since $0<\sup_{0\le t\le T,i}\lambda_i(t)\le \Lambda_T<\infty$, these eigenvalues are
bounded above by
\[
\frac{\Lambda_T}{\Lambda_T+\beta^{-1}}<1.
\]
Thus the strict gap from \(1\) is uniform for \(t\in[0,T]\).  Hence there exists
\(\alpha\in(0,1)\) such that
\[
\left\langle
(\Gamma_t+\beta^{-1}I)^{-1}(x-m),x-m
\right\rangle
\le \alpha r_t(x)^2.
\]
This gives the lower bound for \(D_{g_t}\).

For the numerator,
\[
H_{g_t}(x)=\nabla D_{g_t}(x)
=
-D_{g_t}(x)(\Gamma_t+\beta^{-1}I)^{-1}(x-m).
\]
Therefore
\[
\frac{|H_{g_t}(x)|}{D_{g_t}(x)}
\le C|\Gamma_t^{-1/2}(x-m)|
=
C r_t(x).
\]
Increasing \(C\) gives \(C(1+r_t(x))\).

The uniform Lipschitz bound for \(r_t\) follows from the uniform upper bound on
\(\|\Gamma_t^{-1/2}\|\).  Finally, if \(X\sim\cN(m,\Gamma_t)\), then
\[
\Gamma_t^{-1/2}(X-m)\sim\cN(0,I_d),
\]
so \(r_t(X)\) has the same law as \(|Z|\).
\end{proof}

\begin{lemma}[Elementary logarithmic cutoff estimates]
\label{lem:elementary-log-cutoff-estimates}
Let \(0<\alpha<1\), set
\[
\kappa:=\frac{1+\alpha}{\alpha}>2,
\]
and let \(C_0\ge e\).  For \(0<\delta<1\), define
\[
A_\delta:=\sqrt{\kappa\log\frac{C_0}{\delta}}.
\]
Then, for every \(q\ge0\), there exist constants \(C<\infty\) and
\(\delta_0\in(0,1)\), depending only on \((\alpha,q,C_0)\), such that for
\(0<\delta\le\delta_0\),
\[
(1+A_\delta)^q\delta e^{\alpha(A_\delta+1)^2/2}\le C
\]
and
\[
(1+A_\delta)^q e^{-(A_\delta-1)^2/2}\le C\delta .
\]
\end{lemma}

\begin{proof}
The first estimate follows from
\[
\delta e^{\alpha(A_\delta+1)^2/2}
=
\delta e^{\alpha A_\delta^2/2}e^{\alpha A_\delta+\alpha/2}
=
C\,\delta^{(1-\alpha)/2}e^{\alpha A_\delta}.
\]
Since
\[
A_\delta=\sqrt{\kappa\log(C_0/\delta)},
\]
the factor \(e^{\alpha A_\delta}\) grows sub-polynomially in \(\delta^{-1}\).  More
precisely, for every \(\eta>0\), after decreasing \(\delta_0\),
\[
e^{\alpha A_\delta}\le C_\eta \delta^{-\eta}.
\]
Also \((1+A_\delta)^q\le C_\eta\delta^{-\eta}\) for every \(\eta>0\), again after
decreasing \(\delta_0\).  Choosing \(\eta>0\) sufficiently small gives
\[
(1+A_\delta)^q\delta^{(1-\alpha)/2}e^{\alpha A_\delta}\le C,
\]
which proves the first estimate.

For the second estimate,
\[
e^{-(A_\delta-1)^2/2}
=
e^{-A_\delta^2/2}e^{A_\delta-1/2}
=
C\,\delta^{\kappa/2}e^{A_\delta}.
\]
Because \(\kappa/2>1\), we can choose \(\eta>0\) so small that
\[
\kappa/2-\eta>1.
\]
Using again the sub-polynomial bounds
\[
e^{A_\delta}\le C_\eta\delta^{-\eta},
\qquad
(1+A_\delta)^q\le C_\eta\delta^{-\eta},
\]
and decreasing \(\eta\), if necessary, gives
\[
(1+A_\delta)^q e^{-(A_\delta-1)^2/2}
\le C\delta .
\]
\end{proof}

\begin{lemma}[Transferring Gaussian radial tails by \(W_1\)]
\label{lem:gaussian-radial-tail-transfer}
Fix \(\beta>0\), \(T>0\), and \(M_2<\infty\).  There exist constants
\(C<\infty\) and \(q<\infty\), depending only on \((\beta,T,m,\Gamma_0,M_2,d)\),
such that the following holds.  If \(t\in[0,T]\), \(\mu\in\mathcal P(\mathbb R^d)\),
\[
\int |x|^2\,\mu(dx)\le M_2,
\qquad
\delta:=W_1(\mu,g_t),
\]
then for every \(a\ge1\),
\[
\int_{\{r_t>a\}}(1+r_t)\,d\mu
\le
C(1+a)\delta
+
C(1+a)^q e^{-(a-1)^2/2}.
\]
\end{lemma}

\begin{proof}
Let
\[
L_r:=\sup_{0\le t\le T}\operatorname{Lip}(r_t)<\infty.
\]
For \(a\ge1\), choose \(\theta_a:\mathbb R\to[0,1]\) such that
\[
\theta_a(s)=0\quad(s\le a-1),
\qquad
\theta_a(s)=1\quad(s\ge a),
\qquad
\operatorname{Lip}(\theta_a)\le2.
\]
Then
\[
\mathbf 1_{\{r_t>a\}}\le \theta_a(r_t)\le \mathbf 1_{\{r_t>a-1\}},
\qquad
\operatorname{Lip}(\theta_a\circ r_t)\le 2L_r.
\]
By Kantorovich--Rubinstein duality,
\[
\mu\{r_t>a\}
\le
g_t\{r_t>a-1\}+C\delta .
\]
Similarly, the function \((r_t-a)_+\) is \(L_r\)-Lipschitz.  Since it is unbounded,
we apply Kantorovich--Rubinstein duality first to bounded Lipschitz truncations
\(\min\{(r_t-a)_+,M\}\), and then let \(M\to\infty\) using monotone convergence
and the finite first moments.  Thus
\[
\int (r_t-a)_+\,d\mu
\le
\int (r_t-a)_+\,dg_t+C\delta .
\]
Therefore
\[
\begin{aligned}
\int_{\{r_t>a\}}(1+r_t)\,d\mu
&\le
(1+a)\mu\{r_t>a\}+\int (r_t-a)_+\,d\mu  \\
&\le
C(1+a)\delta
+
(1+a)g_t\{r_t>a-1\}
+
\int (r_t-a)_+\,dg_t .
\end{aligned}
\]
Under \(g_t\), the variable \(r_t\) has the same law as \(|Z|\) with
\(Z\sim\cN(0,I_d)\).  Hence the standard Gaussian radial tail estimate gives, for
some \(q=q(d)<\infty\),
\[
(1+a)g_t\{r_t>a-1\}
+
\int (r_t-a)_+\,dg_t
\le
C(1+a)^q e^{-(a-1)^2/2}.
\]
Indeed, the radial density of \(|Z|\) is
\[
c_d s^{d-1}e^{-s^2/2}\mathbf 1_{s\ge0},
\]
which gives the stated polynomial-times-Gaussian tail bound.
\end{proof}

\begin{lemma}[Gaussian-reference logarithmic stability estimate]
\label{lem:gaussian-reference-osgood-stability}
Fix \(\beta>0\), \(T>0\), and \(M_2<\infty\).  There exists
\(C=C(\beta,T,m,\Gamma_0,M_2,d)<\infty\) such that the following holds for every
\(t\in[0,T]\) and every \(\mu\in\mathcal P(\R^d)\) satisfying
\[
\int |x|^2\,\mu(dx)\le M_2.
\]
Let
\[
\delta:=W_1(\mu,g_t).
\]
Then
\begin{equation}
\label{eq:gaussian-reference-osgood-estimate}
\int_{\R^d}
|X_\beta[\mu](x)-X_\beta[g_t](x)|\,\mu(dx)
\le
C\,\delta
\left(
1+\sqrt{\log\frac{C}{\delta}}
\right),
\end{equation}
with the convention that the right-hand side is zero when \(\delta=0\).  The constant
\(C\) is chosen large enough that \(C/\delta\ge e\) whenever the estimate is nontrivial.
\end{lemma}

\begin{proof}
If \(\delta=0\), then \(\mu=g_t\), and the estimate is immediate.

We first dispose of the large-\(\delta\) range.  By Lemma~\ref{lem:barycentric-action-estimate},
\[
\int |X_\beta[\mu]|\,d\mu\le (2M_2)^{1/2}.
\]
The reference field is affine:
\[
X_\beta[g_t](x)=-(I+\beta\Gamma_t)^{-1}(x-m),
\]
and the matrices \(\Gamma_t\) are uniformly positive definite and uniformly bounded
on \([0,T]\).  Hence
\[
\int |X_\beta[g_t](x)|\,\mu(dx)\le C(1+M_2^{1/2}).
\]
Also
\[
\delta=W_1(\mu,g_t)
\le
\int |x|\,d\mu+\int |x|\,dg_t
\le C(1+M_2^{1/2}).
\]
Thus the left-hand side of \eqref{eq:gaussian-reference-osgood-estimate} is bounded
by a constant depending only on the fixed parameters.  Therefore, once the estimate is
proved for \(0<\delta\le\delta_0\), the remaining range \(\delta\ge\delta_0\) follows
by increasing \(C\).  It remains to prove the estimate for sufficiently small \(\delta\).

Set
\[
D_\nu(x):=(G_\beta*\nu)(x),
\qquad
H_\nu(x):=\nabla(G_\beta*\nu)(x).
\]
The functions \(y\mapsto G_\beta(x-y)\) and
\(y\mapsto \partial_iG_\beta(x-y)\) are bounded and globally Lipschitz, with bounds
independent of \(x\).  Therefore Kantorovich--Rubinstein duality gives
\begin{equation}
\label{eq:DH-W1-control-new}
\|D_\mu-D_{g_t}\|_{L^\infty(\R^d)}
+
\|H_\mu-H_{g_t}\|_{L^\infty(\R^d)}
\le C\delta .
\end{equation}

Let \(c,C,\alpha\) be as in Lemma~\ref{lem:reference-gaussian-convolution-bounds}.
Choose \(C_0\ge e\), and define
\[
\kappa:=\frac{1+\alpha}{\alpha},
\qquad
A_\delta:=\sqrt{\kappa\log\frac{C_0}{\delta}},
\qquad
E_\delta(t):=\{x:\ r_t(x)\le A_\delta\}.
\]
We split the integral into the core \(E_\delta(t)\) and its complement.

On \(E_\delta(t)\), Lemma~\ref{lem:reference-gaussian-convolution-bounds} gives
\[
D_{g_t}(x)\ge c e^{-\alpha A_\delta^2/2}.
\]
By \eqref{eq:DH-W1-control-new},
\[
\frac{|D_\mu(x)-D_{g_t}(x)|}{D_{g_t}(x)}
\le
C\delta e^{\alpha A_\delta^2/2}
=
C C_0^{(1+\alpha)/2}\delta^{(1-\alpha)/2}.
\]
After fixing \(C_0\), choose \(\delta_0>0\) so small that this last quantity is at
most \(1/2\) for every \(0<\delta\le\delta_0\).  Then
\begin{equation}
\label{eq:core-denominator-comparison-new}
D_\mu(x)\ge \frac12D_{g_t}(x)
\qquad\text{for }x\in E_\delta(t).
\end{equation}

Using \(X_\beta[\nu]=\beta^{-1}H_\nu/D_\nu\),
\eqref{eq:DH-W1-control-new},
\eqref{eq:core-denominator-comparison-new}, and
Lemma~\ref{lem:reference-gaussian-convolution-bounds}, we obtain on \(E_\delta(t)\)
\[
\begin{aligned}
|X_\beta[\mu](x)-X_\beta[g_t](x)|
&\le
\frac1\beta\frac{|H_\mu-H_{g_t}|}{D_\mu}
+
\frac1\beta |H_{g_t}|
\frac{|D_\mu-D_{g_t}|}{D_\mu D_{g_t}}       \\
&\le
C\delta
\frac{1+|H_{g_t}|/D_{g_t}}{D_{g_t}}        \\
&\le
C\delta(1+r_t(x))e^{\alpha r_t(x)^2/2}.
\end{aligned}
\]
It remains to show that the last weight has bounded \(\mu\)-integral on
\(E_\delta(t)\), uniformly in \(t\) and small \(\delta\).

Let \(\chi\in C^1(\mathbb R)\) satisfy
\[
0\le\chi\le1,\qquad
\chi(s)=1\ \text{for }s\le0,\qquad
\chi(s)=0\ \text{for }s\ge1,\qquad
|\chi'|\le2.
\]
Define
\[
h_{\delta,t}(x):=
(1+r_t(x))e^{\alpha r_t(x)^2/2}\chi(r_t(x)-A_\delta).
\]
Then \(h_{\delta,t}\) dominates
\[
(1+r_t)e^{\alpha r_t^2/2}\mathbf 1_{E_\delta(t)}
\]
and is supported in \(\{r_t\le A_\delta+1\}\).  Moreover,
\begin{equation}
\label{eq:h-lip-new}
\operatorname{Lip}(h_{\delta,t})
\le
C(1+A_\delta)^2e^{\alpha(A_\delta+1)^2/2}.
\end{equation}
Indeed, write
\[
F_A(r):=(1+r)e^{\alpha r^2/2}\chi(r-A).
\]
On the support of \(F_A\), one has \(0\le r\le A+1\), and
\[
\frac{d}{dr}\bigl((1+r)e^{\alpha r^2/2}\bigr)
=
e^{\alpha r^2/2}\bigl(1+\alpha r(1+r)\bigr).
\]
Since \(|\chi|\le1\) and \(|\chi'|\le2\),
\[
\sup_{r\ge0}|F_A'(r)|
\le
C(1+A)^2e^{\alpha(A+1)^2/2}.
\]
Composing with the uniformly Lipschitz functions \(r_t\) gives
\eqref{eq:h-lip-new}.

By Kantorovich--Rubinstein duality,
\[
\int h_{\delta,t}\,d\mu
\le
\int h_{\delta,t}\,dg_t
+
\operatorname{Lip}(h_{\delta,t})\,\delta .
\]
Under \(g_t\), \(r_t\) has the same law as \(|Z|\), \(Z\sim\cN(0,I_d)\).  Since
\(\alpha<1\),
\[
\sup_{t\in[0,T]}
\int (1+r_t)e^{\alpha r_t^2/2}\,dg_t<\infty.
\]
Together with \eqref{eq:h-lip-new} and
Lemma~\ref{lem:elementary-log-cutoff-estimates}, this gives
\[
\int h_{\delta,t}\,d\mu\le C
\]
uniformly in \(t\) and \(0<\delta\le\delta_0\).  Hence
\begin{equation}
\label{eq:core-estimate-new}
\int_{E_\delta(t)}
|X_\beta[\mu]-X_\beta[g_t]|\,d\mu
\le C\delta .
\end{equation}

We now estimate the complement.  The pointwise estimate used in
Lemma~\ref{lem:barycentric-action-estimate} gives
\[
|X_\beta[\mu](x)|^2
\le
\int |y-x|^2\,\mu(dy)
\le
2M_2+2|x|^2.
\]
The affine formula for \(X_\beta[g_t]\) gives the same linear-growth bound for the
reference field.  Since \(1+|x|\) and \(1+r_t(x)\) are uniformly equivalent on
\([0,T]\times\mathbb R^d\), we have
\begin{equation}
\label{eq:linear-growth-both-fields-new}
|X_\beta[\mu](x)|+|X_\beta[g_t](x)|
\le C(1+r_t(x)).
\end{equation}
Using Lemma~\ref{lem:gaussian-radial-tail-transfer} with \(a=A_\delta\), and then
Lemma~\ref{lem:elementary-log-cutoff-estimates}, we obtain
\[
\begin{aligned}
\int_{\mathbb R^d\setminus E_\delta(t)}
|X_\beta[\mu]-X_\beta[g_t]|\,d\mu
&\le
C\int_{\{r_t>A_\delta\}}(1+r_t)\,d\mu  \\
&\le
C(1+A_\delta)\delta
+
C(1+A_\delta)^q e^{-(A_\delta-1)^2/2} \\
&\le
C\delta(1+A_\delta).
\end{aligned}
\]
Together with \eqref{eq:core-estimate-new}, this yields
\[
\int_{\mathbb R^d}
|X_\beta[\mu]-X_\beta[g_t]|\,d\mu
\le
C\delta(1+A_\delta).
\]
Since
\[
A_\delta
=
\sqrt{\kappa\log\frac{C_0}{\delta}}
\le
C\sqrt{\log\frac{C}{\delta}}
\]
after increasing \(C\), the desired logarithmic estimate follows for
\(0<\delta\le\delta_0\).  As explained at the beginning of the proof, increasing
\(C\) gives the full range of \(\delta\).
\end{proof}

\begin{lemma}[Stability against an affine reference field]
\label{lem:continuity-equation-lipschitz-reference-stability}
Let \(\rho_t,\eta_t\in C([0,T];\mathcal P_1(\R^d))\) solve
\[
\partial_t\rho_t+\nabla\cdot(\rho_t b_t)=0,
\qquad
\partial_t\eta_t+\nabla\cdot(\eta_t c_t)=0
\]
in the weak sense, with finite transport integrals:
\[
\int_0^T\int |b_t(x)|\,\rho_t(dx)\,dt<\infty,
\qquad
\int_0^T\int |c_t(x)|\,\eta_t(dx)\,dt<\infty.
\]
Assume that the reference field is affine,
\[
c_t(x)=A_t x+a_t,
\]
where \(A\in C([0,T];\R^{d\times d})\) and \(a\in C([0,T];\R^d)\).  Set
\[
L(t):=\|A_t\|,
\qquad
\Lambda(s,t):=\exp\left(\int_s^t L(r)\,dr\right).
\]
Assume also
\[
\int_0^T\int |b_t(x)-c_t(x)|\,\rho_t(dx)\,dt<\infty.
\]
Then, for every \(t\in[0,T]\),
\begin{equation}
\label{eq:affine-reference-stability}
W_1(\rho_t,\eta_t)
\le
\Lambda(0,t)W_1(\rho_{t=0},\eta_{t=0})
+
\int_0^t \Lambda(s,t)
\int |b_s(x)-c_s(x)|\,\rho_s(dx)\,ds.
\end{equation}
\end{lemma}

\begin{proof}
Let \(\Phi_{s,t}\) be the affine flow generated by \(c_t\).  Then
\[
|\Phi_{s,t}(x)-\Phi_{s,t}(y)|\le \Lambda(s,t)|x-y|.
\]
For a \(1\)-Lipschitz test function \(\psi\), set
\[
\phi_s(x):=\psi(\Phi_{s,t}(x)).
\]
Then \(\partial_s\phi_s+c_s\cdot\nabla\phi_s=0\) and
\(\operatorname{Lip}(\phi_s)\le\Lambda(s,t)\).  Since \(\psi\) is only Lipschitz
and \(\phi_s\) need not be compactly supported, the testing argument is justified by
the standard approximation of Lipschitz functions by smooth compactly supported
functions: first truncate in space, then mollify, and finally pass to the limit using
the finite first moments and the finite transport integrals.  Testing the two
continuity equations against these approximants, subtracting, and passing to the
limit gives, by the Kantorovich--Rubinstein formula,
\[
W_1(\rho_t,\eta_t)
\le
\Lambda(0,t)W_1(\rho_{t=0},\eta_{t=0})
+
\int_0^t\Lambda(s,t)\int |b_s-c_s|\,d\rho_s\,ds.
\]
This is the Dobrushin duality argument adapted to the present affine-reference setting; compare
\cite[Proposition~4]{Dobrushin1979Vlasov}.
\end{proof}

The weak--strong uniqueness argument has three ingredients. We compare an arbitrary
finite-action solution with the explicit Gaussian solution by transporting test functions
along the affine Gaussian flow. The logarithmic stability estimate controls the nonlinear
velocity error by an Osgood modulus of \(W_1(\rho_t,g_t)\). Bihari's lemma then forces
this distance to remain zero, since the two solutions have the same initial law.

\begin{theorem}[Weak--strong uniqueness around the Gaussian solution]
\label{thm:weak-strong-uniqueness-gaussian}
Fix \(\beta>0\) and \(T>0\).  Let \(g_t=f_t^\beta=\cN(m,\Gamma_t)\) be the Gaussian
solution from Lemma~\ref{lem:gaussian-invariance-covariance}.  Let
\[
\rho\in C([0,T];\mathcal P_1(\R^d))
\]
be a finite-action weak solution of
\[
\partial_t\rho_t+\nabla\cdot(\rho_tX_\beta[\rho_t])=0
\]
with \(\rho_{t=0}=g_{t=0}\).  Assume additionally that
\[
\sup_{0\le t\le T}\int |x|^2\,\rho_t(dx)<\infty.
\]
Then
\[
\rho_t=g_t
\qquad\text{for every }t\in[0,T].
\]
\end{theorem}

\begin{proof}
Set
\[
w(t):=W_1(\rho_t,g_t).
\]
Apply Lemma~\ref{lem:continuity-equation-lipschitz-reference-stability} with
\[
b_t=X_\beta[\rho_t],
\qquad
c_t=X_\beta[g_t].
\]
The hypotheses of Lemma~\ref{lem:continuity-equation-lipschitz-reference-stability}
are satisfied: \(\rho\) is finite-action, \(X_\beta[g_t]\) is affine in \(x\)
with uniformly bounded linear part on \([0,T]\), and the second moments of both
\(\rho_t\) and \(g_t\) are uniformly bounded on \([0,T]\).  Since \(\rho_0=g_0\), the
affine-reference stability estimate gives a constant \(C_{\mathrm{st}}\), depending only on
\((\beta,T,m,\Gamma_0,d)\), such that
\[
w(t)
\le
C_{\mathrm{st}}\int_0^t
\int_{\mathbb R^d}
|X_\beta[\rho_s](x)-X_\beta[g_s](x)|\,\rho_s(dx)\,ds .
\]

Let
\[
M_2:=\sup_{0\le s\le T}\int_{\mathbb R^d}|x|^2\,\rho_s(dx)<\infty.
\]
By Lemma~\ref{lem:gaussian-reference-osgood-stability}, there exists a constant
\(C_{\mathrm{G}}\), depending only on
\[
(\beta,T,m,\Gamma_0,M_2,d),
\]
such that, for every \(s\in[0,T]\),
\[
\int_{\mathbb R^d}
|X_\beta[\rho_s](x)-X_\beta[g_s](x)|\,\rho_s(dx)
\le
C_{\mathrm{G}}\, w(s)
\left(
1+\sqrt{\log\frac{C_{\mathrm{G}}}{w(s)}}
\right),
\]
with the right-hand side interpreted as \(0\) when \(w(s)=0\).

The quantity \(\sup_{0\le s\le T}w(s)\) is finite and is controlled by the same data.  Indeed,
\[
\sup_{0\le s\le T}w(s)
\le
\sup_{0\le s\le T}\int |x|\,\rho_s(dx)
+
\sup_{0\le s\le T}\int |x|\,g_s(dx)
\le
M_2^{1/2}+C(\beta,T,m,\Gamma_0,d).
\]
Thus we may enlarge \(C_{\mathrm{G}}\), if necessary, so that
\[
C_{\mathrm{G}}\ge e\sup_{0\le s\le T}w(s).
\]
This does not invalidate the previous estimate, because increasing \(C_{\mathrm{G}}\)
only increases the right-hand side and preserves the same parameter dependence.
Hence, whenever \(w(s)>0\),
\[
0<w(s)\le \frac{C_{\mathrm{G}}}{e}.
\]
Therefore
\[
w(t)
\le
C_{\mathrm{st}}C_{\mathrm{G}}
\int_0^t
w(s)
\left(
1+\sqrt{\log\frac{C_{\mathrm{G}}}{w(s)}}
\right)\,ds .
\]
Equivalently,
\[
w(t)
\le
A\int_0^t \omega_B(w(s))\,ds,
\qquad
A:=C_{\mathrm{st}}C_{\mathrm{G}},
\]
where
\[
\omega_B(r):=
\begin{cases}
 r\left(1+\sqrt{\log(B/r)}\right),&0<r\le B/e,\\
 2r,&r>B/e,\\
 0,&r=0,
\end{cases}
\qquad B:=C_{\mathrm{G}}.
\]
The modulus \(\omega_B\) satisfies the Osgood condition
\[
\int_0^{B/e}\frac{dr}{\omega_B(r)}=\infty.
\]
Bihari's generalized Bellman lemma \cite[Sections~3--4]{Bihari1956GronwallGeneralization} therefore gives \(w\equiv0\) on
\([0,T]\).  Hence
\[
\rho_t=g_t
\qquad\text{for every }t\in[0,T].
\]
\end{proof}

\begin{corollary}[Convergence of hard-truncated Gaussian flows]
\label{cor:hard-truncated-gaussian-flows-converge}
Fix \(\beta>0\) and \(T>0\).  Then
\[
\sup_{0\le t\le T}
W_1(f_t^{\beta,[R]},f_t^\beta)
\longrightarrow0
\qquad\text{as }R\to\infty.
\]
\end{corollary}

\begin{proof}
We prove convergence by contradiction.  Suppose that the asserted convergence fails.
Then there exist \(\varepsilon_0>0\) and a sequence \(R_k\to\infty\) such that
\[
\sup_{0\le t\le T}
W_1(f_t^{\beta,[R_k]},f_t^\beta)
\ge \varepsilon_0
\qquad\text{for every }k.
\]
By Theorem~\ref{thm:hard-trunc-compactness-identification}, after passing to a
subsequence, not relabeled, there exists
\[
\rho\in C([0,T];\mathcal P_1(\mathbb R^d))
\]
such that
\[
\sup_{0\le t\le T}
W_1(f_t^{\beta,[R_k]},\rho_t)\to0.
\]
Moreover, \(\rho_{t=0}=f_0\), \(\rho\) is a finite-action weak solution of
\[
\partial_t\rho_t+\nabla\cdot(\rho_tX_\beta[\rho_t])=0
\]
on \([0,T]\), and
\[
\sup_{0\le t\le T}\int |x|^2\,\rho_t(dx)<\infty.
\]
Since \(f_0=g_0\), where \(g_t=f_t^\beta\) is the explicit Gaussian solution,
Theorem~\ref{thm:weak-strong-uniqueness-gaussian} gives
\[
\rho_t=g_t=f_t^\beta
\qquad\text{for every }t\in[0,T].
\]
Consequently,
\[
\sup_{0\le t\le T}
W_1(f_t^{\beta,[R_k]},f_t^\beta)
\to0,
\]
contradicting the choice of \(R_k\).  Therefore
\[
\sup_{0\le t\le T}
W_1(f_t^{\beta,[R]},f_t^\beta)
\to0
\qquad\text{as }R\to\infty.
\]
\end{proof}

\begin{theorem}[Gaussian-prior noisy laws are noncompact admissible]
\label{thm:gaussian-noisy-laws-admissible}
Fix \(\tau>0\).  Let \(m\in\R^d\), let \(\Sigma_0\) be a symmetric nonnegative semidefinite matrix, and set
\[
P_0=\cN(m,\Sigma_0),
\qquad
f_0:=\gamma_\tau*P_0=\cN(m,\Sigma_0+\tau I).
\]
Then, for every fixed \(\beta>0\) and \(T>0\), the law \(f_0\) is
\((\beta,T)\)-admissible in the sense of Definition~\ref{def:nc-admissible-law}.  The
admissible noncompact flow is the explicit Gaussian flow
\[
f_t^\beta=\cN(m,\Gamma_t),
\qquad
\dot\Gamma_t=-2\Gamma_t(I+\beta\Gamma_t)^{-1}.
\]
\end{theorem}

\begin{proof}
Since
\[
f_0=\cN(m,\Sigma_0+\tau I)
\]
and \(\tau>0\), the initial law \(f_0\) has a smooth strictly positive Gaussian density
and finite moments of all orders.  In particular, its hard truncations \(f_0^{[R]}\) are
well-defined for every \(R>0\), and
\[
W_1(f_0^{[R]},f_0)\to0
\qquad\text{as }R\to\infty
\]
by Lemma~\ref{lem:nc-hard-truncation-W1}.

The explicit Gaussian curve
\[
f_t^\beta=\cN(m,\Gamma_t),
\qquad
\dot\Gamma_t=-2\Gamma_t(I+\beta\Gamma_t)^{-1},
\]
is a weak solution by Lemma~\ref{lem:gaussian-invariance-covariance}.  It has finite
action on every finite time interval.  Indeed,
\[
X_\beta[f_t^\beta](x)=-(I+\beta\Gamma_t)^{-1}(x-m),
\]
and \(\Gamma_t\) remains uniformly positive definite and uniformly bounded on
\([0,T]\).  Hence
\[
\int_0^T\int |X_\beta[f_t^\beta](x)|^2\,f_t^\beta(dx)\,dt<\infty.
\]
Thus the first condition in Definition~\ref{def:nc-admissible-law} holds with admissible
flow \((f_t^\beta)_{0\le t\le T}\).

The second admissibility condition is exactly
Corollary~\ref{cor:hard-truncated-gaussian-flows-converge}, namely
\[
\sup_{0\le t\le T}
W_1(f_t^{\beta,[R]},f_t^\beta)\to0
\qquad\text{as }R\to\infty.
\]
Therefore \(f_0\) is \((\beta,T)\)-admissible.
\end{proof}

\begin{corollary}[Gaussian priors are recoverable admissible]
\label{cor:gaussian-priors-are-recoverable-admissible}
Fix \(\tau>0\).  Let
\[
P_0=\cN(m,\Sigma_0),
\]
where \(m\in\R^d\) and \(\Sigma_0\) is symmetric nonnegative semidefinite.  Then
\[
P_0\in\mathcal A_\tau.
\]
In particular, the class \(\mathcal A_\tau\) is nonempty.
\end{corollary}

\begin{proof}
Set
\[
f_0:=\gamma_\tau*P_0=\cN(m,\Sigma_0+\tau I).
\]
By Theorem~\ref{thm:gaussian-noisy-laws-admissible}, for every \(\beta>0\) and
\(T>0\), the law \(f_0\) is \((\beta,T)\)-admissible.  In particular, it is
\((\beta,T_\beta)\)-admissible for
\[
T_\beta=\frac{\beta\tau}{2}.
\]
The recovery condition
\[
W_1(f_{T_\beta}^\beta,P_0)\to0
\qquad\text{as }\beta\to\infty
\]
is exactly Lemma~\ref{lem:gaussian-recovery-condition}.  Hence \(P_0\in\mathcal A_\tau\).
\end{proof}

\subsection{Gaussian-prior posterior-mean recovery}
\label{sec:gaussian-prior-posterior-mean-recovery}

\begin{theorem}[Hard-truncated Gaussian posterior-mean recovery]
\label{thm:gaussian-prior-hard-truncated-posterior-recovery}
Fix \(\tau>0\).  Let \(m\in\R^d\), let \(\Sigma_0\) be symmetric nonnegative semidefinite, and set
\[
P_0=\cN(m,\Sigma_0),
\qquad
f_0:=\gamma_\tau*P_0=\cN(m,\Sigma_0+\tau I).
\]
For \(R>0\), define
\[
f_0^{[R]}:=\frac{\mathbf 1_{B_R}f_0}{f_0(B_R)}.
\]
Let \(\mu_t^{N,\beta,[R]}\) be the empirical measure of the exact particle system
initialized i.i.d. from \(f_0^{[R]}\).  Set
\[
T_\beta:=\frac{\beta\tau}{2}.
\]
Then, for every \(M<\infty\),
\begin{equation}
\label{eq:final-gaussian-hard-truncated-corollary}
\lim_{\beta\to\infty}
\limsup_{R\to\infty}
\limsup_{N\to\infty}
\EE
\sup_{|y|\le M}
\left|
m_{\mu_{T_\beta}^{N,\beta,[R]}}(y)-m_{P_0}(y)
\right|
=0.
\end{equation}
Consequently, the same convergence holds in probability.
\end{theorem}

\begin{proof}
By Corollary~\ref{cor:gaussian-priors-are-recoverable-admissible}, the Gaussian prior
\(P_0\) belongs to \(\mathcal A_\tau\).  The result is therefore an immediate application
of Theorem~\ref{thm:recoverable-prior-hard-truncated-posterior-recovery}.
\end{proof}

\section{Stopping time for noise removal}
\label{sec:stoppingtime}
The goal of this appendix is to explain the slowed denoising observed for small $\beta$ in Fig. 2(b). We do this in the simplest solvable setting: an isolated Gaussian cluster evolving under the Stage-1 Gaussian-attention dynamics. This yields an explicit scalar ODE for the cluster variance, from which one can read off both the slowdown at finite $\beta$ and the first-order correction to the denoising time. 
This calculation is only a calibration model, but it explains the qualitative regimes observed numerically and motivates the finite-$\beta$ correction to the denoising time.

Throughout this section we use the effective denoising time $t:=\frac{\eta \ell}{\beta}$, where $\ell$ is layer depth and $\eta$ is the residual step size in Algorithm 1. This is the normalization for which the ideal $\beta=\infty$ limit has stopping time $\sigma^2/2$. Equivalently, if $s:=\eta\ell$ denotes raw layer time, then $t=s/\beta$. All ODEs below are written in the $t-$variable.

\begin{theorem}[Variance ODE and denoising time for isotropic Gaussian prior]\label{thm:gaussiandenoise}
Let $G_\beta$ be the centered Gaussian kernel on $\mathbb{R}^d$ with covariance $\beta^{-1}I_d$, and consider the kernelized evolution
$$\partial_t f_t = -\nabla\cdot \left( f_t \nabla \log(G_\beta*f_t) \right), \qquad t\ge 0$$ with initial condition $f_0(x)=\mathcal{N}(m,(\tau_a^2+\sigma^2)I_d)(x),$ where $m\in\mathbb{R}^d$, $\tau_a>0$, and $\sigma>0$. 

Then for all $t\ge 0$, there is a well-defined solution, with a solution being the Gaussian $f_t(x)=\mathcal{N}(m,v(t)I_d)(x)$, where the variance parameter $v(t)$ solves
$$ v'(t)=-\frac{2\beta v(t)}{\beta v(t)+1}, \qquad v(0)=\tau_a^2+\sigma^2. $$
For any target variance $v_*\in(0,\tau_a^2+\sigma^2]$, the hitting time $T_a(v_*):=\inf\{t\ge 0\mid v(t)=v_*\}$ is given exactly by
$$T_a(v_*) = \frac{\tau_a^2+\sigma^2-v_*}{2} + \frac{1} {2\beta} \log \left(\frac{\tau_a^2+\sigma^2}{v_*}\right).$$
In particular, the exact time to denoise back to the clean variance $\tau_a^2$ is
$$T_a^* = \frac{\sigma^2}{2} + \frac{1}{2\beta}\log\left(1+\frac{\sigma^2}{\tau_a^2}\right).$$
\end{theorem}

\begin{proof} 
Verification that $f_t=\gamma_{v(t)}$ solves this ODE is done in Lemma~\ref{lem:gaussian-invariance-covariance} (in a different time scale) and it yields the ODE
$$v'(t)=-\frac{2v(t)}{v(t)+\beta^{-1}},  \qquad v(0)=\tau_a^2+\sigma^2$$ for the variance.

We now solve for the hitting time. Rearranging,
$$dt = -\left(\frac{1}{2}+\frac{1}{2\beta v}\right)dv.$$
Integrating from the initial variance $v_0=\tau_a^2+\sigma^2$ down to a target variance $v_*$ gives
$$T_a(v_*) =\int_{v_*}^{v_0} \left(\frac{1}{2}+\frac{1}{2\beta v}\right) dv = \frac{\tau_a^2+\sigma^2-v_*}{2} + \frac{1}{2\beta}\log\left(\frac{\tau_a^2+\sigma^2}{v_*}\right).$$
Setting $v_*=\tau_a^2$ gives
$$ T_a^* = \frac{(\tau_a^2+\sigma^2)-\tau_a^2}{2} + \frac{1}{2\beta}\log\left(\frac{\tau_a^2+\sigma^2}{\tau_a^2}\right)  = \frac{\sigma^2}{2} + \frac{1}{2\beta}\log\left(1+\frac{\sigma^2}{\tau_a^2}\right).$$
This proves the claim.
\end{proof}

With the observation that the point mass measure $\delta_z$ at point $z$ is just $\mathcal N(z,0)$, we immediately get the following result if the clean distribution on $X$ is $\delta_z$.

\begin{corollary}[Denoising time for Dirac delta]
Consider the kernelized evolution $$\partial_t f_t+\nabla\cdot\left(f_t\nabla\log(G_\beta *f_t)\right), \qquad f_0 = \mathcal N(z,\sigma^2I_d), \qquad t\ge 0.$$
Assume the same regularity assumptions hold as Theorem~\ref{thm:gaussiandenoise}. Then for all $t\ge 0$ the system is uniquely solved by $f_t(x) = \mathcal N(z,v(t)I_d)$ where $v(t)$ is the solution of $$v'(t) = -\frac{2\beta v(t)}{\beta v(t)+1}, \qquad v(0) = \sigma^2.$$
The time to reach a target variance $v_*$ is $$T(v_*) = \frac{\sigma^2-v_*}{2} + \frac{1}{2\beta}\log\left(\frac{\sigma^2}{v_*}\right).$$
In particular, the time to denoise completely is $T(0) = +\infty$ and denoising to variance $\beta^{-1}$ is $$T(\beta^{-1}) = \frac{\sigma^2}{2} - \frac{1}{2\beta} + \frac{\log \sigma}{\beta} + \frac{\log \beta}{2\beta}.$$
\end{corollary}
\begin{proof}
This is same as Theorem~\ref{thm:gaussiandenoise} where the initial distribution is $\delta_z = \mathcal N(z,0)$. So take $\tau_a\downarrow 0, m=z$ in Theorem~\ref{thm:gaussiandenoise}.
\end{proof}

Now we move on to the stopping times for more general distributions. We discuss here that under some assumptions, the stopping time is $\sigma^2/2$ with a correction term of order $\beta^{-1}$. 

\begin{lemma}\label{lem:purturb}
Let
$$ \partial_t f_t^{(\beta)} = -\nabla\cdot \left(f_t^{(\beta)} \nabla\log(G_\beta*f_t^{(\beta)})\right), \qquad f_0^{(\beta)}=\mu*\gamma_{\sigma^2}, $$
on $\mathbb R^d$, where $\sigma>0$ and $\beta\ge 1$. Define
$$ T_0:=\frac{\sigma^2}{2}, \qquad g_t:=\mu*\gamma_{\sigma^2-2t}, \qquad 0\le t\le T_0. $$
Then $g_t$ solves the backward heat equation
$$
\partial_t g_t=-\Delta g_t, \qquad g_0=\mu*\gamma_{\sigma^2}, \qquad g_{T_0}=\mu.
$$

Fix $s>d/2+4$. Assume:
\begin{enumerate}
\item $\mu$ is strictly positive and belongs to $H^{s+6}(\mathbb R^d)$;
\item the map $N_\beta(f) = -\nabla\cdot\left(f\nabla\log(G_\beta * f)\right)$ admits the expansion
$$\mathcal N_\beta(f) = -\Delta f-\frac{1}{2\beta}Q[f]+\frac{1}{\beta^2}\mathcal E_\beta[f], $$ with $$ Q[f]:=\nabla\cdot\!\left(f\,\nabla\!\left(\frac{\Delta f}{f}\right)\right), $$ and with a uniform bound $$\sup_{0\le t\le T_0}\|\mathcal E_\beta[g_t]\|_{H^s}\le C_0$$ for all sufficiently large $\beta$;
\item the linear backward heat propagator on the interval $[0,T_0]$ is bounded on $H^s$ along this orbit, in the sense that for every $r\in L^1([0,T_0];H^s)$, the solution $u$ of $$\partial_t u=-\Delta u+r,\qquad u_0=0,$$ satisfies $$\sup_{0\le t\le T_0}\|u_t\|_{H^s} \le C_1\int_0^{T_0}\|r_s\|_{H^s} ds.$$
\end{enumerate}

Then there exists a function $h_t\in C([0,T_0];H^s)$ solving
$$\partial_t h_t=-\Delta h_t-\frac12\,Q[g_t], \qquad h_0=0,$$
such that
$$f_t^{(\beta)} = g_t+\frac{1}{\beta}h_t+r_t^{(\beta)}, \qquad 0\le t\le T_0,$$
with remainder satisfying
$$ \sup_{0\le t\le T_0}\|r_t^{(\beta)}\|_{H^s} \le \frac{C}{\beta^2}$$
for some constant $C$ independent of $\beta$.

In particular, at the denoising time $T_0=\sigma^2/2$,
$$ f_{T_0}^{(\beta)} = \mu+\frac{1}{\beta}h_{T_0}+O_{H^s}(\beta^{-2}),$$
and hence
$$\bigl\|f_{T_0}^{(\beta)}-\mu\bigr\|_{H^s} \le \frac{\|h_{T_0}\|_{H^s}}{\beta} + \frac{C}{\beta^2}.$$
\end{lemma}

\begin{proof}
We have $\mathcal N_\beta(f):=-\nabla\cdot\left(f\nabla\log(G_\beta*f)\right).$ By construction, $g_t=\mu*\gamma_{\sigma^2-2t}$ solves the PDE for $\beta=+\infty$ with the agreement that $G_\infty*f = f$. So $f_t^{(\infty)} = g_t$.

We seek an expansion of the form
$$
f_t^{(\beta)}=g_t+\beta^{-1}h_t+r_t^{(\beta)}.
$$
Substituting into the equation and using the expansion of $\mathcal N_\beta$ gives
$$\partial_t\left(g_t+\beta^{-1}h_t+r_t^{(\beta)}\right) = -\Delta \left(g_t+\beta^{-1}h_t+r_t^{(\beta)}\right) -\frac{1}{2\beta}Q[g_t] +\frac{1}{\beta^2}\mathcal E_\beta[g_t] +\mathcal R_t^{(\beta)},$$
where $\mathcal R_t^\beta$ contains the nonlinear Taylor remainder coming from replacing $g_t$ by $g_t+\beta^{-1}h_t+r_t^\beta$. Since $g_t$ solves $\partial_t g_t=-\Delta g_t$, the $O(1)$ terms cancel. Matching the $O(\beta^{-1})$ terms leads to
$$ \partial_t h_t=-\Delta h_t-\frac12 Q[g_t], \qquad h_0=0.
$$
This determines $h_t$ uniquely.

Now define the remainder by
$$ r_t^\beta:=f_t^\beta-g_t-\beta^{-1}h_t. $$
Subtracting the equations for $f_t^\beta$, $g_t$, and $h_t$ yields
$$ \partial_t r_t^{(\beta)} = -\Delta r_t^\beta+\frac{1}{\beta^2}\mathcal E_\beta[g_t]+\mathcal R_t^\beta, \qquad r_0^\beta=0. $$
By assumption, $\|\mathcal E_\beta[g_t]\|_{H^s}\le C_0$ uniformly on $[0,T_0]$. Moreover, because $s>d/2+4$, Sobolev multiplication and composition estimates imply that the nonlinear Taylor remainder satisfies
$$\|\mathcal R_t^\beta\|_{H^s} \le C_2\left(\beta^{-2}+\|r_t^\beta\|_{H^s}^2+\beta^{-1}\|r_t^\beta\|_{H^s}\right)$$
for all $t\in[0,T_0]$, provided $\beta$ is large enough.

Applying the assumed backward heat estimate to the remainder equation gives
$$ \sup_{0\le t\le T_0}\|r_t^\beta\|_{H^s} \le C_1\int_0^{T_0} \left( \frac{C_0}{\beta^2} + C_2\left(\beta^{-2}+\|r_s^\beta\|_{H^s}^2+\beta^{-1}\|r_s^\beta\|_{H^s}\right)
\right)ds. $$
A standard bootstrap argument now shows that
$$ \sup_{0\le t\le T_0}\|r_t^\beta\|_{H^s}\le \frac{C}{\beta^2} $$
for all sufficiently large $\beta$.

Evaluating at $t=T_0$ and using $g_{T_0}=\mu$ gives
$$ f_{T_0}^{(\beta)} = \mu+\frac{1}{\beta}h_{T_0}+r_{T_0}^{(\beta)}, $$
hence
$$ \|f_{T_0}^\beta-\mu\|_{H^s} \le \frac{\|h_{T_0}\|_{H^s}}{\beta} + \frac{C}{\beta^2}. $$
\end{proof}

The above perturbative expansion of $f_{T_0}^{(\beta)}$ immediately gives a first order correction for the stopping time. The next lemma should be read as a conditional perturbative criterion near the backward-heat horizon $T_0=\sigma^2/2$; in particular, it assumes the existence of a sufficiently regular backward-heat continuation in the chosen function space instead of an $H^s$-like criterion as in the above lemma.

\begin{lemma}[Conditional perturbative criterion]\label{lem:time}
Let $X$ be a Banach space, and let $ T_0:=\frac{\sigma^2}{2}$. Assume that for some $\delta>0$ and all sufficiently large $\beta$, the solution $f_t^\beta$ admits an expansion
$$ f_t^{(\beta)}=g_t+\beta^{-1}h_t+r_t^{(\beta)}, \qquad T_0-\delta\le t\le T_0+\delta, $$
where $g_t=\mu*\gamma_{\sigma^2-2t}, g_{T_0}=\mu,$ and
$$ \sup_{|t-T_0|\le \delta}\|h_t\|_X\le C_h,\qquad \sup_{|t-T_0|\le \delta}\|r_t^{(\beta)}\|_X\le C_r\beta^{-2}.
$$
For $t>T_0$, the notation $g_t$ is understood as a chosen $X-$valued backward-heat continuation of the curve $t\mapsto \mu * \gamma_{\sigma^2-2t}$, not as a literal Gaussian convolution. Assume also that $t\mapsto g_t$ is $\mathcal C^1$ as an $X$-valued map on $[T_0-\delta,T_0+\delta]$, and that
$$
\partial_t g_t\big|_{t=T_0}=-\Delta\mu\neq 0
\quad\text{in }X.
$$
 Define
$$T_{(\beta)}^*= \inf \{\arg\min_{|t-T_0|\le \delta}\|f_t^{(\beta)}-\mu\|_X\}.$$
Then $\exists C>0$, independent of $\beta$, such that $|T_{(\beta)}^*-T_0|\le \frac{C}{\beta}$ for sufficiently large $\beta$. In particular, $T_{(\beta)}^*=\frac{\sigma^2}{2}+O(\beta^{-1}).$
\end{lemma}

\begin{proof}
Since $g_{T_0}=\mu$, the expansion at $t=T_0$ gives
$$
f_{T_0}^\beta-\mu = \beta^{-1}h_{T_0}+r_{T_0}^\beta.
$$
Hence
$$
\|f_{T_0}^\beta-\mu\|_X
\le
\beta^{-1}\|h_{T_0}\|_X+\|r_{T_0}^\beta\|_X
\le
\frac{C_h}{\beta}+\frac{C_r}{\beta^2}.
$$
Therefore there exists $C_0>0$ such that
$$
\|f_{T_0}^\beta-\mu\|_X\le \frac{C_0}{\beta}
$$
for all sufficiently large $\beta$.

Since $g_t$ is $\mathcal C^1$ in $X$ and
$$g_{T_0}=\mu, \qquad \partial_t g_t\big|_{t=T_0}=-\Delta\mu\neq 0, $$
the Fr\'echet differentiability of $t\mapsto g_t$ at $T_0$ implies
$$g_t-\mu = (t-T_0)\,\partial_t g_t\big|_{t=T_0} + \omega(t), $$
where $\displaystyle \lim_{t\to T_0}\frac{\|\omega(t)\|_X}{|t-T_0|} = 0.$
Hence there exists $0<\delta_1\le \delta$ such that for all $|t-T_0|\le \delta_1$,
$$\|\omega(t)\|_X\le \frac12 \|\Delta\mu\|_X\,|t-T_0|.$$
Therefore, using $\partial_t g_t|_{t=T_0}=-\Delta\mu$,
$$ \|g_t-\mu\|_X \ge |t-T_0|\|\Delta\mu\|_X-\|\omega(t)\|_X \ge \frac12\|\Delta\mu\|_X\,|t-T_0|.$$
Set $$ c:=\frac12\|\Delta\mu\|_X>0.$$
Then $$\|g_t-\mu\|_X\ge c |t-T_0| \qquad\text{for all }|t-T_0|\le \delta_1.$$

For $|t-T_0|\le \delta_1$, the expansion gives
$$ f_t^\beta-g_t = \beta^{-1}h_t+r_t^\beta\implies \|f_t^\beta-g_t\|_X \le \frac{C_h}{\beta}+\frac{C_r}{\beta^2} \le \frac{C_1}{\beta} $$
for some constant $C_1>0$ and all sufficiently large $\beta$.

Hence, for $|t-T_0|\le \delta_1$,
$$\|f_t^\beta-\mu\|_X \ge \|g_t-\mu\|_X-\|f_t^\beta-g_t\|_X \ge c |t-T_0|-\frac{C_1}{\beta}.$$
Now let $$M:=\frac{C_0+C_1+1}{c}. $$
If $|t-T_0|\ge M/\beta$ and also $|t-T_0|\le \delta_1$, then $$ \|f_t^\beta-\mu\|_X \ge c\frac{M}{\beta}-\frac{C_1}{\beta} = \frac{C_0+1}{\beta} > \frac{C_0}{\beta} \ge \|f_{T_0}^\beta-\mu\|_X. $$
Therefore no minimizer of $t\mapsto \|f_t^\beta-\mu\|_X$ over $|t-T_0|\le \delta_1$ can lie outside the interval
$$|t-T_0|<\frac{M}{\beta}.$$

We have $M/\beta<\delta_1$ for sufficiently large $\beta$, so  $|T_{(\beta)}^*-T_0|\le \frac{M}{\beta}.$
\end{proof}

%%%%%%%%%%%%%%%%%%%%%%%%%%%%%%%%%%%%%%%%%%%%%%%%%%%%%%%%%%%%

%\newpage
%\input{checklist.tex}


\begin{thebibliography}{63}
\providecommand{\natexlab}[1]{#1}
\providecommand{\url}[1]{\texttt{#1}}
\expandafter\ifx\csname urlstyle\endcsname\relax
  \providecommand{\doi}[1]{doi: #1}\else
  \providecommand{\doi}{doi: \begingroup \urlstyle{rm}\Url}\fi

\bibitem[Ambrosio et~al.(2008)Ambrosio, Gigli, and Savar\'e]{AmbrosioGigliSavare2008GradientFlows}
Luigi Ambrosio, Nicola Gigli, and Giuseppe Savar\'e.
\newblock \emph{Gradient flows in metric spaces and in the space of probability measures}.
\newblock Lectures in Mathematics ETH Z\"urich. Birkh\"auser Verlag, Basel, second edition, 2008.
\newblock ISBN 978-3-7643-8721-1.

\bibitem[Belkin and Niyogi(2003)]{belkin2003laplacian}
Mikhail Belkin and Partha Niyogi.
\newblock Laplacian eigenmaps for dimensionality reduction and data representation.
\newblock \emph{Neural computation}, 15\penalty0 (6):\penalty0 1373--1396, 2003.

\bibitem[Bihari(1956)]{Bihari1956GronwallGeneralization}
I.~Bihari.
\newblock A generalization of a lemma of {B}ellman and its application to uniqueness problems of differential equations.
\newblock \emph{Acta Math. Acad. Sci. Hungar.}, 7:\penalty0 81--94, 1956.
\newblock ISSN 0001-5954,1588-2632.
\newblock \doi{10.1007/BF02022967}.
\newblock URL \url{https://doi.org/10.1007/BF02022967}.

\bibitem[Bishop(2006)]{bishop2006pattern}
Christopher~M. Bishop.
\newblock \emph{Pattern Recognition and Machine Learning}.
\newblock Springer, 2006.

\bibitem[Bruno et~al.(2025)Bruno, Pasqualotto, and Agazzi]{bruno2025_neurips_meanfield}
Giuseppe Bruno, Federico Pasqualotto, and Andrea Agazzi.
\newblock A multiscale analysis of mean-field transformers in the moderate interaction regime.
\newblock In \emph{The Thirty-ninth Annual Conference on Neural Information Processing Systems}, 2025.
\newblock URL \url{https://openreview.net/forum?id=WCRPgBpbcA}.

\bibitem[Burger et~al.(2025)Burger, Kabri, Korolev, Roith, and Weigand]{burger2025roysoc}
Martin Burger, Samira Kabri, Yury Korolev, Tim Roith, and Lukas Weigand.
\newblock Analysis of mean-field models arising from self-attention dynamics in transformer architectures with layer normalization.
\newblock \emph{Philosophical Transactions of the Royal Society A: Mathematical, Physical and Engineering Sciences}, 383\penalty0 (2298):\penalty0 20240233, 06 2025.
\newblock ISSN 1364-503X.
\newblock \doi{10.1098/rsta.2024.0233}.
\newblock URL \url{https://doi.org/10.1098/rsta.2024.0233}.

\bibitem[Castin et~al.(2026)Castin, Ablin, Carrillo, and Peyré]{castin2026unified_v2-2026}
Valérie Castin, Pierre Ablin, José~Antonio Carrillo, and Gabriel Peyré.
\newblock A unified perspective on the dynamics of deep transformers, 2026.
\newblock URL \url{https://arxiv.org/abs/2501.18322}.

\bibitem[Chaintron and Diez(2022)]{ChaintronDiez2022PropagationChaosI}
Louis-Pierre Chaintron and Antoine Diez.
\newblock Propagation of chaos: a review of models, methods and applications. {I}. {M}odels and methods.
\newblock \emph{Kinet. Relat. Models}, 15\penalty0 (6):\penalty0 895--1015, 2022.
\newblock ISSN 1937-5093,1937-5077.
\newblock \doi{10.3934/krm.2022017}.
\newblock URL \url{https://doi.org/10.3934/krm.2022017}.

\bibitem[Chen et~al.(2018)Chen, Rubanova, Bettencourt, and Duvenaud]{chen2018neural}
Ricky~TQ Chen, Yulia Rubanova, Jesse Bettencourt, and David~K Duvenaud.
\newblock Neural ordinary differential equations.
\newblock \emph{Advances in neural information processing systems}, 31, 2018.

\bibitem[Chen et~al.(2021)Chen, Zeng, Ji, and Yang]{chen2021skyformer}
Yifan Chen, Qi~Zeng, Heng Ji, and Yun Yang.
\newblock Skyformer: Remodel self-attention with gaussian kernel and nyström method.
\newblock \emph{Advances in Neural Information Processing Systems}, 34:\penalty0 2122--2135, 2021.

\bibitem[Comaniciu and Meer(1999)]{CamaniciuMeer1999meanshift}
D.~Comaniciu and P.~Meer.
\newblock Mean shift analysis and applications.
\newblock In \emph{Proceedings of the Seventh IEEE International Conference on Computer Vision}, volume~2, pages 1197--1203 vol.2, 1999.
\newblock \doi{10.1109/ICCV.1999.790416}.

\bibitem[Dehmamy et~al.(2026)Dehmamy, Hoover, Saha, Kozachkov, Slotine, and Krotov]{dehmamy2026nrgpt}
Nima Dehmamy, Benjamin Hoover, Bishwajit Saha, Leo Kozachkov, Jean-Jacques Slotine, and Dmitry Krotov.
\newblock {NRGPT}: An energy-based alternative for {GPT}.
\newblock In \emph{The Fourteenth International Conference on Learning Representations}, 2026.
\newblock URL \url{https://openreview.net/forum?id=B3Muyi2zgo}.

\bibitem[Del~Moral(2013)]{del2013mean}
Pierre Del~Moral.
\newblock Mean field simulation for monte carlo integration.
\newblock \emph{Monographs on Statistics and Applied Probability}, 126\penalty0 (26):\penalty0 6, 2013.

\bibitem[Dobrushin(1979)]{Dobrushin1979Vlasov}
Roland~L’vovich Dobrushin.
\newblock Vlasov equations.
\newblock \emph{Functional Analysis and Its Applications}, 13\penalty0 (2):\penalty0 115--123, 1979.

\bibitem[Efron(2011)]{efron2011tweedie}
Bradley Efron.
\newblock Tweedie’s formula and selection bias.
\newblock \emph{Journal of the American Statistical Association}, 106\penalty0 (496):\penalty0 1602--1614, 2011.

\bibitem[Efron(2014)]{efron2014two}
Bradley Efron.
\newblock Two modeling strategies for empirical bayes estimation.
\newblock \emph{Statistical science: a review journal of the Institute of Mathematical Statistics}, 29\penalty0 (2):\penalty0 285, 2014.

\bibitem[Fournier and Guillin(2015)]{FournierGuillin2015EmpiricalWasserstein}
Nicolas Fournier and Arnaud Guillin.
\newblock On the rate of convergence in {W}asserstein distance of the empirical measure.
\newblock \emph{Probab. Theory Related Fields}, 162\penalty0 (3-4):\penalty0 707--738, 2015.
\newblock ISSN 0178-8051,1432-2064.
\newblock \doi{10.1007/s00440-014-0583-7}.
\newblock URL \url{https://doi.org/10.1007/s00440-014-0583-7}.

\bibitem[Fukunaga and Hostetler(1975)]{FukunagaHostetler1975}
K.~Fukunaga and L.~Hostetler.
\newblock The estimation of the gradient of a density function, with applications in pattern recognition.
\newblock \emph{IEEE Transactions on Information Theory}, 21\penalty0 (1):\penalty0 32--40, 1975.
\newblock \doi{10.1109/TIT.1975.1055330}.

\bibitem[Geshkovski et~al.(2023)Geshkovski, Letrouit, Polyanskiy, and Rigollet]{geshkovski2023emergence}
Borjan Geshkovski, Cyril Letrouit, Yury Polyanskiy, and Philippe Rigollet.
\newblock The emergence of clusters in self-attention dynamics.
\newblock \emph{Advances in Neural Information Processing Systems}, 36:\penalty0 57026--57037, 2023.

\bibitem[Geshkovski et~al.(2025)Geshkovski, Letrouit, Polyanskiy, and Rigollet]{geshkovski2025mathematical}
Borjan Geshkovski, Cyril Letrouit, Yury Polyanskiy, and Philippe Rigollet.
\newblock A mathematical perspective on transformers.
\newblock \emph{Bulletin of the American Mathematical Society}, 62\penalty0 (3):\penalty0 427--479, 2025.

\bibitem[Gladstone et~al.(2026)Gladstone, Nanduru, Islam, Han, Ha, Chadha, Du, Ji, Li, and Iqbal]{gladstone2026energybased}
Alexi Gladstone, Ganesh Nanduru, Md~Mofijul Islam, Peixuan Han, Hyeonjeong Ha, Aman Chadha, Yilun Du, Heng Ji, Jundong Li, and Tariq Iqbal.
\newblock Energy-based transformers are scalable learners and thinkers.
\newblock In \emph{The Fourteenth International Conference on Learning Representations}, 2026.
\newblock URL \url{https://openreview.net/forum?id=ZBj3Qp1bYg}.

\bibitem[Greengard and Strain(1991)]{greengard1991fast}
Leslie Greengard and John Strain.
\newblock The fast gauss transform.
\newblock \emph{SIAM Journal on Scientific and Statistical Computing}, 12\penalty0 (1):\penalty0 79--94, 1991.

\bibitem[Hein and Maier(2006)]{hein2006manifold}
Matthias Hein and Markus Maier.
\newblock Manifold denoising.
\newblock \emph{Advances in neural information processing systems}, 19, 2006.

\bibitem[Ho et~al.(2020)Ho, Jain, and Abbeel]{ho2020denoising}
Jonathan Ho, Ajay Jain, and Pieter Abbeel.
\newblock Denoising diffusion probabilistic models.
\newblock \emph{Advances in neural information processing systems}, 33:\penalty0 6840--6851, 2020.

\bibitem[Hoover et~al.(2023)Hoover, Liang, Pham, Panda, Strobelt, Chau, Zaki, and Krotov]{hoover2023energytransformer}
Benjamin Hoover, Yuchen Liang, Bao Pham, Rameswar Panda, Hendrik Strobelt, Duen~Horng Chau, Mohammed~J Zaki, and Dmitry Krotov.
\newblock Energy transformer.
\newblock In \emph{Thirty-seventh Conference on Neural Information Processing Systems}, 2023.
\newblock URL \url{https://openreview.net/forum?id=MbwVNEx9KS}.

\bibitem[Ilin and Sushko(2026)]{ilin2026discoformerplugindensityscore}
Vasily Ilin and Peter Sushko.
\newblock Discoformer: Plug-in density and score estimation with transformers, 2026.
\newblock URL \url{https://arxiv.org/abs/2511.05924}.

\bibitem[Jaffe et~al.(2025)Jaffe, Ignatiadis, and Sen]{jaffe2025constrained}
Adam~Quinn Jaffe, Nikolaos Ignatiadis, and Bodhisattva Sen.
\newblock Constrained denoising, empirical bayes, and optimal transport.
\newblock \emph{arXiv preprint arXiv:2506.09986}, 2025.

\bibitem[Johnstone and Silverman(2005)]{JohnstoneSilverman2005}
Iain~M. Johnstone and Bernard~W. Silverman.
\newblock Empirical bayes selection of wavelet thresholds.
\newblock \emph{Annals of Statistics}, 33, 2005.
\newblock ISSN 00905364.
\newblock \doi{10.1214/009053605000000345}.

\bibitem[Kadkhodaie et~al.(2026)Kadkhodaie, Pooladian, Chewi, and Simoncelli]{kadkhodaie2026_blinddenoising}
Zahra Kadkhodaie, Aram-Alexandre Pooladian, Sinho Chewi, and Eero Simoncelli.
\newblock Blind denoising diffusion models and the blessings of dimensionality, 2026.
\newblock URL \url{https://arxiv.org/abs/2602.09639}.

\bibitem[Krotov and Hopfield(2016)]{krotov2016dense}
Dmitry Krotov and John~J Hopfield.
\newblock Dense associative memory for pattern recognition.
\newblock \emph{Advances in neural information processing systems}, 29, 2016.

\bibitem[Li and He(2025)]{li2026basicsletdenoisinggenerative}
Tianhong Li and Kaiming He.
\newblock Back to basics: Let denoising generative models denoise.
\newblock \emph{arXiv preprint arXiv:2511.13720}, 2025.

\bibitem[Ma et~al.(2024)Ma, Goldstein, Albergo, Boffi, Vanden-Eijnden, and Xie]{Ma2024_SiT_Albergo_Boffi_EVE}
Nanye Ma, Mark Goldstein, Michael~S. Albergo, Nicholas~M. Boffi, Eric Vanden-Eijnden, and Saining Xie.
\newblock Sit: Exploring flow and diffusion-based generative models with scalable interpolant transformers.
\newblock In \emph{Computer Vision -- ECCV 2024: 18th European Conference, Milan, Italy, September 29--October 4, 2024, Proceedings, Part LXXVII}, pages 23--40, Berlin, Heidelberg, 2024. Springer-Verlag.
\newblock ISBN 978-3-031-72979-9.
\newblock \doi{10.1007/978-3-031-72980-5_2}.
\newblock URL \url{https://doi.org/10.1007/978-3-031-72980-5_2}.

\bibitem[Ma et~al.(2026)Ma, Wang, and Samworth]{ma2026optimalincontextadaptivitydistributional}
Tianyi Ma, Tengyao Wang, and Richard~J. Samworth.
\newblock Optimal in-context adaptivity and distributional robustness of transformers, 2026.
\newblock URL \url{https://arxiv.org/abs/2510.23254}.

\bibitem[Miyasawa(1961)]{miyasawa1961empirical}
Koichi Miyasawa.
\newblock An empirical bayes estimator of the mean of a normal population.
\newblock \emph{Bull. Inst. Internat. Statist}, 38\penalty0 (181-188):\penalty0 1--2, 1961.

\bibitem[Modi et~al.(2025)Modi, Han, Vanden-Eijnden, and Bruna]{modi2025generative}
Chirag Modi, Jiequn Han, Eric Vanden-Eijnden, and Joan Bruna.
\newblock Generative modeling from black-box corruptions via self-consistent stochastic interpolants.
\newblock \emph{arXiv preprint arXiv:2512.10857}, 2025.

\bibitem[Nadaraya(1964)]{nadaraya1964estimating}
EA~Nadaraya.
\newblock On estimating regression. theor.
\newblock \emph{Probab. Appl}, 9\penalty0 (1), 1964.

\bibitem[Peebles and Xie(2023)]{Peebles_2023_ICCV_DiT}
William Peebles and Saining Xie.
\newblock Scalable diffusion models with transformers.
\newblock In \emph{Proceedings of the IEEE/CVF International Conference on Computer Vision (ICCV)}, pages 4195--4205, October 2023.

\bibitem[Ramsauer et~al.(2021)Ramsauer, Sch{\"{a}}fl, Lehner, Seidl, Widrich, Gruber, Holzleitner, Adler, Kreil, Kopp, Klambauer, Brandstetter, and Hochreiter]{ramsauer2021iclr}
Hubert Ramsauer, Bernhard Sch{\"{a}}fl, Johannes Lehner, Philipp Seidl, Michael Widrich, Lukas Gruber, Markus Holzleitner, Thomas Adler, David~P. Kreil, Michael~K. Kopp, G{\"{u}}nter Klambauer, Johannes Brandstetter, and Sepp Hochreiter.
\newblock Hopfield networks is all you need.
\newblock In \emph{9th International Conference on Learning Representations, {ICLR} 2021, Virtual Event, Austria, May 3-7, 2021}. OpenReview.net, 2021.
\newblock URL \url{https://openreview.net/forum?id=tL89RnzIiCd}.

\bibitem[Raphan and Simoncelli(2006)]{raphan2006_simoncelli}
Martin Raphan and Eero Simoncelli.
\newblock Learning to be bayesian without supervision.
\newblock \emph{Advances in neural information processing systems}, 19, 2006.

\bibitem[Raphan and Simoncelli(2011)]{raphan2011_simoncelli}
Martin Raphan and Eero~P Simoncelli.
\newblock Least squares estimation without priors or supervision.
\newblock \emph{Neural computation}, 23\penalty0 (2):\penalty0 374--420, 2011.

\bibitem[Rezende and Mohamed(2015)]{rezende2015variational}
Danilo Rezende and Shakir Mohamed.
\newblock Variational inference with normalizing flows.
\newblock In \emph{International conference on machine learning}, pages 1530--1538. PMLR, 2015.

\bibitem[Rigollet(2025)]{rigollet2025mean}
Philippe Rigollet.
\newblock The mean-field dynamics of transformers.
\newblock \emph{arXiv preprint arXiv:2512.01868}, 2025.

\bibitem[Robbins(1956)]{robbins1956empirical}
Herbert~E Robbins.
\newblock An empirical bayes approach to statistics.
\newblock In \emph{Breakthroughs in Statistics: Foundations and basic theory}, pages 388--394. Springer, 1956.

\bibitem[Rombach et~al.(2022)Rombach, Blattmann, Lorenz, Esser, and Ommer]{Rombach_2022_CVPR}
Robin Rombach, Andreas Blattmann, Dominik Lorenz, Patrick Esser, and Bj\"orn Ommer.
\newblock High-resolution image synthesis with latent diffusion models.
\newblock In \emph{Proceedings of the IEEE/CVF Conference on Computer Vision and Pattern Recognition (CVPR)}, pages 10684--10695, June 2022.

\bibitem[Rosu et~al.(2025)Rosu, Carin, and Cheng]{rosu2025_denoising}
Paul Rosu, Lawrence Carin, and Xiang Cheng.
\newblock From softmax to score: Transformers can effectively implement in-context denoising steps.
\newblock In \emph{The Thirty-ninth Annual Conference on Neural Information Processing Systems}, 2025.
\newblock URL \url{https://openreview.net/forum?id=4QRoLzD11x}.

\bibitem[Sander et~al.(2022)Sander, Ablin, Blondel, and Peyr{\'e}]{sander2022sinkformers}
Michael~E Sander, Pierre Ablin, Mathieu Blondel, and Gabriel Peyr{\'e}.
\newblock Sinkformers: Transformers with doubly stochastic attention.
\newblock In \emph{International Conference on Artificial Intelligence and Statistics}, pages 3515--3530. PMLR, 2022.

\bibitem[Saunshi et~al.(2025)Saunshi, Dikkala, Li, Kumar, and Reddi]{saunshi2025reasoning}
Nikunj Saunshi, Nishanth Dikkala, Zhiyuan Li, Sanjiv Kumar, and Sashank~J. Reddi.
\newblock Reasoning with latent thoughts: On the power of looped transformers.
\newblock In \emph{The Thirteenth International Conference on Learning Representations}, 2025.
\newblock URL \url{https://openreview.net/forum?id=din0lGfZFd}.

\bibitem[Smart and Zilman(2023)]{smart2023emergent}
Matthew Smart and Anton Zilman.
\newblock Emergent properties of collective gene-expression patterns in multicellular systems.
\newblock \emph{Cell Reports Physical Science}, 4\penalty0 (2), 2023.

\bibitem[Smart et~al.(2025)Smart, Bietti, and Sengupta]{smart2025context}
Matthew Smart, Alberto Bietti, and Anirvan~M Sengupta.
\newblock In-context denoising with one-layer transformers: Connections between attention and associative memory retrieval.
\newblock In \emph{International Conference on Machine Learning}, pages 55950--55971. PMLR, 2025.

\bibitem[Song et~al.(2021)Song, Sohl-Dickstein, Kingma, Kumar, Ermon, and Poole]{song2021scorebased}
Yang Song, Jascha Sohl-Dickstein, Diederik~P Kingma, Abhishek Kumar, Stefano Ermon, and Ben Poole.
\newblock Score-based generative modeling through stochastic differential equations.
\newblock In \emph{International Conference on Learning Representations}, 2021.
\newblock URL \url{https://openreview.net/forum?id=PxTIG12RRHS}.

\bibitem[Stein(1981)]{stein1981estimation}
Charles~M Stein.
\newblock Estimation of the mean of a multivariate normal distribution.
\newblock \emph{The annals of Statistics}, pages 1135--1151, 1981.

\bibitem[Sznitman(1991)]{Sznitman1991PropagationChaos}
Alain-Sol Sznitman.
\newblock Topics in propagation of chaos.
\newblock In \emph{\'Ecole d'\'Et\'e{} de {P}robabilit\'es de {S}aint-{F}lour {XIX}---1989}, volume 1464 of \emph{Lecture Notes in Math.}, pages 165--251. Springer, Berlin, 1991.
\newblock ISBN 3-540-53841-0.
\newblock \doi{10.1007/BFb0085169}.
\newblock URL \url{https://doi.org/10.1007/BFb0085169}.

\bibitem[Team et~al.(2026)Team, Chen, Zhang, Su, Xu, Pan, Wang, Wang, Chen, Yin, Chen, Yan, Wei, Zhang, Meng, Hong, Xie, Liu, Lu, Tai, Chen, Men, Guo, Charles, Lu, Sui, Zhu, Zhou, He, Huang, Xu, Wang, Lai, Du, Wu, Yang, and Zhou]{kimiteam2026attentionresiduals}
Kimi Team, Guangyu Chen, Yu~Zhang, Jianlin Su, Weixin Xu, Siyuan Pan, Yaoyu Wang, Yucheng Wang, Guanduo Chen, Bohong Yin, Yutian Chen, Junjie Yan, Ming Wei, Y.~Zhang, Fanqing Meng, Chao Hong, Xiaotong Xie, Shaowei Liu, Enzhe Lu, Yunpeng Tai, Yanru Chen, Xin Men, Haiqing Guo, Y.~Charles, Haoyu Lu, Lin Sui, Jinguo Zhu, Zaida Zhou, Weiran He, Weixiao Huang, Xinran Xu, Yuzhi Wang, Guokun Lai, Yulun Du, Yuxin Wu, Zhilin Yang, and Xinyu Zhou.
\newblock Attention residuals, 2026.
\newblock URL \url{https://arxiv.org/abs/2603.15031}.

\bibitem[Teh et~al.(2025)Teh, Jabbour, and Polyanskiy]{teh2025solving}
Anzo Teh, Mark Jabbour, and Yury Polyanskiy.
\newblock Solving empirical bayes via transformers.
\newblock \emph{arXiv preprint arXiv:2502.09844}, 2025.

\bibitem[Vaswani et~al.(2017)Vaswani, Shazeer, Parmar, Uszkoreit, Jones, Gomez, Kaiser, and Polosukhin]{vaswani2017attention}
Ashish Vaswani, Noam Shazeer, Niki Parmar, Jakob Uszkoreit, Llion Jones, Aidan~N Gomez, {\L}ukasz Kaiser, and Illia Polosukhin.
\newblock Attention is all you need.
\newblock \emph{Advances in neural information processing systems}, 30, 2017.

\bibitem[Villani(2009)]{Villani2009OptimalTransport}
C\'edric Villani.
\newblock \emph{Optimal transport}, volume 338 of \emph{Grundlehren der mathematischen Wissenschaften [Fundamental Principles of Mathematical Sciences]}.
\newblock Springer-Verlag, Berlin, 2009.
\newblock ISBN 978-3-540-71049-3.
\newblock \doi{10.1007/978-3-540-71050-9}.
\newblock URL \url{https://doi.org/10.1007/978-3-540-71050-9}.
\newblock Old and new.

\bibitem[Vincent(2011)]{vincent2011ieee}
Pascal Vincent.
\newblock A connection between score matching and denoising autoencoders.
\newblock \emph{Neural Computation}, 23\penalty0 (7):\penalty0 1661--1674, 2011.
\newblock \doi{10.1162/NECO_a_00142}.

\bibitem[Von~Oswald et~al.(2023)Von~Oswald, Niklasson, Randazzo, Sacramento, Mordvintsev, Zhmoginov, and Vladymyrov]{von2023transformers}
Johannes Von~Oswald, Eyvind Niklasson, Ettore Randazzo, Jo{\~a}o Sacramento, Alexander Mordvintsev, Andrey Zhmoginov, and Max Vladymyrov.
\newblock Transformers learn in-context by gradient descent.
\newblock In \emph{International Conference on Machine Learning}, pages 35151--35174. PMLR, 2023.

\bibitem[Wakayama and Suzuki(2026)]{wakayama2026incontext}
Tomoya Wakayama and Taiji Suzuki.
\newblock In-context learning is provably bayesian inference: A generalization theory for meta-learning.
\newblock In \emph{Forty-third International Conference on Machine Learning}, 2026.
\newblock URL \url{https://openreview.net/forum?id=BUFSSOuphA}.

\bibitem[Wang et~al.(2025)Wang, Lu, Yu, Pai, Qu, and Ma]{wang2025attention_unrolled_denoising}
Peng Wang, Yifu Lu, Yaodong Yu, Druv Pai, Qing Qu, and Yi~Ma.
\newblock Attention-only transformers via unrolled subspace denoising.
\newblock In \emph{International Conference on Machine Learning}, pages 63840--63859. PMLR, 2025.

\bibitem[Watson(1964)]{watson1964smooth}
Geoffrey~S Watson.
\newblock Smooth regression analysis.
\newblock \emph{Sankhy{\=a}: The Indian Journal of Statistics, Series A}, pages 359--372, 1964.

\bibitem[Xiong et~al.(2021)Xiong, Zeng, Chakraborty, Tan, Fung, Li, and Singh]{xiong2021nystromformer}
Yunyang Xiong, Zhanpeng Zeng, Rudrasis Chakraborty, Mingxing Tan, Glenn Fung, Yin Li, and Vikas Singh.
\newblock Nystr{\"o}mformer: A nystr{\"o}m-based algorithm for approximating self-attention.
\newblock In \emph{Proceedings of the AAAI conference on artificial intelligence}, volume~35, pages 14138--14148, 2021.

\bibitem[Yang et~al.(2024)Yang, Lee, Nowak, and Papailiopoulos]{yang2024looped}
Liu Yang, Kangwook Lee, Robert~D Nowak, and Dimitris Papailiopoulos.
\newblock Looped transformers are better at learning learning algorithms.
\newblock In \emph{The Twelfth International Conference on Learning Representations}, 2024.
\newblock URL \url{https://openreview.net/forum?id=HHbRxoDTxE}.

\end{thebibliography}
\end{document}